\documentclass{article}%
\usepackage{amssymb}
\usepackage{amsmath}
\usepackage{amsfonts}
\usepackage{graphicx}%
\newtheorem{theorem}{Theorem}

\newtheorem{definition}[theorem]{Definition}
\newtheorem{example}[theorem]{Example}

\newtheorem{lemma}[theorem]{Lemma}
\newtheorem{notation}[theorem]{Notation}

\newtheorem{proposition}[theorem]{Proposition}

\newenvironment{proof}[1][Proof]{\noindent\textbf{#1.} }{\ \rule{0.5em}{0.5em}}
\begin{document}

\title{The probatilistic Quantifier Fuzzification Mechanism $\mathcal{F}^{A}$: A
theoretical analysis.}
\author{F\'{e}lix D\'{\i}az-Hermida, Alberto Bugar\'{\i}n, David E. Losada\ \\University of Santiago de Compostela}
\maketitle

\begin{abstract}
The main goal of this work is to analyze the behaviour of the $\mathcal{F}%
^{A}$ quantifier fuzzification mechanism
\cite{DiazHermida04IPMU,DiazHermida04-IEEE,DiazHermida06Tesis}. As we prove in
the paper, this model has a very solid theorethical behaviour, superior to
most of the models defined in the literature. Moreover, we show that the
underlying probabilistic interpretation has very interesting consequences.


\end{abstract}

\section{Introduction}

The evaluation of fuzzy quantified expressions is a topic that has been widely
dealt with in literature
\cite{Barro02,BoscLietard94,Bosc95,Delgado97,Delgado98,Delgado99,Delgado00,DiazHermida00,Dubois85,Dubois89fss,Glockner97,Glockner00,Glockner01,Glockner03Thesis,Glockner03-Generalized,Glockner04Libro,Glockner06Libro,Liu98,Ming2006,Ralescu95,Sanchez99,Yager83,Yager84b,Yager84,Yager88,Yager91,Yager91IFES,Yager92,Yager93,Zadeh83}%
\ since the use of quantified expressions in fields such as fuzzy control
\cite{Yager84}, temporal reasoning in robotics,
\cite{Carinena01tfcis,Carinena03tesis,Mucientes03,Mucientes01}, complex fuzzy
queries in databases \cite{Bosc95,Bosc95Sqlf}, information retrieval
\cite{Bordogna00,bordogna-pasi95,losada-etal04,DiazHermida04IPMU,DiazHermida04-IEEE,Glockner99-image}%
, data fusion \cite{Yager88,Glockner98-Fusion}, etc. can take advantage of
using vague and interpretable quantification models. Moreover, the definition
of adequate models to evaluate quantified expressions is fundamental to
perform \textquotedblleft computing with words\textquotedblright, topic that
was suggested by Zadeh \cite{Zadeh96} to express the ability of programming
systems in a linguistic way. In this paper we analyze the theoretical behavior
and some practical consequences of the $\mathcal{F}^{A}$ model defined on
\cite{DiazHermida04IPMU,DiazHermida04-IEEE}\footnote{Most of the theoretical
results presented in this paper have been previously published in the
dissertation \cite{DiazHermida06Tesis}, in spanish.}. Furthermore, we show
that the underlying probabilistic interpretation of this model hints the
utility of the model for a number of applications.

In general, approaches to fuzzy quantification in the literature use the
concept of \textit{fuzzy linguistic quantifier} \cite{Zadeh83} to represent
absolute or proportional fuzzy quantities. Zadeh \cite{Zadeh83} defines
\textit{quantifiers of the first kind} as quantifiers used for representing
absolute quantities (defined by using fuzzy numbers on $\mathbb{N}$) , and
\textit{quantifiers of the second kind }as quantifiers used for representing
relative quantities (defined by using fuzzy numbers on $\left[  0,1\right]
$). In the literature, quantifiers of the first kind are associated to
sentences involving only one single fuzzy property (as in
\textit{\textquotedblleft about three men are tall\textquotedblright} where
\textit{\textquotedblleft tall\textquotedblright} is a fuzzy property); and
quantifiers of the second kind are associated to sentences involving two fuzzy
properties (as in \textit{\textquotedblleft about 70\% of\ blond men are
tall\textquotedblright} where \textit{\textquotedblleft
blond\textquotedblright} and \textit{\textquotedblleft tall\textquotedblright}
are fuzzy properties). The linguistic quantifier associated to the former
sentence denotes the semantics of \textit{\textquotedblleft about
3\textquotedblright} and is defined by using a fuzzy number with domain on
$\mathbb{N}$. The linguistic quantifier associated to the second sentence
represents the semantics of \textit{\textquotedblleft about
70\%\textquotedblright} and is defined by using a fuzzy number with domain on
$\left[  0,1\right]  $.

Moreover, most of the existing approaches for dealing with fuzzy
quantification are based on the evaluation of the compatibility between the
linguistic quantifier and a scalar, possibilistic or probabilistic cardinality
measure for the involved fuzzy sets. Scalar approaches \cite{Zadeh83}, usually
consist of a simple evaluation of the quantifier on the cardinality value. For
possibilistic approaches, an overlapping measure \textit{SUP-min} is generally
used \cite{Delgado98,Delgado00,Ralescu95} whilst for probabilistic approaches,
\cite{Delgado99,Delgado00,DiazHermida02-FuzzySets} a weighted mean of all the
compatibility values is computed. OWA approaches \cite{Yager88,Yager92} can
also be related to the probabilistic interpretation. A different approach is
used in
\cite{Glockner97,Glockner00,Glockner01,Glockner03Thesis,Glockner03-Generalized,Glockner04Libro,Glockner06Libro}%
, where families of models that are based on a three valued interpretation of
fuzzy sets are defined.

For analyzing the behavior of fuzzy quantification models different properties
of convenient or necessary fulfillment have been defined
\cite{Delgado00,Glockner97,Glockner00,Glockner01,Glockner03Thesis,Glockner03-Generalized,Glockner04Libro,Glockner06Libro,Sanchez99}%
. Most of the approaches in literature fail to exhibit a plausible behavior
\cite{Barro02,Delgado00,Glockner99,Glockner03Thesis,Glockner04Libro,DiazHermida06Tesis,Sanchez99}%
, and only a few
\cite{Delgado00,DiazHermida02-FuzzySets,Glockner97,Glockner00,Glockner01,Glockner03Thesis,Glockner03-Generalized,Glockner04Libro}
seem to exhibit an adequate behavior in the general case.

In this work we will follow the Gl\"{o}ckner approximation to fuzzy
quantification
\cite{Glockner97,Glockner00,Glockner01,Glockner03Thesis,Glockner03-Generalized,Glockner04Libro,Glockner06Libro}%
. In his approach, the author generalizes the concept of \textit{generalized
classic quantifier} \cite{Barwise81,Gamut84,Keenan97VanBenthem} (second order
predicates or set relationships) to the fuzzy case; that is, a \textit{fuzzy
quantifier }is a fuzzy relationship between fuzzy sets. And then rewrites the
fuzzy quantification problem as the problem of looking for mechanism to
transform \textit{semi-fuzzy quantifiers} (quantifiers between generalized
classic quantifiers and fuzzy quantifiers that are adequate to specify the
meaning of quantified expressions) to fuzzy quantifiers.

Moreover, Gl\"{o}ckner has also defined a rigorous axiomatic framework to
assure the good behavior of QFMs. Models fulfilling this framework are called
\textit{Determiner fuzzification schemes (DFSs) }and they fulfill an important
set of appropriate behavior properties.

The main goal of this work is to analyze the behavior of the $\mathcal{F}^{A}$
model \cite{DiazHermida04IPMU,DiazHermida04-IEEE,DiazHermida06Tesis}. This
model has a very solid theoretical behavior, superior to most of the models
defined in the literature. Moreover, we show that the underlying probabilistic
interpretation based on likelihood functions
\cite{Mabuchi92,Thomas95,Tursken2000Fundamentals,Dubois2000Fundamentals} has
very interesting consequences, that assure its utility for a number of
applications. For example, in
\cite{DiazHermida04IPMU,DiazHermida04-IEEE,DiazHermida06Tesis} the application
of the model in a information retrieval task was shown, with competitive
results. In \cite{DiazHermida10Estylf}\ the model has been used in a
summarization application for the evaluation of quantified temporal
expressions. From a theoretical point of view the model is a DFS, although is
only defined in finite domains. The fulfillment of the DFS axioms guarantees a
very good theoretical behavior. As an important point, the fuzzy operators
induced by the model are the product t-norm and the probabilistic sum
t-conorm. This fact makes the $\mathcal{F}^{A}$ model essentially different of
the models defined in \cite{Glockner04Libro} because all those
\textquotedblleft standard models\textquotedblright\ induce the min tnorm and
the max tconorm. To our knowledge, the $\mathcal{F}^{A}$ model is the unique
known non standard DFSs.

The paper is organized as follows. In the first section, we resume the
Gl\"{o}ckner's approach to fuzzy quantification, based on quantifier
fuzzification mechanisms\footnote{A complete explanation of the QFM framework
can be consulted in the excellent work \cite{Glockner06Libro}.}. In the second
section we explain some of the properties that let us to analyze the behavior
of the quantification model. Most of them are a compilation of the properties
defined on \cite[chapters 3 and 4]{Glockner04Libro,Glockner06Libro}, but we
have added to these properties two very interesting properties fulfilled by
the $\mathcal{F}^{A}$ model and by the probabilistic models defined in
\cite{DiazHermida02-FuzzySets}. In section three the $\mathcal{F}^{A}$ QFM is
defined. We also explores the behavior of the model when the cardinality of
the base set tends to infinite, with a surprising relation with original $\sum
Count$ Zadeh's model \cite{Zadeh83}. Next section is devoted to some
interesting consequences of the probabilistic interpretation of the
$\mathcal{F}^{A}$ QFM, with relation with a number of application fields.
Proofs of the properties and efficient algorithm solutions are collected in
two apendixes. A bibliographic analysis of quantification models has not been
included as it can be found in
\cite{Barro02,Delgado00,Glockner99,Glockner03Thesis,Glockner04Libro,DiazHermida06Tesis,Sanchez99}%
.

\section{Quantifier fuzzification mechanisms}

To overcome the Zadeh's framework to fuzzy quantification Gl\"{o}ckner
\cite{Glockner06Libro} rewrites the problem of fuzzy quantification as the
problem of looking for adequate means to convert the specification means
(semi-fuzzy quantifiers) into the operational means (fuzzy quantifiers)
\cite{Glockner06Libro}. In this section we explain in some detail the
framework proposed by Gl\"{o}ckner to achieve that result.

Fuzzy quantifiers are just a fuzzy generalization of crisp or classic
quantifiers. Before giving the definition of fuzzy quantifiers, we will show
the definition of classic quantifiers and some examples:

\begin{definition}
[Classic quantifier.]\cite[pag. 57]{Glockner06Libro} A two valued
(generalized) quantifier on a base set $E\neq\varnothing$ is a mapping
$Q:\mathcal{P}\left(  E\right)  ^{n}\longrightarrow\mathbf{2}$, where
$n\in\mathbb{N}$ is the arity (number of arguments) of $Q$, $\mathbf{2}%
=\left\{  0,1\right\}  $ denotes the set of crisp truth values, and
$\mathcal{P}\left(  E\right)  $ is the powerset of $E$.
\end{definition}

In this work\ we assume the base set $E$ is finite as the $\mathcal{F}^{A}$
model is only defined on finite base sets.

Examples of some definitions of classic quantifiers are:%
\begin{align}
\mathbf{all}\left(  Y_{1},Y_{2}\right)   &  =Y_{1}\subseteq Y_{2}%
\label{TFGQ_1}\\
\mathbf{at\_least}80\%\left(  Y_{1},Y_{2}\right)   &  =\left\{
\begin{array}
[c]{cc}%
\frac{\left\vert Y_{1}\cap Y_{2}\right\vert }{\left\vert Y_{1}\right\vert
}\geq0.80 & X_{1}\neq\varnothing\\
1 & X_{1}=\varnothing
\end{array}
\right. \nonumber
\end{align}

\begin{example}
Let us consider the evaluation of the sentence \textquotedblleft at least
eighty percent of the members are lawyers\textquotedblright\ where the
properties \textquotedblleft members\textquotedblright\ and \textquotedblleft
lawyers\textquotedblright\ are respectively defined as $Y_{1}=\left\{
1,0,1,0,1,0,1,1\right\}  ,Y_{2}=\left\{  1,0,1,0,1,0,0,0\right\}  $, and
\textquotedblleft at least eighty percent\textquotedblright\ is defined in
expression \ref{TFGQ_1}. Then $\mathbf{at\_least}80\%\left(  Y_{1}%
,Y_{2}\right)  =0$.
\end{example}

In a fuzzy quantifier arguments and result can be fuzzy. The definition of a
fuzzy quantifier is:

\begin{definition}
[Fuzzy Quantifier]\cite[pag. 66]{Glockner06Libro} An n-ary fuzzy quantifier
$\widetilde{Q}$ on a base set $E\neq\varnothing$ is a mapping $\widetilde
{Q}:\widetilde{\mathcal{P}}\left(  E\right)  ^{n}\longrightarrow
\mathbf{I=}\left[  0,1\right]  $. Here $\widetilde{\mathcal{P}}\left(
E\right)  $ denotes the fuzzy powerset of $E$.
\end{definition}

A fuzzy quantifier assigns a gradual result to each choice of $X_{1}%
,\ldots,X_{n}\in\widetilde{\mathcal{P}}\left(  E\right)  $.

An example of a fuzzy quantifier could be $\widetilde{\mathbf{all}}%
:\widetilde{\mathcal{P}}\left(  E\right)  ^{2}\longrightarrow\mathbf{I}$. A
reasonable fuzzy definition of the fuzzy quantifier $\widetilde{\mathbf{all}}$
is:
\begin{equation}
\widetilde{\mathbf{all}}\left(  X_{1},X_{2}\right)  =\inf\left\{  \max\left(
1-\mu_{X_{1}}\left(  e\right)  ,\mu_{X_{2}}\left(  e\right)  \right)  :e\in
E\right\}  \label{FuzzyEvaluationAll_1}%
\end{equation}

\begin{example}
Let us consider the evaluation of the sentence \textquotedblleft all big
houses are overvaluated\textquotedblright\ in a referential set $E=\left\{
e_{1},\ldots,e_{4}\right\}  $. Let us assume that properties \textquotedblleft
big\textquotedblright\ and \textquotedblleft overvaluated\textquotedblright%
\ are respectively defined as: $X_{1}=\left\{  0.8/e_{1},1/e_{2}%
,0.6/e_{3},0.3/e_{4}\right\}  $, $X_{2}=\left\{  0.9/e_{1},0.7/e_{2}%
,0.3/e_{3},0.2/e_{4}\right\}  $. If we use expression
(\ref{FuzzyEvaluationAll_1}) then: $\widetilde{\mathbf{all}}\left(
X_{1},X_{2}\right)  =\inf\left\{  \max\left(  1-\mu_{X_{1}}\left(  e\right)
,\mu_{X_{2}}\left(  e\right)  \right)  :e\in E\right\}  =0.4$.
\end{example}

Although a certain consensus may be achieved to accept this previous
expression as a suitable definition for $\widetilde{\mathbf{all}}$ this is not
the unique one. The problem of establishing consistent fuzzy definitions for
quantifiers (e.g., \textit{\textquotedblleft at least eighty
percent\textquotedblright}) is faced in \cite{Glockner06Libro} by introducing
the concept of semi-fuzzy quantifiers. A semi-fuzzy quantifier represents a
medium point between classic quantifiers and fuzzy quantifiers, and it is
close but is far more general than the idea of Zadeh's linguistic quantifiers
\cite{Zadeh83}. A semi-fuzzy quantifier only accepts crisp arguments, as
classic quantifiers, but lets the result range on the truth grade scale
$\mathbf{I}$, as for fuzzy quantifiers\footnote{An interesting classification
of semi-fuzzy quantifiers is shown in \cite{DiazHermida02-clasificacion}. In
\cite[chapter 4]{DiazHermida06Tesis} an extended classification is defined.}.

\begin{definition}
[Semi-fuzzy quantifier]\cite[pag. 71]{Glockner06Libro} An n-ary semi-fuzzy
quantifier $Q$ on a base set $E\neq\varnothing$ is a mapping $Q:\mathcal{P}%
\left(  E\right)  ^{n}\longrightarrow\mathbf{I}$.
\end{definition}

$Q$ assigns a gradual result to each pair of crisp sets $\left(  Y_{1}%
,\ldots,Y_{n}\right)  $.

Examples of semi-fuzzy quantifiers are:
\begin{align}
\mathbf{about\_5}\left(  Y_{1},Y_{2}\right)   &  =T_{2,4,6,8}\left(
\left\vert Y_{1}\cap Y_{2}\right\vert \right) \label{about_or_more_80}\\
\mathbf{at\_least}\_\mathbf{about}80\%\left(  Y_{1},Y_{2}\right)   &
=\left\{
\begin{array}
[c]{cc}%
S_{0.5,0.8}\left(  \frac{\left\vert Y_{1}\cap Y_{2}\right\vert }{\left\vert
Y_{1}\right\vert }\right)  & X_{1}\neq\varnothing\\
1 & X_{1}=\varnothing
\end{array}
\right. \nonumber
\end{align}
where $T_{2,4,6,8}\left(  x\right)  $ and $S_{0.5,0.8}\left(  x\right)  $ are
shown in figure (\ref{FigCuantificador})\footnote{Functions $T_{a,b,c,d}$ and
$S_{\alpha,\gamma}$ are defined as
\[
T_{a,b,c,d}\left(  x\right)  =\left\{
\begin{array}
[c]{cc}%
0 & x\leq a\\
\frac{x-a}{b-a} & a<x\leq b\\
1 & b<x\leq c\\
1-\frac{x-c}{d-c} & c<x\leq d\\
0 & d<x
\end{array}
\right.  \quad,S_{\alpha,\gamma}\left(  x\right)  =\left\{
\begin{tabular}
[c]{ll}%
$0$ & $x<\alpha$\\
$2\left(  \frac{\left(  x-\alpha\right)  }{\left(  \gamma-\alpha\right)
}\right)  ^{2}$ & $\alpha<x\leq\frac{\alpha+\gamma}{2}$\\
$1-2\left(  \frac{\left(  x-\gamma\right)  }{\left(  \gamma-\alpha\right)
}\right)  ^{2}$ & $\frac{\alpha+\gamma}{2}<x\leq\gamma$\\
$1$ & $\gamma<x$%
\end{tabular}
\ \ \ \ \ \ \ \ \ \ \ \ \ \ \ \right.
\]
\par
In this work, we will use the following relative definitions for the
existential and the universal fuzzy number:
\par%
\[
\exists\left(  x\right)  =\left\{
\begin{array}
[c]{cc}%
0 & x=0\\
1 & x>0
\end{array}
\right.  \quad,\forall\left(  x\right)  =\left\{
\begin{tabular}
[c]{ll}%
$0$ & $x<1$\\
$1$ & $x=1$%
\end{tabular}
\ \ \ \ \ \ \ \ \ \ \ \ \ \ \ \right.
\]
}.

\begin{example}
Let us consider the evaluation of the sentence \textquotedblleft about at
least 80\% the students are Spanish". Let us assume that properties
\textquotedblleft students\textquotedblright\ and \textquotedblleft
Spanish\textquotedblright\ are respectively defined as: $Y_{1}=\left\{
1,0,1,0,1,0,1,1\right\}  ,Y_{2}=\left\{  1,0,1,0,1,0,0,0\right\}  $, then
$\mathbf{at\_least}\_\mathbf{about}80\%\left(  Y_{1},Y_{2}\right)
=S_{0.5,0.8}^{{}}\left(  \frac{\left\vert Y_{1}\cap Y_{2}\right\vert
}{\left\vert Y_{1}\right\vert }\right)  =0.22$.
\end{example}

Semi-fuzzy quantifiers are much more intuitive and easier to define than fuzzy
quantifiers, but they do not solve the problem of evaluating fuzzy quantified sentences.

In order to do so mechanisms are needed that enable us to transform semi-fuzzy
quantifiers into fuzzy quantifiers, i.e., mappings with domain in the universe
of semi-fuzzy quantifiers and range in the universe of fuzzy quantifiers.
Glockner names those mechanisms \textit{quantifier fuzzification mechanisms}.

\begin{definition}
\cite[pag. 74]{Glockner06Libro}A quantifier fuzzification mechanism (QFM)
$\mathcal{F}$ assigns to each semi-fuzzy quantifier $Q:\mathcal{P}\left(
E\right)  ^{n}\rightarrow\mathbf{I}$ a corresponding fuzzy quantifier
$\mathcal{F}\left(  Q\right)  :\widetilde{\mathcal{P}}\left(  E\right)
^{n}\rightarrow\mathbf{I}$ of the same artity $n\in\mathbb{N}$ and on the same
base set.
\end{definition}

%

\begin{figure}
[ptb]
\begin{center}
\includegraphics[
height=3.4541cm,
width=9.7066cm
]%
{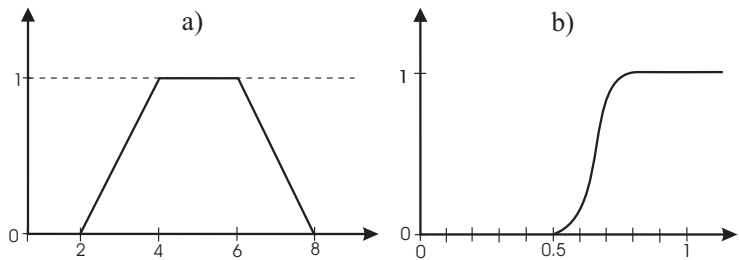}%
\caption{quantifiers \textbf{about\_5} (a) and \textbf{at\_least\_about\_80\%}
(b)}%
\label{FigCuantificador}%
\end{center}
\end{figure}

\section{Some properties to guarantee the good behavior of QFMs}

Before proceeding to explain the QFM $\mathcal{F}^{A}$ we will introduce some
of the properties that let us to guarantee a good behavior of the QFMs. For
the sake of brevity, we have only selected some of the more important
properties to characterize the behavior of quantification models. A complete
and detailed exposition, showing the intuitions under those definitions can be
found in \cite[chapters three and four.]{Glockner06Libro}.

The set of properties is organized in three sets. First set is composed of the
most important properties that are consequence of the DFS axioms. Second group
is composed of some properties that are not consequence of the DFS framework
but are important to characterize the behavior of QFMs for different reasons.
The last group includes two very important properties that the $\mathcal{F}%
^{A}$ model and the probabilistic models defined on
\cite{DiazHermida02-FuzzySets} fulfills\footnote{One of the models defined in
\cite{DiazHermida02-FuzzySets} is a generalization of an original proposal of
Delgado et al. \cite{Delgado99,Delgado00} to semi-fuzzy quantifiers.}.

In the appendix we show the proof of those properties for the $\mathcal{F}%
^{A}$ QFM.

\subsection{Some properties that are consequence of the DFS axiomatic
framework}

\subsubsection{Correct generalization property (P.1)}

Perhaps the most fundamental property to be fulfilled by a QFM is the correct
generalization property. This property, defined independently by Gl\"{o}ckner
\cite{Glockner97} for QFMs and by Delgado et al. for models following the
Zadeh's framework \cite{Sanchez99, Delgado00}, requires that the behavior of a
fuzzy quantifier $\mathcal{F}\left(  Q\right)  $ on crisp arguments was the
expected; that is, the results obtained with a fuzzy quantifier $\mathcal{F}%
\left(  Q\right)  $ and with the corresponding semi-fuzzy quantifier $Q$ must
coincide on crisp arguments.

We show now the definition of the property:

\begin{definition}
[\textbf{Property of correct generalization}]\cite[pag. 112]{Glockner06Libro}%
\textbf{\ }Let $Q:\mathcal{P}\left(  E\right)  ^{n}\rightarrow\mathbf{I},n>0$
be an n-ary semi-fuzzy quantifier. We say that a QFM $\mathcal{F}$ fulfills
the property of correct generalization if for all the crisp subsets
$Y_{1},\ldots,Y_{n}\in\mathcal{P}\left(  E\right)  $, then it holds
$\mathcal{F}\left(  Q\right)  \left(  Y_{1},\ldots,Y_{n}\right)  =Q\left(
Y_{1},\ldots,Y_{n}\right)  $.
\end{definition}

For a detailed explanation of this property \cite[Sections 3.2. and
4.2.]{Glockner06Libro} can be consulted.

For example, given crisp sets $Y_{1},Y_{2}\in\mathcal{P}\left(  E\right)  $,
$Y_{1}=student$, $Y_{2}=spanish$, then this property guarantees that
\[
\mathcal{F}\left(  \mathbf{some}\right)  \left(  student,spanish\right)
=\mathbf{some}\left(  student,spanish\right)
\]

In the DFS axiomatic framework it is sufficient to guarantee this property in
the unary case.

\subsubsection{Membership assessment (P.2)}

This property is related with the evaluation of the membership grade of a
particular element \cite[section 3.3.]{Glockner06Libro}, and belongs to the
set of axioms that are used to characterize the DFSs.

In the classic case, we can define a crisp quantifier $\pi_{e}:\mathcal{P}%
\left(  E\right)  \rightarrow\mathbf{2}$ that test if the element $e$ belongs
to the argument set. In the same way, in the fuzzy case, we can define a fuzzy
quantifier $\widetilde{\pi}_{e}$ that returns the membership grade of $e$. It
is natural to require that a reasonable QFM $\mathcal{F}$ maps $\pi_{e}$ to
$\widetilde{\pi}_{e}$.

The formal definitions of $\pi_{e}$ and $\widetilde{\pi}_{e}$ are:

\begin{definition}
\cite[pag. 88]{Glockner06Libro} Let $E$ a base set and $e\in E$. The
projection quantifier $\pi_{e}:\mathcal{P}\left(  E\right)  \rightarrow
\mathbf{2}$ is defined by $\pi_{e}\left(  Y\right)  =\chi_{Y}\left(  e\right)
$ for all $Y\in\mathcal{P}\left(  E\right)  $, where $\chi_{Y}\left(
e\right)  $ denotes the crisp characteristic funtion of the set $Y$.
\end{definition}

The corresponding fuzzy definition is:

\begin{definition}
\cite[pag. 88]{Glockner06Libro} Let a base set $E$ be given and $e\in E$. The
fuzzy projection quantifier $\widetilde{\pi}_{e}:\widetilde{\mathcal{P}%
}\left(  E\right)  \rightarrow\mathbf{2}$ is defined by $\widetilde{\pi}%
_{e}\left(  Y\right)  =\chi_{Y}\left(  e\right)  $ for all $Y\in
\mathcal{P}\left(  E\right)  $.
\end{definition}

Using these definitions the property that establishes that a QFM $\mathcal{F}$
generalizes the quantifier $\pi_{e}$ in the correct way is defined:

\begin{definition}
[Projection quantifiers]\cite[pag. 89, pag. 112]{Glockner06Libro}%
\label{DefPropValorVerdadInducido} Let $\mathcal{F}$ a QFM. $\mathcal{F}$
fulfills the property of projection quantifiers if it holds $\mathcal{F}%
\left(  \pi_{e}\right)  =\widetilde{\pi}_{e}$ for $E\neq\varnothing$ and $e\in
E$.
\end{definition}

\subsubsection{Induced operators (P3)}

Gl\"{o}ckner explains that a QFM can be used to transform crisp logical
operators into fuzzy operators. For example, logical \textquotedblleft
or\textquotedblright\ can be extended by using the following semi-fuzzy
quantifier defined on a referential set $E$ composed by two elements
($E=\left\{  e_{1},e_{2}\right\}  $):%
\[
Q_{\vee}\left(  X\right)  =\left\{
\begin{array}
[c]{cc}%
0 & if\text{ }X=\varnothing\\
1 & if\text{ }X=\left\{  e_{1}\right\}  \vee X=\left\{  e_{2}\right\}  \vee
X=\left\{  e_{1},e_{2}\right\}
\end{array}
\right.
\]
and in this way is possible to define the fuzzy logical function
$\widetilde{\vee}$ that is induced by the fuzzification mechanism
$\mathcal{F}$ as
\[
\widetilde{\vee}\left(  x_{1},x_{2}\right)  =\widetilde{\mathcal{F}}\left(
\vee\right)  \left(  x_{1},x_{2}\right)  =\mathcal{F}\left(  Q_{\vee}\right)
\left(  \left\{  x_{1}/e_{1},x_{2}/e_{2}\right\}  \right)
\]

This construction is shown in \cite{Glockner00,Glockner01}, \cite[Section
3.4]{Glockner06Libro}. In \cite[Secci\'{o}n 1]{Glockner97}, \cite[Section
4.4]{Glockner06Libro} a different construction is shown.

To formally define this property the next bijection $\eta:\mathbf{2}%
^{n}\rightarrow\mathcal{P}\left(  \left\{  1,\ldots,n\right\}  \right)  $ is
needed:%
\[
\eta\left(  x_{1},\ldots,x_{n}\right)  =\left\{  k\in\left\{  1,\ldots
,n\right\}  :x_{k}=1\right\}
\]
for all $x_{1},\ldots,x_{n}\in\mathbf{2}$. In the fuzzy case the analogous
bijection is $\mu_{\widetilde{\eta}\left(  x_{1},\ldots,x_{n}\right)  }\left(
k\right)  =x_{k}$ for all $x_{1},\ldots,x_{n}\in\mathbf{I}$ and $k\in\left\{
1,\ldots,n\right\}  $.

These bijections are used to transform the fuzzy truth functions (i.e.
mappings $\mathbf{2}^{n}\rightarrow\mathbf{I}$) in semi-fuzzy quantifiers
$Q_{f}:\mathcal{P}\left(  \left\{  1,\ldots,n\right\}  \right)  \rightarrow
\mathbf{I}$\textbf{.} In the same way fuzzy quantifieres $\widetilde
{Q}:\widetilde{\mathcal{P}}\left(  \left\{  1,\ldots,n\right\}  \right)
\rightarrow\mathbf{I}$ can be transformed in fuzzy truth functions
$\widetilde{f}:\mathbf{I}^{n}\rightarrow\mathbf{I}$.

The definition that let us to transform semi-fuzzy truth function in fuzzy
truth functions by means of a QFM is the following:

\begin{definition}
\cite[pag. 90]{Glockner06Libro} Suppose $\mathcal{F}$ is a QFM and
$f:\mathbf{2}^{n}\rightarrow\mathbf{I}$ is a mapping (i.e. a `semi-fuzzy truth
funtion') for some $n\in\mathbb{N}$. The semi-fuzzy quantifier $Q_{f}%
:\mathcal{P}\left(  \left\{  1,\ldots,n\right\}  \right)  \rightarrow
\mathbf{I}$ is defined by $Q_{f}\left(  Y\right)  =f\left(  \eta^{-1}\left(
Y\right)  \right)  $ for all $Y\in\mathcal{P}\left(  \left\{  1,\ldots
,n\right\}  \right)  $. In terms of $Q_{f}$, the induced fuzzy truth function
$\widetilde{\mathcal{F}}\left(  f\right)  :\mathbf{I}^{n}\rightarrow
\mathbf{I}$ is defined by%
\[
\widetilde{\mathcal{F}}\left(  f\right)  \left(  x_{1},\ldots,x_{n}\right)
=\widetilde{\mathcal{F}}\left(  Q_{f}\right)  \left(  \eta^{-1}\left(
x_{1},\ldots,x_{n}\right)  \right)
\]
for all $x_{1},\ldots,x_{n}\in\mathbf{I}$.
\end{definition}

The construction allows us to transform the usual crisp logical operators
($\lnot$, $\wedge$, $\vee$, $\rightarrow$) into the analogous fuzzy operators
($\widetilde{\lnot}$, $\widetilde{\wedge}$, $\widetilde{\vee}$, $\widetilde
{\rightarrow}$). For a reasonable QFM we should expect that the induced
operators were fuzzy valid operators.

For a DFS the next property is guaranteed\footnote{This is a resume of the
longer exposition maked in \cite[section 4.3]{Glockner06Libro}.}:

\begin{definition}
[Property of the induced truth functions]Truth operations induced by a
quantifier fuzzification mechanism must be coherent with fuzzy logic;\ i.e.,
the following must hold:\newline a. $\widetilde{id_{\mathbf{I}}}=\widetilde
{F}\left(  id_{2}\right)  $ (where $id_{2}:\mathbf{2}\rightarrow\mathbf{2}$ is
the bivalued identity truth function) is the fuzzy identity truth
function.\newline b. $\widetilde{\lnot}=\widetilde{F}\left(  \lnot\right)  $
is a strong negation operator.\newline c. $\widetilde{\wedge}=\widetilde
{F}\left(  \wedge\right)  $ is a tnorm.\newline d. $\widetilde{\vee
}=\widetilde{F}\left(  \vee\right)  $ is a tconorm.\newline e. $\widetilde
{\rightarrow}=\widetilde{F}\left(  \rightarrow\right)  $ is an implication function.
\end{definition}

In this manner it is guaranteed that the fuzzy operators that are generated
are reasonable from the perspective of fuzzy logic. For example,
for\ $\mathcal{F}\left(  \mathbf{some}\right)  \left(  tall,blond\right)  $
where $tall=\left\{  0.7/John\right\}  $ and $blond=\left\{  0.4/John\right\}
$ it is guaranteed we obtaine the result of using the\ induced $tconorm$ on
$\left(  0.7,0.4\right)  $.

\subsubsection{External negation property (P.4)}

Now we are going to present a set of three very important properties from a
linguistic point of view. The properties of\textit{ external negation},
\textit{internal negation} and \textit{duality}. We will begin defining the
external negation property \cite[section 3.5]{Glockner06Libro}:

\begin{definition}
[\textbf{External negation}]\cite[pag. 93]{Glockner06Libro}The external
negation of a semi-fuzzy quantifier $Q:\mathcal{P}\left(  E\right)
^{n}\rightarrow\mathbf{I}$ is defined by $\left(  \widetilde{\lnot}Q\right)
\left(  Y_{1},\ldots,Y_{n}\right)  =\widetilde{\lnot}\left(  Q\left(
Y_{1},\ldots,Y_{n}\right)  \right)  $ for all $Y_{1},\ldots,Y_{n}%
\in\mathcal{P}\left(  E\right)  $. The definition of $\widetilde{\lnot
}\widetilde{Q}:\widetilde{\mathcal{P}}\left(  E\right)  \rightarrow\mathbf{I}$
in the case of fuzzy quantifiers $\widetilde{Q}:\widetilde{\mathcal{P}}\left(
E\right)  \rightarrow\mathbf{I}$ is analogous\footnote{The reasonable choice
of the fuzzy negation $\widetilde{\lnot}:\mathbf{I}\rightarrow\mathbf{I}$ is
the induced negation of the QFM.}.
\end{definition}

From a linguistic point of view, the external negation of
\textit{\textquotedblleft all the students are spanish\textquotedblright} is
\textit{\textquotedblleft not all the students are spanish\textquotedblright.}

A QFM correctly generalizes the external negation property if it fulfills the
next property:\footnote{The property of external negation is one of the
initial axioms of the axiomatic framework presented in \cite[pag.
22]{Glockner97} to define the DFSs.}

\begin{definition}
[\textbf{External negation property.}]\cite[pag. 22]{Glockner97},
\cite[section 3.5]{Glockner06Libro}\label{DefPropNegExterna}\textbf{\ }Let
$Q:\mathcal{P}\left(  E\right)  ^{n}\rightarrow\mathbf{I}$ a semi-fuzzy
quantifier. $\mathcal{F}$ fulfills the property of external negation if
$\mathcal{F}\left(  \widetilde{\lnot}Q\right)  =\widetilde{\lnot}%
\mathcal{F}\left(  Q\right)  $.
\end{definition}

For example, the fulfillment of this property assures:%
\[
\mathcal{F}\left(  \text{\textbf{at most 10}}\right)  \left(  X_{1}%
,X_{2}\right)  =\mathcal{F}\left(  \widetilde{\lnot}\text{\textbf{at least
11}}\right)  \left(  X_{1},X_{2}\right)  =\widetilde{\lnot}\mathcal{F}\left(
\text{\textbf{at least 11}}\right)  \left(  X_{1},X_{2}\right)
\]
That is, the equivalence between the expressions \textit{\textquotedblleft at
most ten rich students are intelligent\textquotedblright} and
\textit{\textquotedblleft no more than eleven rich students are
intelligent\textquotedblright} is assured in the fuzzy case.

\subsubsection{Internal negation property (P.5)}

The internal negation or antonym of a semi-fuzzy quantifier is defined as:

\begin{definition}
[\textbf{Internal negation.}]\cite[pag. 93]{Glockner06Libro} Let a semi-fuzzy
quantifier $Q:\mathcal{P}\left(  E\right)  ^{n}\rightarrow\mathbf{I}$ of arity
$n>0$ be given. The internal negation $Q\lnot:\mathcal{P}\left(  E\right)
^{n}\rightarrow\mathbf{I}$ of $Q$ is defined by
\[
Q\lnot\left(  Y_{1},\ldots,Y_{n}\right)  =Q\lnot\left(  Y_{1},\ldots,\lnot
Y_{n}\right)
\]
for all $Y_{1},\ldots,Y_{n}\in\mathcal{P}\left(  E\right)  $. The internal
negation $\widetilde{Q}\widetilde{\lnot}:\widetilde{\mathcal{P}}\left(
E\right)  ^{n}\rightarrow\mathbf{I}$ of a fuzzy quantifier $\widetilde
{Q}:\widetilde{\mathcal{P}}\left(  E\right)  ^{n}\rightarrow\mathbf{I}$ is
defined analogously, based on the given fuzzy complement $\widetilde{\lnot}$.
\end{definition}

For example, the internal negation of $\mathbf{all:}\mathcal{P}\left(
E\right)  ^{2}\rightarrow\mathbf{I}$ is $\mathbf{no}:\mathcal{P}\left(
E\right)  ^{2}\rightarrow\mathbf{I}$ because%
\[
\mathbf{all}\left(  Y_{1},Y_{2}\right)  \lnot=\mathbf{all}\left(  Y_{1},\lnot
Y_{2}\right)  =\mathbf{no}\left(  Y_{1},Y_{2}\right)
\]

The definition of the property of internal negation is:\footnote{The property
of internal negation is one of the initial axioms of the axiomatic framework
presented in \cite[pag. 22]{Glockner97} to define the DFSs.}

\begin{definition}
[\textbf{Internal negation property}]\cite[pag. 22]{Glockner97}\cite[section
3.5]{Glockner06Libro}\label{DefPropNegInterna}\textbf{ }Let $Q:\widetilde
{\mathcal{P}}\left(  E\right)  ^{n}\rightarrow\mathbf{I}$ be a semi-fuzzy
quantifier of arity $n>0$. A QFM $\mathcal{F}$ fulfills the property of
internal negation if $\mathcal{F}\left(  Q\lnot\right)  =\mathcal{F}\left(
Q\right)  \widetilde{\lnot}$.
\end{definition}

For example, this property assures%
\[
\mathcal{F}\left(  \mathbf{all}\right)  \left(  X_{1},X_{2}\right)
=\mathcal{F}\left(  \mathbf{all}\lnot\right)  \left(  X_{1},\widetilde{\lnot
}X_{2}\right)  =\mathcal{F}\left(  \mathbf{no}\right)  \left(  X_{1}%
,\widetilde{\lnot}X_{2}\right)
\]

That is, the equivalence between the expressions \textit{\textquotedblleft all
big houses are overvaluated\textquotedblright} and \textit{\textquotedblleft
no big houses are undervaluated\textquotedblright} is assured in the fuzzy case.

\subsubsection{Duality property (P.6)}

This property is a consequence of the fulfillment of the external and internal
negation properties. In \cite{Glockner06Libro} is one of the axioms used to
define the DFSs.

\begin{definition}
[\textbf{Dual quantifier.}]\cite[pag. 99]{Glockner04Libro}The dual
$Q\widetilde{\square}:\mathcal{P}\left(  E\right)  ^{n}\rightarrow\mathbf{I}$
of a semi-fuzzy quantifier $Q\widetilde{\square}:\mathcal{P}\left(  E\right)
^{n}\rightarrow\mathbf{I}$, $n>0$ is defined by%
\[
Q\widetilde{\square}\left(  Y_{1},\ldots,Y_{n}\right)  =\widetilde{\lnot
}Q\left(  Y_{1},\ldots,\lnot Y_{n}\right)
\]
for all $Y_{1},\ldots,Y_{n}\in\mathcal{P}\left(  E\right)  $. The dual
$\widetilde{Q}\widetilde{\square}=\widetilde{\lnot}\widetilde{Q}%
\widetilde{\lnot}$ of a fuzzy quantifier $\widetilde{Q}$ is defined analogously.
\end{definition}

For example, the dual of $\mathbf{all:}\mathcal{P}\left(  E\right)
^{2}\rightarrow\mathbf{I}$ is%
\[
\mathbf{all}\widetilde{\square}\left(  Y_{1},Y_{2}\right)  =\widetilde{\lnot
}\mathbf{all}\left(  Y_{1},\lnot Y_{2}\right)  =\mathbf{some}\left(
Y_{1},Y_{2}\right)
\]

Using the axiom of duality \cite[pag. 94-96]{Glockner06Libro} the duality
property can be defined:

\begin{definition}
[\textbf{Duality property}]\label{DefPropDualidad}\textbf{\ }We say that a QFM
$\mathcal{F}$ fulfills the property of duality if for all semi-fuzzy
quantifiers $Q:\widetilde{\mathcal{P}}\left(  E\right)  ^{n}\rightarrow
\mathbf{I}$ of arity $n>0$ $\mathcal{F}\left(  Q\widetilde{\square}\right)
=\mathcal{F}\left(  Q\right)  \widetilde{\square}$.
\end{definition}

For example this property assures that%
\[
=\mathcal{F}\left(  \mathbf{all}\right)  \widetilde{\square}\left(
X_{1},X_{2}\right)  =\mathcal{F}\left(  \mathbf{some}\right)  \left(
X_{1},X_{2}\right)
\]
that is, the equivalence of the sentences \textit{\textquotedblleft not all
the expensives cars are not good\textquotedblright} and
\textit{\textquotedblleft some expensive car is good\textquotedblright} is
assured in the fuzzy case.

\subsubsection{Internal meets property
(P.7)\label{PropiedadEncuentrosInternos}}

In combination with negation properties, this property assures boolean
combination of arguments are mapped to the fuzzy case.

First, we show the \textquotedblleft union\textquotedblright\ and
\textquotedblleft intersection\textquotedblright\ quantifiers:

\begin{definition}
[\textbf{Union quantifier}]\cite[section 3.7]{Glockner06Libro} Let
$Q:\mathcal{P}\left(  E\right)  ^{n}\rightarrow\mathbf{I}$ be a semi-fuzzy
quantifier, $n>0$, be given. We define the fuzzy quantifier $Q\cup
:\mathcal{P}\left(  E\right)  ^{n+1}\rightarrow\mathbf{I}$ as
\[
Q\cup\left(  Y_{1},\ldots,Y_{n},Y_{n+1}\right)  =Q\left(  Y_{1},\ldots
,Y_{n-1},Y_{n}\cup Y_{n+1}\right)
\]
for all $Y_{1},\ldots,Y_{n+1}\in\mathcal{P}\left(  E\right)  $. In the case of
fuzzy quantifiers $\widetilde{Q}\widetilde{\cup}$ is defined analogously,
based on a fuzzy definition of $\widetilde{\cup}$.
\end{definition}

\begin{definition}
[\textbf{Intersection quantifier}]Let $Q:\mathcal{P}\left(  E\right)
^{n}\rightarrow\mathbf{I}$ a semi-fuzzy quantifier, $n>0$, be given. We define
the semi-fuzzy quantifier $Q\cap:\mathcal{P}\left(  E\right)  ^{n+1}%
\rightarrow\mathbf{I}$ as%
\[
Q\cap\left(  Y_{1},\ldots,Y_{n},Y_{n+1}\right)  =Q\left(  Y_{1},\ldots
,Y_{n-1},Y_{n}\cap Y_{n+1}\right)
\]
for all $Y_{1},\ldots,Y_{n+1}\in\mathcal{P}\left(  E\right)  $. In the case of
fuzzy quantifiers $\widetilde{Q}\widetilde{\cap}$ is defined analogously,
based on a fuzzy definition of $\widetilde{\cap}$.
\end{definition}

Expressions like \textit{\textquotedblleft all }$Y_{1}$ are $Y_{2}$ or $Y_{2}%
$\textquotedblright\ where $Y_{1},Y_{2},Y_{3}$ are crisp can be evaluated by
means of less arity quantifiers with these constructions:%

\[
\mathbf{all}\cup\left(  Y_{1},Y_{2},Y_{3}\right)  =\mathbf{all}\left(
Y_{1},Y_{2}\cup Y_{3}\right)
\]

The definition of the property is:

\begin{definition}
[\textbf{Internal meets property}]\label{DefPropInterUnion}\cite[pag.
97]{Glockner06Libro}\textbf{\ }Let $Q:\mathcal{P}\left(  E\right)
^{n}\rightarrow\mathbf{I}$ a semi-fuzzy quantifier, $n>0$. We will say a QFM
$\mathcal{F}$ preserves the property of internal meets if:%
\begin{align*}
\mathcal{F}\left(  Q\cup\right)   &  =\mathcal{F}\left(  Q\right)
\widetilde{\cup}\\
\mathcal{F}\left(  Q\cap\right)   &  =\mathcal{F}\left(  Q\right)
\widetilde{\cap}%
\end{align*}
\smallskip
\end{definition}

As a consequence,%
\begin{align*}
\mathcal{F}\left(  \mathbf{\exists}\right)  \left(  X_{1}\widetilde
{\mathbf{\cap}}X_{2}\right)   &  =\mathcal{F}\left(  \mathbf{\exists}\right)
\widetilde{\mathbf{\cap}}\left(  X_{1},X_{2}\right) \\
&  =\mathcal{F}\left(  \mathbf{\exists\cap}\right)  \left(  X_{1},X_{2}\right)
\\
&  =\mathcal{F}\left(  \text{\textbf{some}}\right)  \left(  X_{1}%
,X_{2}\right)
\end{align*}

\subsubsection{Monotonicity in arguments property (P.8)
\label{PropiedadMonotoniaArgumentos}}

In this section we present the property of monotonicity in arguments. This
property is one of the axioms used to define the DFSs.

\begin{definition}
[\textbf{Monotonicity}]\cite[pag. 98]{Glockner06Libro} A semi-fuzzy quantifier
$Q:\mathcal{P}\left(  E\right)  ^{n}\rightarrow\mathbf{I}$ is said to be
nondecreasing in its i-th argument, $i\in\left\{  1,\ldots,n\right\}  $ if%
\[
Q\left(  Y_{1},\ldots,Y_{i},\ldots,Y_{n}\right)  \leq Q\left(  Y_{1}%
,\ldots,Y_{i-1},Y_{i}^{\prime},Y_{i+1},\ldots,Y_{n}\right)
\]
whenever the involved arguments $Y_{1},\ldots,Y_{n},Y_{i}^{\prime}%
\in\mathcal{P}\left(  E\right)  $ satisfy $Y_{i}\subseteq Y_{i}^{\prime}$. $Q$
is said to be nonincreasing in the i-th argument if under the same conditions,
it always holds that
\[
Q\left(  Y_{1},\ldots,Y_{i},\ldots,Y_{n}\right)  \geq Q\left(  Y_{1}%
,\ldots,Y_{i-1},Y_{i}^{\prime},Y_{i+1},\ldots,Y_{n}\right)
\]
The corresponding definitions for fuzzy quantifiers $Q:\widetilde{\mathcal{P}%
}\left(  E\right)  ^{n}\rightarrow\mathbf{I}$ are entirely analogous. In this
case, the arguments range over $\widetilde{\mathcal{P}}\left(  E\right)  $,
and `$\subseteq$' is the usual fuzzy inclusion relation ($X_{1}\subseteq
X_{2}$ if $\mu_{X_{1}}\left(  e\right)  \leq\mu_{X_{2}}\left(  e\right)  $ for
all $e\in E$).
\end{definition}

For example, the semi-fuzzy quantifier $\mathbf{some}:\mathcal{P}\left(
E\right)  ^{2}\rightarrow\mathbf{I}$ is monotonic nondecreasing in both arguments.

The next property guarantees the extension of the monotonicity to fuzzy quantifiers:

\begin{definition}
[\textbf{Monotonicity property}]\label{DefPropMonotoniaArg}\cite[pag.
100]{Glockner06Libro}A QFM $\mathcal{F}$ is said to preserve monotonicity in
the arguments if semi-fuzzy quantifiers $Q:\mathcal{P}\left(  E\right)
^{n}\rightarrow\mathbf{I}$ which are nondecreasing (nonincreasing) in their
i-th argument $i\in\left\{  1,\ldots,n\right\}  $ are mapped to fuzzy
quantifiers $\mathcal{F}$ which are also nondecreasing (nonincreasing) in
their $i$-ih argument.
\end{definition}

For example, if a QFM $\mathcal{F}$ guarantees this property then
$\mathcal{F}\left(  \mathbf{some}\right)  :\widetilde{\mathcal{P}}\left(
E\right)  ^{2}\rightarrow\mathbf{I}$ is monotonic non-decreasing in both arguments.

\subsubsection{Monotonicity in quantifiers property
(P.9)\ \label{PropiedadMonotoniaCuantificadores}}

The \textit{property of monotonicity in quantifiers} is a very important
consequence of the DFS axioms \cite{Glockner97,Glockner06Libro}.
Independently, this property has also been defined in \cite[pag.
73]{Sanchez99},\cite{Delgado00} for unary quantifiers with the name of
\textit{property of inclusion of quantifiers}.

This property establishes that if a semi-fuzzy quantifier $Q$ is included in
other semi-fuzzy quantifier $Q^{\prime}$ (i.e., the results of $Q$ are smaller
than the results of $Q^{\prime}$ for all the selections of crisp arguments
$Y_{1},\ldots,Y_{n}\in\mathcal{P}\left(  E\right)  $) then the fuzzy extension
$\mathcal{F}\left(  Q\right)  $ is also included in $\mathcal{F}\left(
Q^{\prime}\right)  $.

\begin{definition}
[\textbf{Monotonicity in the quantifiers}]\cite[pag. 128]{Glockner06Libro}
Suppose $Q,Q^{\prime}:\mathcal{P}\left(  E\right)  ^{n}\rightarrow\mathbf{I}$
are semi-fuzzy quantifiers. Let us write $Q\leq Q^{\prime}$ if for all
$Y_{1},\ldots,Y_{n}\in\mathcal{P}\left(  E\right)  $, $Q\left(  Y_{1}%
,\ldots,Y_{n}\right)  \leq Q^{\prime}\left(  Y_{1},\ldots,Y_{n}\right)  $. On
fuzzy quantifiers we define $\leq$ analogously, based on arguments in
$\widetilde{\mathcal{P}}\left(  E\right)  $.
\end{definition}

For example, for the following semi-fuzzy quantifiers%
\begin{align}
Q\left(  X_{1},X_{2}\right)   &  =\left\{
\begin{array}
[c]{cc}%
S_{0.5,0.7}\left(  \frac{\left\vert X_{1}\cap X_{2}\right\vert }{\left\vert
X_{1}\right\vert }\right)  & X_{1}\neq\varnothing\\
1 & X_{1}=\varnothing
\end{array}
\right. \label{EqMonotoniaCuantificador_1}\\
Q^{\prime}\left(  X_{1},X_{2}\right)   &  =\left\{
\begin{array}
[c]{cc}%
S_{0.3,0.5}\left(  \frac{\left\vert X_{1}\cap X_{2}\right\vert }{\left\vert
X_{1}\right\vert }\right)  & X_{1}\neq\varnothing\\
1 & X_{1}=\varnothing
\end{array}
\right. \nonumber
\end{align}
it holds that $Q\leq Q^{\prime}$.

The next property is defined based on the Theorem 4.32 in \cite[pag.
128]{Glockner06Libro}.

\begin{definition}
[\textbf{Property of monotonicity in quantifiers}]%
\label{DefPropMonotoniaCuant}\textbf{\ }Suppose $\mathcal{F}$ is a QFM, and
$Q,Q^{\prime}:\mathcal{P}\left(  E\right)  ^{n}\rightarrow\mathbf{I}$ are
semi-fuzzy quantifiers. We say that $\mathcal{F}$ fulfills the property of
monotonicity in quantifiers if and only if $\mathcal{F}\left(  Q\right)
\leq\mathcal{F}\left(  Q^{\prime}\right)  $.
\end{definition}

This property guarantees that $\mathcal{F}\left(  Q\right)  \leq
\mathcal{F}\left(  Q^{\prime}\right)  $ for the semi-fuzzy quantifiers defined
on the expression \ref{EqMonotoniaCuantificador_1}.

\subsubsection{Property of functional application (P.10)}

The \textit{property of compatibility with functional application} forms part
of the axioms that are used to define the DFSs \cite{Glockner06Libro}. This
property requires that a QFM must be compatible with its induced extension principle.

\begin{definition}
[\textbf{Extension of a function to sets}]\label{DefPropCuantCuantitavos_1}
Let us consider $\beta:E\rightarrow S$ function. Function $\widehat{\beta
}:\mathcal{P}\left(  E\right)  \rightarrow\mathcal{P}\left(  S\right)  $ is
defined in the following way: $\widehat{\beta}\left(  Y\right)  =\left\{
\beta\left(  e\right)  :e\in Y\right\}  $.
\end{definition}

The extension principle induced by a QFM is defined as:

\begin{definition}
[Induced extension principle]\label{DefPropFuncAplPrinExtension}\cite[p\'{a}g.
101]{Glockner06Libro} All QFM $\mathcal{F}$ induce an extension principle
$\widehat{\mathcal{F}}$ that to each function $f:E\rightarrow E^{\prime}$
(where $E,E^{\prime}\neq\varnothing$) assigns a function $\widehat
{\mathcal{F}}\left(  f\right)  :\widetilde{\mathcal{P}}\left(  E\right)
\rightarrow\widetilde{\mathcal{P}}\left(  E^{\prime}\right)  $ defined by
$\mu_{\widehat{\mathcal{F}}\left(  f\right)  \left(  X\right)  }\left(
e^{\prime}\right)  =\mathcal{F}\left(  \chi_{\widehat{f}\left(  \cdot\right)
}\left(  e^{\prime}\right)  \right)  \left(  X\right)  $ for all
$X\in\widetilde{\mathcal{P}}\left(  E\right)  $, $e^{\prime}\in E^{\prime}$.
\end{definition}

It should be noted that in $\chi_{\widehat{f}\left(  \cdot\right)  }\left(
e^{\prime}\right)  $ the function $\widehat{f}:\mathcal{P}\left(  E\right)
\rightarrow\mathcal{P}\left(  E^{\prime}\right)  $ is the extension to sets of
the function $f$ and then $\chi_{\widehat{f}\left(  \cdot\right)  }\left(
e^{\prime}\right)  $ is the characteristic function of this extension; that
is, $\chi_{\widehat{f}\left(  \cdot\right)  }\left(  e^{\prime}\right)  $ is a
semi-fuzzy quantifier that for a set $Y\in\mathcal{P}\left(  E\right)  $
returns $1$ if $e^{\prime}\in\widehat{f}\left(  Y\right)  $ and $0$ in other case.

The property of compatibility with functional application is defined as:

\begin{proposition}
[Compatibility with functional application]\cite[P\'{a}g. 104]%
{Glockner06Libro} Let $\mathcal{F}$ a given QFM. We will say that
$\mathcal{F}$ is compatible with its induced extension principle if
$\mathcal{F}\left(  Q\circ\underset{i=1}{\overset{n}{\times}}\widehat{f_{i}%
}\right)  =\mathcal{F}\left(  Q\right)  \circ\underset{i=1}{\overset{n}%
{\times}}\widehat{\mathcal{F}}\left(  f_{i}\right)  $ or equivalently
\[
\mathcal{F}\left(  Q\circ\underset{i=1}{\overset{n}{\times}}\widehat
{\mathcal{F}}\left(  f_{i}\right)  \right)  \left(  X_{1}^{\prime}%
,\ldots,X_{n}^{\prime}\right)  =\mathcal{F}\left(  Q\right)  \left(
\widehat{\mathcal{F}}\left(  f_{1}\right)  \left(  X_{1}^{\prime}\right)
,\ldots,\widehat{\mathcal{F}}\left(  f_{n}\right)  \left(  X_{n}^{\prime
}\right)  \right)
\]
is valid for all semi-fuzzy quantifier $Q:\mathcal{P}\left(  E\right)
^{n}\rightarrow\mathbf{I}$ and all the function $f_{1},\ldots,f_{n}:E^{\prime
}\rightarrow E$ with domain $E^{\prime}\neq\varnothing$, $X_{1}^{\prime
},\ldots,X_{n}^{\prime}\in\widetilde{\mathcal{P}}\left(  E^{\prime}\right)  $.
\end{proposition}

That is, if a QFM $\mathcal{F}$ fulfills the property of functional
application, the same results are obtained when we first apply the induced
extension principle to the argument sets $X_{1}^{\prime},\ldots,X_{n}^{\prime
}\in\widetilde{\mathcal{P}}\left(  E^{\prime}\right)  $ and then we apply the
quantifier $\mathcal{F}\left(  Q\right)  $, and when we first apply the
semi-fuzzy quantifier $Q\circ\underset{i=1}{\overset{n}{\times}}\widehat
{f_{i}}$ (that to the crisp sets $Y_{1}^{\prime},\ldots,Y_{n}^{\prime}%
\in\mathcal{P}\left(  E^{\prime}\right)  $ apply the function $\underset
{i=1}{\overset{n}{\times}}\widehat{f_{i}}$, and then evaluates $Q:\mathcal{P}%
^{n}\left(  E\right)  \rightarrow\mathbf{I}$), and then we apply $\mathcal{F}$
to compute the function $\mathcal{F}\left(  Q\circ\underset{i=1}{\overset
{n}{\times}}\widehat{f_{i}}\right)  $ on $X_{1}^{\prime},\ldots,X_{n}^{\prime
}\in\widetilde{\mathcal{P}}\left(  E^{\prime}\right)  $.

This propery is very important in union with the rest of the axioms used to
define the QFMs because all toghether assures the fulfillment of a very
important and intuitive set of properties.

\subsection{The DFS axiomatic framework}

We now present the DFS axiomatic framework. In \cite{Glockner06Libro} the
author dedicates the whole 4 chapter to describe the properties that are
consequence of the axiomatic framework. For the sake of brevity, we have only
described the set of properties we have consider more relevant. Other
important properties the author describes in \cite{Glockner06Libro} are
argument permutations (the QFMs are compatible with the trasposition of
arguments), cylindrical extensions (that guaratees vacuous arguments are
irrelevant), quantitativity (QFMs guarantees that quantitative semi-fuzzy
quantifiers are mapped to quantitative fuzzy quantifiers), etc.

The framework the author sets out in \cite[section 3.9]{Glockner06Libro} is a
refinement of the original framework defined on \cite[pag. 22]{Glockner97}
that it was composed by 9 interdependent axioms. The two frameworks are
equivalent. We present now the definition of the DFS framework:

\begin{definition}
A QFM $\mathcal{F}$ is called a determiner fuzzification scheme (DFS) if the
following conditions are satisfied for all semi-fuzzy quantifiers
$Q:\mathcal{P}\left(  E\right)  ^{n}\rightarrow\mathbf{I}$.\newline%
\begin{tabular}
[c]{|l|l|l|}\hline
Correct generalisation & $\mathcal{U}\left(  \mathcal{F}\left(  Q\right)
\right)  =Q$\quad if $n\leq1$ & (Z-1)\\\hline
Projection quantifiers & $\mathcal{F}\left(  Q\right)  =\widetilde{\pi_{e}}%
$\quad if $Q=\pi_{e}$ for some $e\in E$ & (Z-2)\\\hline
Dualisation & $\mathcal{F}\left(  Q\widetilde{\square}\right)  =\mathcal{F}%
\left(  Q\right)  \widetilde{\square}$\quad$n>0$ & (Z-3)\\\hline
Internal joins & $\mathcal{F}\left(  Q\cup\right)  =\mathcal{F}\left(
Q\right)  \widetilde{\cup}$\quad$n>0$ & (Z-4)\\\hline
Preservation of monotonicity &
\begin{tabular}
[c]{l}%
If $Q$ is nonincreasing in the $n$-th arg, then\\
$\mathcal{F}\left(  Q\right)  $ is nonincreasing in $n$-th arg, $n>0$%
\end{tabular}
& (Z-5)\\\hline
Functional application &
\begin{tabular}
[c]{l}%
$\mathcal{F}\left(  Q\circ\underset{i=1}{\overset{n}{\times}}\widehat{f_{i}%
}\right)  =\mathcal{F}\left(  Q\right)  \circ\underset{i=1}{\overset{n}%
{\times}}\widehat{\mathcal{F}}\left(  f_{i}\right)  $\\
where $f_{1},\ldots,f_{n}:E^{\prime}\rightarrow E,E^{\prime}\neq\varnothing$%
\end{tabular}
& (Z-6)\\\hline
\end{tabular}

\end{definition}

In the previous definition $\mathcal{U}:\left(  \widetilde{Q}:\widetilde
{\mathcal{P}}\left(  E\right)  ^{n}\rightarrow\mathbf{I}\right)
\rightarrow\left(  Q:\mathcal{P}\left(  E\right)  ^{n}\rightarrow
\mathbf{I}\right)  $ is \textit{the underlying semi-fuzzy quantifier}
\cite[pag. 75]{Glockner06Libro}; that is, the semi-fuzzy quantifier
$Q:\mathcal{P}\left(  E\right)  ^{n}\rightarrow\mathbf{I}$ defined as:%
\[
\mathcal{U}\left(  \widetilde{Q}\right)  \left(  Y_{1},\ldots,Y_{n}\right)
=\widetilde{Q}\left(  Y_{1},\ldots,Y_{n}\right)
\]
for all crisp $Y_{1},\ldots,Y_{n}\in\mathcal{P}\left(  E\right)  $. The axiom
1 is equivalent to the fulfillment of the correct generalization property in
the unary case.

\subsection{Some properties that are not a consequence of the DFS axioms}

Now we will describe some adequacy properties that are not guaranteed by the
DFS framework because they impose an excesive restriction on the class of
plausible models. In \cite[chapter 6]{Glockner06Libro} a detailed exposition
considering these and other properties can be consulted.

\subsubsection{Property of continuity in arguments (P.11)}

Continuity properties are fundamental. Models that do not fulfil these
properties generally will not be valid from a practical viewpoint. One reason
is that it is impossible to avoid measure errors and, as a consequence, errors
in data measures could cause completely different analysis. Other reason is
that from a user viewpoint, it would be very difficult to understand why no
significant differences produce different results. Continutiy is also
necessary from an application view (for example, imagine we need to use fuzzy
quantifiers in a control system).

In this section we will explain the continuity in arguments property
\cite[Section 6.2]{Glockner06Libro}. The definition of this property is based
on the next metric to measure the difference between two pairs of fuzzy sets
$\left(  X_{1},\ldots,X_{n}\right)  ,\left(  X_{1}^{\prime},\ldots
,X_{n}^{\prime}\right)  \in\widetilde{\mathcal{P}}\left(  E\right)  $:

\begin{definition}
[$d\left(  \left(  X_{1},\ldots,X_{n}\right)  ,\left(  X_{1}^{\prime}%
,\ldots,X_{n}^{\prime}\right)  \right)  $]\cite[pag. 162]{Glockner06Libro} For
all base sets $E\neq\varnothing$ and all $n\in\mathbb{N}$ the metric
$d:\widetilde{\mathcal{P}}\left(  E\right)  ^{n}\times\widetilde{\mathcal{P}%
}\left(  E\right)  ^{n}\rightarrow\mathbf{I}$ is defined by%
\[
d\left(  \left(  X_{1},\ldots,X_{n}\right)  ,\left(  X_{1}^{\prime}%
,\ldots,X_{n}^{\prime}\right)  \right)  =\max_{i=1}^{n}\sup\left\{  \left\vert
\mu_{X_{i}}\left(  e\right)  -\mu_{X_{i}^{\prime}}\left(  e\right)  :e\in
E\right\vert \right\}
\]
for all $X_{1},\ldots,X_{n},X_{1}^{\prime},\ldots,X_{n}^{\prime}\in
\widetilde{\mathcal{P}}\left(  E\right)  $.
\end{definition}

Using this metric the property of continuity in arguments is defined:

\begin{definition}
[\textbf{Continuity in arguments property}]\label{DefPropContinuidadArg}%
\textbf{\ }\cite[pag. 163]{Glockner06Libro} We say that a QFM $\mathcal{F}$ is
arg-continuous if and only if $\mathcal{F}$ maps all semi-fuzzy quantifiers to
continuous fuzzy quantifiers $\mathcal{F}\left(  Q\right)  $; i.e. for all
$X_{1},\ldots,X_{n}\in\widetilde{\mathcal{P}}\left(  E\right)  $ and
$\varepsilon>0$ there exists $\delta>0$ such that $d\left(  \mathcal{F}\left(
Q\right)  \left(  X_{1},\ldots,X_{n}\right)  ,\mathcal{F}\left(  Q\right)
\left(  X_{1}^{^{\prime}},\ldots,X_{n}^{^{\prime}}\right)  \right)
<\varepsilon$ for all $X_{1}^{^{\prime}},\ldots,X_{n}^{\prime}\in
\mathcal{P}\left(  E\right)  $ with $d\left(  \left(  X_{1},\ldots
,X_{n}\right)  ,\left(  X_{1}^{^{\prime}},\ldots,X_{n}^{^{\prime}}\right)
\right)  <\delta$
\end{definition}

\subsubsection{Property of continuity in quantifiers (P.12)}

In the same way we require continuity on argument sets, we also require
continuity in quantifiers.\ That is, we do not expect big differences in
results when we modify slightly the quantifiers.

The distance between two semi-fuzzy quantifiers is defined as:

\begin{definition}
[$d\left(  Q,Q^{\prime}\right)  $]\cite[pag. 163]{Glockner06Libro} For all
semi-fuzzy quantifiers $Q,Q^{\prime}:\mathcal{P}\left(  E\right)
^{n}\rightarrow\mathbf{I}$ the distance between $Q$ and $Q^{\prime}$ is
defined as:%
\[
d\left(  Q,Q^{^{\prime}}\right)  =\sup\left\{  \left\vert Q\left(
Y_{1},\ldots,Y_{n}\right)  -Q^{^{\prime}}\left(  Y_{1},\ldots,Y_{n}\right)
\right\vert :Y_{1},\ldots,Y_{n}\in\mathcal{P}\left(  E\right)  ^{n}\right\}
\]
and similarity for all fuzzy quantifiers $\widetilde{Q},\widetilde{Q^{\prime}%
}:\widetilde{\mathcal{P}}\left(  E\right)  ^{n}\rightarrow\mathbf{I}$%
\[
d\left(  \widetilde{Q},\widetilde{Q^{\prime}}\right)  =\sup\left\{  \left\vert
\mathcal{F}\left(  Q\right)  \left(  X_{1},\ldots,X_{n}\right)  -\mathcal{F}%
\left(  Q^{^{\prime}}\right)  \left(  X_{1},\ldots,X_{n}\right)  \right\vert
:X_{1},\ldots,X_{n}\in\widetilde{\mathcal{P}}\left(  E\right)  \right\}
\]

\end{definition}

$Q$-continuity is defined as:

\begin{definition}
[\textbf{Continuity in quantifiers property}]\cite[pag. 163]{Glockner06Libro}%
\label{DefPropContinuidadCuant} We say that a QFM $\mathcal{F}$ is
$Q$-continuous if and only if for each semi-fuzzy quantifier $Q:\mathcal{P}%
\left(  E\right)  ^{n}\rightarrow\mathbf{I}$ and all $\varepsilon>0$, there
exists $\delta>0$ such that $d\left(  \mathcal{F}\left(  Q\right)
,\mathcal{F}\left(  Q^{^{\prime}}\right)  \right)  <\varepsilon$ whenever
$Q^{^{\prime}}:\mathcal{P}\left(  E\right)  ^{n}\rightarrow\mathbf{I}$
satisfies $d\left(  Q,Q^{^{\prime}}\right)  <\delta$.
\end{definition}

\subsubsection{Property of the fuzzy argument insertion (P.13)}

The property of fuzzy argument insertion is the fuzzy generalization of the
crisp argument insertion \cite[section 4.10]{Glockner06Libro}. Let
$Q:\mathcal{P}^{n}\left(  E\right)  \rightarrow\mathbf{I}$ a semi-fuzzy
quantifier $n>0$, and $A\in\mathcal{P}\left(  E\right)  $. By
$Q\vartriangleleft A:\mathcal{P}^{n-1}\left(  E\right)  \rightarrow\mathbf{I}$
we will denote the semi-fuzzy quantifier defined as
\[
Q\vartriangleleft A\left(  Y_{1},\ldots,Y_{n-1}\right)  =Q\left(  Y_{1}%
,\ldots,Y_{n-1},A\right)
\]
for all $Y_{1},\ldots,Y_{n-1}\in\mathcal{P}\left(  E\right)  $. As a
consequence of the DFS axioms it is fulfilled that
\[
\mathcal{F}\left(  Q\vartriangleleft A\right)  =\mathcal{F}\left(  Q\right)
\vartriangleleft A
\]
for all semi-fuzzy quantifier $Q$ of arity $n>0$, and all crisp $A\in
\mathcal{P}\left(  E\right)  $.

Fuzzy argument insertion cannot be modeled directly, because a semi-fuzzy
quantifier $Q:\mathcal{P}^{n}\left(  E\right)  \rightarrow\mathbf{I,}$ $n>0$
only accepts crisp arguments; that is, for all $A\in\widetilde{\mathcal{P}%
}\left(  E\right)  $ fuzzy only $\mathcal{F}\left(  Q\right)  \vartriangleleft
A$ is defined and no $Q\vartriangleleft A$. But as is explained in
\cite[secci\'{o}n 6.8]{Glockner06Libro}, a QFM $\mathcal{F}$ and a semi-fuzzy
quantifier $Q:\mathcal{P}^{n}\left(  E\right)  \rightarrow\mathbf{I}$ we can
study if there exists a semi-fuzzy quantifier $Q^{\prime}:\mathcal{P}%
^{n-1}\left(  E\right)  \rightarrow\mathbf{I}$ fulfilling
\begin{equation}
\mathcal{F}\left(  Q\right)  \vartriangleleft A=\mathcal{F}\left(  Q^{\prime
}\right)  \label{EqInsercionArgumentosBorrosa_1}%
\end{equation}
for all $A\in\widetilde{\mathcal{P}}\left(  E\right)  $.

The reasonable election $Q^{\prime}$ is the following:

\begin{definition}
\cite[pag. 172]{Glockner06Libro}Let $\mathcal{F}$ a QFM, $Q:\mathcal{P}\left(
E\right)  ^{n+1}\rightarrow\mathbf{I}$ a semi-fuzzy quantifier and
$A\in\widetilde{\mathcal{P}}\left(  E\right)  $ a fuzzy set. Then
$Q\widetilde{\vartriangleleft}A:\mathcal{P}\left(  E\right)  ^{n}%
\rightarrow\mathbf{I}$ is defined as%
\[
Q\widetilde{\vartriangleleft}A=\mathcal{U}\left(  \mathcal{F}\left(  Q\right)
\vartriangleleft A\right)
\]
that is, $Q\widetilde{\vartriangleleft}A\left(  Y_{1},\ldots,Y_{n}\right)
=\mathcal{F}\left(  Q\right)  \left(  Y_{1},\ldots,Y_{n},A\right)  $ for all
crisp sets $Y_{1},\ldots,Y_{n}\in\mathcal{P}\left(  E\right)  $.
\end{definition}

In \cite[secci\'{o}n 6.8]{Glockner06Libro} the author mentions $Q^{\prime
}=Q\widetilde{\vartriangleleft}A$ is the unique election fo $Q^{\prime}$ that
could satisfy \ref{EqInsercionArgumentosBorrosa_1}. It should be noted that if
$Q^{\prime}$ satisfies $\mathcal{F}\left(  Q\right)  \vartriangleleft
A=\mathcal{F}\left(  Q^{\prime}\right)  $ then also satisfies%
\[
Q^{\prime}=\mathcal{U}\left(  \mathcal{F}\left(  Q^{\prime}\right)  \right)
=\mathcal{U}\left(  \mathcal{F}\left(  Q\right)  \vartriangleleft A\right)
=Q\widetilde{\vartriangleleft}A
\]

The next property resumes the fulfillent of the fuzzy argument insertion in
the fuzzy case:

\begin{definition}
\label{DefPropInserArgBorrosa}\cite[pag. 172]{Glockner06Libro}Let
$\mathcal{F}$ be a QFM. We will say $\mathcal{F}$ fulfills fuzzy argument
insertion if for all semi-fuzzy quantifier $Q:\mathcal{P}\left(  E\right)
^{n}\rightarrow\mathbf{I}$ of artity $n>0$ and all $A\in\widetilde
{\mathcal{P}}\left(  E\right)  $ fuzzy is fulfilled%
\[
\mathcal{F}\left(  Q\right)  \vartriangleleft A=\mathcal{F}\left(
Q\widetilde{\vartriangleleft}A\right)
\]

\end{definition}

This property has a very strong relation with \textit{nested quantification}.
Althoug the sufficiency of this property for a DFS to adequate model nested
quantifiers, in \cite[section 12.6]{Glockner06Libro} the author has state the
necessity of fulfilling this property. Moreover, the fulfillment of this
property for standard DFSs is only achieved by the $\mathcal{M}_{CX}$, a
paradigmatic example of good theoretical behavior.

\subsection{Some probabilistic properties}

Now, we will present two properties of probabilistic nature that are fulfilled
by a number of probabilistic models
\cite{Delgado00,DiazHermida02-FuzzySets,DiazHermida06Tesis}.

\subsubsection{Property of averaging for the identity quantifier
(P.14)\label{SubSubSubPropMedia}}

The fulfillment of this property for a QFM $\mathcal{F}$ assures that when we
apply the model to the unary semi-fuzzy quantifier \textbf{identity}%
$:\mathcal{P}\left(  E\right)  \rightarrow\mathbf{I}$ we obtain the average of
the membership grades. First of all, the definition of this semi-fuzzy
quantifer is:

\begin{definition}
The unary semi-fuzzy quantifier \textbf{identity}$:\mathcal{P}\left(
E\right)  \rightarrow\mathbf{I}$\textit{ is defined as}%
\[
\mathbf{identity}\left(  Y\right)  =\frac{\left\vert Y\right\vert }{\left\vert
E\right\vert },Y\in\mathcal{P}\left(  E\right)
\]

\end{definition}

It should be noted that for the \textbf{identity} semi-fuzzy quantifier the
addition of one element improves the result in $\frac{1}{m}$. That is, the
improvement obtained with the addition of elements to the argument set is
linear. We can interpret the meaning of this semi-fuzzy quantifier as
\textquotedblleft as many as possible\textquotedblright.

The definition of the property is:

\begin{definition}
[Property of averaging for the identity quantifier]We will say that a QFM
$\mathcal{F}$ fulfills the property of averaging for the identity quantifier
if:%
\[
\mathcal{F}\left(  \mathbf{identity}\right)  \left(  X\right)  =\frac{1}%
{m}\sum_{j=1}^{m}\mu_{X}\left(  e_{j}\right)
\]

\end{definition}

As a result of the fulfillment of the property of averaging for the identity
quantifier, the improvement obtained in $\mathcal{F}^{A}\left(
\mathbf{identity}\right)  \left(  X\right)  $ is linear with respect to the
increase of the membership grades of the argument fuzzy set.

\subsubsection{Property of the probabilistic interpretation of quantifiers
(P.15)\label{SubSubSubPropRecubProbab}}

Let us suppose we use a set of semi-fuzzy quantifiers
(\textit{\textquotedblleft at most about 20\%\textquotedblright},
\textit{\textquotedblleft about between 20\% and 80\%\textquotedblright},
\textit{\textquotedblleft at least about 80\%\textquotedblright}) to split the
quantification universe. Then, if semi-fuzzy quantifiers can be interpreted in
a probabilistic way, the fulfillment of this property guarantees that fuzzy
quantifiers also can be interpreted in a probabilistic way.

\begin{definition}
We will say that a set of semi-fuzzy quantifiers $Q_{1},\ldots,Q_{r}%
:\mathcal{P}^{n}\left(  E\right)  \rightarrow\mathbf{I}$ forms a probabilistic
Ruspini partition of the quantification universe if for all $Y_{1}%
,\ldots,Y_{n}\in\mathcal{P}\left(  E\right)  $ it holds that%
\[
Q_{1}\left(  Y_{1},\ldots,Y_{n}\right)  +\ldots+Q_{r}\left(  Y_{1}%
,\ldots,Y_{n}\right)  =1
\]

\end{definition}

\begin{example}
The next set of quantifiers forms a probabilistic Ruspini partition of the
quantification universe:
\end{example}

\begin{align}
\text{\textbf{at most about }}20\%\left(  Y_{1},Y_{2}\right)   &  =\left\{
\begin{array}
[c]{cc}%
T_{-\infty,0,0.2,0.4}\left(  \frac{\left\vert Y_{1}\cap Y_{2}\right\vert
}{\left\vert Y_{1}\right\vert }\right)  & Y_{1}\neq\varnothing\\
\frac{1}{3} & Y_{1}=\varnothing
\end{array}
\right. \label{ParticionCuantificadores}\\
\text{\textbf{about between }}20\%\text{ \textbf{and} }80\%\left(  Y_{1}%
,Y_{2}\right)   &  =\left\{
\begin{array}
[c]{cc}%
T_{0.2,0.4,0.6,0.8}\left(  \frac{\left\vert Y_{1}\cap Y_{2}\right\vert
}{\left\vert Y_{1}\right\vert }\right)  & Y_{1}\neq\varnothing\\
\frac{1}{3} & Y_{1}=\varnothing
\end{array}
\right. \nonumber\\
\text{\textbf{at least about }}80\%\left(  Y_{1},Y_{2}\right)   &  =\left\{
\begin{array}
[c]{cc}%
T_{0.6,0.8,1,\infty}\left(  \frac{\left\vert Y_{1}\cap Y_{2}\right\vert
}{\left\vert Y_{1}\right\vert }\right)  & Y_{1}\neq\varnothing\\
\frac{1}{3} & Y_{1}=\varnothing
\end{array}
\right. \nonumber
\end{align}
because%
\[
\text{\textbf{at most about }}20\%\left(  Y_{1},Y_{2}\right)
+\text{\textbf{about between }}20\%\text{ \textbf{and} }80\%\left(
Y_{1},Y_{2}\right)  +
\]%
\[
\text{\textbf{at least about }}80\%\left(  Y_{1},Y_{2}\right)  =1
\]
for all $Y_{1},Y_{2}\mathcal{P}\left(  E\right)  $.

\begin{definition}
[\textbf{Property of probabilistic interpretation of quantifiers}]We will say
that a QFM $\mathcal{F}$ fulills the property of probabilistic interpretation
of quantifiers if for all probabilistic Ruspini partitions of the
quantification universe $Q_{1},\ldots,Q_{r}:\mathcal{P}\left(  E\right)
^{n}\rightarrow\mathbf{I}$ it holds that
\[
\mathcal{F}\left(  Q_{1}\right)  \left(  X_{1},\ldots,X_{n}\right)
+\ldots+\mathcal{F}\left(  Q_{r}\right)  \left(  X_{1},\ldots,X_{n}\right)
=1
\]

\end{definition}

This property is very interesting because let us to interpret the result of
evaluating a fuzzy expression as a probability distribution on the labels
related to the quantifiers.\footnote{In \cite{Lawry01-Computing} a
probabilistic interpretation of quantifiers is also used under the label
semantics interpretation of fuzzy sets.
\par
{}}

\section{Probabilistic interpretation of fuzzy sets based on likelihood
functions\label{SubSecInterpretacionVerosimilitudes}}

In this section we use the interpretation of fuzzy sets based on likelihood
functions to establish the necessary background to define the $\mathcal{F}%
^{A}$ model. In \cite{DiazHermida02-FuzzySets,DiazHermida06Tesis} another
probabilistic view of fuzzy sets have been used to define a probabilistic
framework for the definition of QFMs, and some models in this framework have
been presented.

The semantic interpretation of fuzzy sets based on likelihood functions
\cite{Mabuchi92,Thomas95,Tursken2000Fundamentals,Dubois2000Fundamentals}
interprets vagueness in the data as a consequence of making a random
experiment in which a set of individuals are asked about the fulfillment of a
certain property. Let us consider the following example:

\begin{example}
\label{EjemInterpret_1}To decide if the height value $185cm.$ is considered
\textquotedblleft tall for male adults\textquotedblright\ a random experiment
is performed in which four individuals (henceforth voters) are asked about
their opinion. Let us denote by $P$ the statement \textquotedblleft the value
$185cm.$ is tall for male adults\textquotedblright, by $V=\left\{  v_{1}%
,v_{2},v_{3},v_{4}\right\}  $ the set of voters and by $C\left(  v,P\right)
\in\mathbf{2=}\left\{  0,1\right\}  $, $v\in V$ the answer for each voter. If%
\[
C\left(  v_{1},P\right)  =1,C\left(  v_{2},P\right)  =0\mathbf{,}C\left(
v_{3},P\right)  =1,C\left(  v_{4},P\right)  =1
\]
then we can define the degree of fulfillment of the statement $P$ as%
\[
\mu\left(  P\right)  =\frac{\left\vert v\in V:C\left(  v,P\right)
=1\right\vert }{\left\vert V\right\vert }=\frac{3}{4}%
\]

\end{example}

The above experiment can be extended to the height values of the universe. Let
be $h\in\mathbb{R}$. We can define the degree of fulfillment of the statement
\textit{\textquotedblleft the value of height }$h$\textit{\ is
tall\textquotedblright} as:%
\[
\mu\left(  \text{\textquotedblleft}h\text{ is }tall\text{\textquotedblright%
}\right)  =\Pr\left(  \text{\textquotedblleft}h\text{ is }%
tall\text{\textquotedblright}\right)  =\frac{\left\vert v\in V:C\left(
v,\text{\textquotedblleft}h\text{ is }tall\text{\textquotedblright}\right)
=1\right\vert }{\left\vert V\right\vert }%
\]
In this way we can assign a degree of fulfillment to the reference universe.
In the common notation of fuzzy sets we assign to the label
\textit{\textquotedblleft tall\textquotedblright\ }the fuzzy set
$tall\in\widetilde{\mathcal{P}}\left(  \mathbb{R}\right)  $ defined as:
$\mu_{tall}\left(  h\right)  =\mu\left(  \text{\textquotedblleft}h\text{ is
}tall\text{\textquotedblright}\right)  $.

Under this view of fuzzy sets, $\mu_{tall}\left(  h\right)  >\mu_{tall}\left(
h^{\prime}\right)  $ indicates that is more probable\ that
\textit{\textquotedblleft}$h$\textit{ is tall\textquotedblright}\ than
\textit{\textquotedblleft}$h^{\prime}$\textit{ is tall\textquotedblright}.

One of the accepted suppositions of this view is to assume that the answer of
one voter for a certain value $h\in\mathbb{R}$ does not constrain his answer
for other element $h^{\prime}$\footnote{The situation in which the answer of
one voter for a value of the universe constrains its answer for other values
is related to the interpretation of fuzzy sets based on random sets. This view
is used in \cite{DiazHermida02-FuzzySets,DiazHermida06Tesis} for proposing a
probabilistic framework to define models of fuzzy quantification.}. Let us
suppose that the universe $E$ is finite. As we are interpreting that $\mu
_{X}\left(  e\right)  =\Pr\left(  \text{\textquotedblleft}e\text{ is
}X\text{\textquotedblright}\right)  $ then under the independence assumption
we have:%
\begin{equation}
\Pr\left(  \text{\textquotedblleft}e\text{ is }X\text{\textquotedblright%
}\wedge\text{\textquotedblleft}e^{\prime}\text{ is }X\text{\textquotedblright%
}\right)  =\Pr\left(  \text{\textquotedblleft}e\text{ is }%
X\text{\textquotedblright}\right)  \cdot\Pr\left(  \text{\textquotedblleft%
}e^{\prime}\text{ is }X\text{\textquotedblright}\right)  =\mu_{X}\left(
e\right)  \cdot\mu_{X}\left(  e^{\prime}\right) \nonumber
\end{equation}

We can apply the same idea to compute the probability that a crisp set
$Y\in\mathcal{P}\left(  E\right)  $ was a representative of a fuzzy set
$X\in\widetilde{\mathcal{P}}\left(  E\right)  $ when we suppose the base set
$E$ finite. The intuition is that this probability is the probability that
only the elements in $Y$ belongs to $X$:

\begin{definition}
[$\Pr\left(  representative_{X}=Y\right)  $]\label{DefInterpretProbConj}Let
$X\in\widetilde{\mathcal{P}}\left(  E\right)  $ be a fuzzy set, $E$ finite.
The probability of the crisp set $Y\in\mathcal{P}\left(  E\right)  $ to be a
representative of the fuzzy set $X\in\widetilde{\mathcal{P}}\left(  E\right)
$ is defined as%
\[
\Pr\left(  representative_{X}=Y\right)  =%
{\displaystyle\prod\limits_{e\in Y}}
\mu_{X}\left(  e\right)
{\displaystyle\prod\limits_{e\in E\backslash Y}}
\left(  1-\mu_{X}\left(  e\right)  \right)
\]

\end{definition}

It should be pointed out that in the previous definition the probability
points are the subsets of $E$. In this way the $\sigma$-algebra on which the
probability is defined is $\mathcal{P}\left(  E\right)  $.

It is worthy to note that definition \ref{DefInterpretProbConj}\ can be
explained without mention to probability theory. If we consider the product
tnorm ($\wedge\left(  x_{1},x_{2}\right)  =x_{1}\cdot x_{2}$) and the
Lukasiewicz implication then $\Pr\left(  representative_{X}=Y\right)  $ is the
\textit{equipotence} \cite{Bandler80}\ between $Y$ and $X$:%
\[
Eq\left(  Y,X\right)  =\wedge_{e\in E}\left(  \mu_{X}\left(  e\right)
\rightarrow\mu_{Y}\left(  e\right)  \right)  \wedge\left(  \mu_{Y}\left(
e\right)  \rightarrow\mu_{X}\left(  e\right)  \right)
\]

If $e\in E$ then $\mu_{Y}\left(  e\right)  =1$ and%
\begin{align*}
\left(  \mu_{X}\left(  e\right)  \rightarrow\mu_{Y}\left(  e\right)  \right)
\wedge\left(  \mu_{Y}\left(  e\right)  \rightarrow\mu_{X}\left(  e\right)
\right)   &  =\left(  \mu_{X}\left(  e\right)  \rightarrow1\right)
\wedge\left(  1\rightarrow\mu_{X}\left(  e\right)  \right) \\
&  =1\wedge\left(  1-1+\mu_{X}\left(  e\right)  \right) \\
&  =\mu_{X}\left(  e\right)
\end{align*}

If $e\notin E$ then $\mu_{Y}\left(  e\right)  =0$ and%
\begin{align*}
\left(  \mu_{X}\left(  e\right)  \rightarrow\mu_{Y}\left(  e\right)  \right)
\wedge\left(  \mu_{Y}\left(  e\right)  \rightarrow\mu_{X}\left(  e\right)
\right)   &  =\left(  \mu_{X}\left(  e\right)  \rightarrow0\right)
\wedge\left(  0\rightarrow\mu_{X}\left(  e\right)  \right) \\
&  =\left(  1-\mu_{X}\left(  e\right)  \right)  \wedge1\\
&  =1-\mu_{X}\left(  e\right)
\end{align*}

And then:%
\begin{align*}
Eq\left(  Y,X\right)   &  =\left(  \wedge_{e\in E}\mu_{X}\left(  e\right)
\right)  \wedge\left(  \wedge_{e\notin E}\left(  1-\mu_{X}\left(  e\right)
\right)  \right) \\
&  =%
{\displaystyle\prod\limits_{e\in Y}}
\mu_{X}\left(  e\right)
{\displaystyle\prod\limits_{e\in E\backslash Y}}
\left(  1-\mu_{X}\left(  e\right)  \right)
\end{align*}

The next notation will be used for $\Pr\left(  \text{representative}%
_{X}=Y\right)  $ in the rest of the paper:

\begin{notation}
[$m_{X}\left(  Y\right)  $]\label{NotacionVerosimilitudes_1}Let $X\in
\widetilde{\mathcal{P}}\left(  E\right)  $ be a fuzzy set and $Y\in
\mathcal{P}\left(  E\right)  $ a crisp set. We will denote $m_{X}\left(
Y\right)  =\Pr\left(  representative_{X}=Y\right)  $.
\end{notation}

Let us see now an example in which this probability is computed:

\begin{example}
Let be $E=\left\{  e_{1},e_{2},e_{3}\right\}  $ and $X\in\widetilde
{\mathcal{P}}\left(  E\right)  $ the fuzzy set defined as: $X=\left\{
0.8/e_{1},0.2/e_{2},0.6/e_{3}\right\}  $. Then%
\begin{align*}
m_{X}\left(  \left\{  e_{1},e_{3}\right\}  \right)   &  =%
{\displaystyle\prod\limits_{e\in Y}}
\mu_{X}\left(  e\right)
{\displaystyle\prod\limits_{e\in E\backslash Y}}
\left(  1-\mu_{X}\left(  e\right)  \right) \\
&  =\mu_{X}\left(  e_{1}\right)  \times\mu_{X}\left(  e_{3}\right)
\times\left(  1-\mu_{X}\left(  e_{2}\right)  \right)  =0.384
\end{align*}

\end{example}

Sometimes, we will need to restrict the probability of a fuzzy set
$X^{E^{\prime}}\in\widetilde{\mathcal{P}}\left(  E\right)  $ to a subset
$E^{\prime}\subseteq E$ of the referential. Let $X^{E^{\prime}}\in
\widetilde{\mathcal{P}}\left(  E^{\prime}\right)  $ be the projection of $X$
in $E^{\prime}$; that is, $X^{\prime}$ is the fuzzy set on $E^{\prime}$
defined as: $\mu_{X^{\prime}}\left(  e\right)  =\mu_{X}\left(  e\right)  ,e\in
E^{\prime}$.

In this case, the probability of $X^{\prime}$ on $E^{\prime}$ is denoted as:

\begin{notation}
[$m_{X}^{E^{\prime}}\left(  Y\right)  $]Let $X\in\widetilde{\mathcal{P}%
}\left(  E\right)  $ be a fuzzy set and $E^{\prime}\subseteq E$ a restriction
of the base set, and $Y\in\mathcal{P}\left(  E^{\prime}\right)  $ a crisp set
on $E^{\prime}$. We will denote $m_{X}^{E^{\prime}}\left(  Y\right)
=m_{X^{\prime}}\left(  Y\right)  $ where $X^{\prime}$ is the projection of $X$
on $E^{\prime}$; that is, the fuzzy set defined as $\mu_{X^{\prime}}\left(
e\right)  =\mu_{X}\left(  e\right)  ,e\in E^{\prime}$.
\end{notation}

It should be noted that%
\begin{align*}
\sum_{Y\in\mathcal{P}\left(  E\right)  |e\in Y}m_{X}\left(  Y\right)   &
=\sum_{\left\{  e\right\}  \subseteq Y\subseteq E}m_{X}\left(  Y\right)
=\sum_{\left\{  e\right\}  \subseteq Y\subseteq E}\mu_{X}\left(  e\right)
m_{X}^{E\backslash\left\{  e\right\}  }\left(  Y\backslash\left\{  e\right\}
\right) \\
&  =\mu_{X}\left(  e\right)  \sum_{\varnothing\subseteq Y\subseteq
E\backslash\left\{  e\right\}  }m_{X}^{E\backslash\left\{  e\right\}  }\left(
Y\right)  =\mu_{X}\left(  e\right)
\end{align*}

We consider now the situation in which we want to compute the probability of
two crisp sets $Y_{1},Y_{2}\in\mathcal{P}\left(  E\right)  $ to be
respectively the representatives of the fuzzy sets $X_{1},X_{2}\in
\widetilde{\mathcal{P}}\left(  E\right)  $. That is, we consider the
computation of the probability of the event \textquotedblleft%
$representative_{X_{1}}=Y_{1}\wedge representative_{X_{2}}=Y_{2}%
$\textquotedblright.

If the two fuzzy sets $X_{1}$, $X_{2}$ are related to different reference
universes (i.e., intelligence and height) it is reasonable to suppose that the
probability of $Y_{1}$ to be representative of $X_{1}$ is independent of the
probability of $Y_{2}$ to be representative of $X_{2}$\footnote{In the work
\cite{Thomas95} is analyzed deeply the interpretation of fuzzy sets based on
likelihood functions. In \cite[P\'{a}g. 95]{Thomas95} the author argue that
the most reasonable is to suppose independence between different universes.}.
It should be noted that once we have assumed independency between the voter
decisions for elements related to same property is natural to assume
independency for elements related to different properties. Then we define:

\begin{definition}
[$\Pr\left(  representative_{X_{1}}=Y_{1}\wedge representative_{X_{2}}%
=Y_{2}\right)  $]\label{DefInterpretProbConj_2}Let $X_{1},X_{2}\in
\widetilde{\mathcal{P}}\left(  E\right)  $ be fuzzy sets, $Y_{1},Y_{2}%
\in\widetilde{\mathcal{P}}\left(  E\right)  $ crisp sets, and $E$ a finite
referential. Under the independence assumption for properties the probability
that $Y_{1}$ to be a representative of $X_{1}$ and $Y_{2}$ to be a
representative of $X_{2}$ is:%
\[
\Pr\left(  representative_{X_{1}}=Y_{1}\wedge representative_{X_{2}}%
=Y_{2}\right)  =m_{X_{1}}\left(  Y_{1}\right)  \cdot m_{X_{2}}\left(
Y_{2}\right)
\]
where we have assumed that the sets $X_{1}$ and $X_{2}$ are based on
independent concepts.
\end{definition}

In the definition of the model $\mathcal{F}^{A}$ we assumme always the
independence hypothesis. Even this could not seem appropiate in some cases
\cite{Mabuchi92,Thomas95} \cite{Mabuchi92,Thomas95}, this hypothesis simplify
considerably the definition of the models, it allows us a relative
straigforwardly algebraic manipulation, and the definition of efficient algorithms.

\section{The QFM $\mathcal{F}^{A}$\label{SubSectionModeloFA}}

In this section we define the finite QFM $\mathcal{F}^{A}$
\cite{DiazHermida04IPMU,DiazHermida04-IEEE,DiazHermida06Tesis}. This model is
based on the probabilistic interpretation of fuzzy sets previously explained.

Using expressions \ref{DefInterpretProbConj} and \ref{DefInterpretProbConj_2}
the definition of the QFM $\mathcal{F}^{A}$ is easily made:

\begin{definition}
[$\mathcal{F}^{A}$]\cite[pag. 1359]{DiazHermida04IPMU}Let $Q:\mathcal{P}%
\left(  E\right)  ^{n}\rightarrow\mathbf{I}$ be a semi-fuzzy quantifier, $E$
finite. The QFM $\mathcal{F}^{A}$ is defined as%
\begin{equation}
\mathcal{F}^{A}\left(  Q\right)  \left(  X_{1},\ldots,X_{n}\right)
=\sum_{Y_{1}\in\mathcal{P}\left(  E\right)  }\ldots\sum_{Y_{n}\in
\mathcal{P}\left(  E\right)  }m_{X_{1}}\left(  Y_{1}\right)  \ldots m_{X_{n}%
}\left(  Y_{n}\right)  Q\left(  Y_{1},\ldots,Y_{n}\right)
\label{ModeloVerosimilitudes}%
\end{equation}
for all $X_{1},\ldots,X_{n}\in\widetilde{\mathcal{P}}\left(  E\right)  $.
\end{definition}

In expression \ref{ModeloVerosimilitudes} we are assuming that the probability
of being $Y_{i}$ a representative of the fuzzy set $X_{i}$ is independent of
the probability of being $Y_{j}$ a representative of the fuzzy set $X_{j}$ for
$i\neq j$. $\mathcal{F}^{A}\left(  Q\right)  \left(  X_{1},\ldots
,X_{n}\right)  $ can be interpreted as the average opinion of voters.

The next expression is an alternative definition of the model $\mathcal{F}%
^{A}$:%
\[
\mathcal{F}^{A}\left(  X_{1},\ldots,X_{n}\right)  =%
{\displaystyle\bigvee\limits_{Y_{1}\in\mathcal{P}\left(  E\right)  }}
\ldots%
{\displaystyle\bigvee\limits_{Y_{n}\in\mathcal{P}\left(  E\right)  }}
Eq\left(  Y_{1},X_{1}\right)  \wedge\ldots\wedge Eq\left(  Y_{n},X_{n}\right)
\wedge Q\left(  Y_{1},\ldots,Y_{n}\right)
\]
where $\vee$ the Lukasiewicz tconorm ($\vee\left(  x_{1},x_{2}\right)
=\min\left(  x_{1}+x_{2},1\right)  $), $\wedge$ is the product tnorm
($\wedge\left(  x_{1},x_{2}\right)  =x_{1}\cdot x_{2}$) and $Eq\left(
Y,X\right)  $ is the equipotence between the crisp set $Y$ and the fuzzy set
$X$, that it was defined in previous section. In this way, the model can be
defined without mention to probability theory.

The following example shows the application of the QFM $\mathcal{F}^{A}$:

\begin{example}
\label{EjemVerosimilitudes}Let us consider the sentence%
\[
\text{\textquotedblleft Nearly all the intelligent workers are well
paid\textquotedblright}%
\]
where the semi-fuzzy quantifier $Q=$\textbf{\textquotedblleft nearly
all\textquotedblright}, and the fuzzy sets \textbf{\textquotedblleft
intelligent workers\textquotedblright} and \textquotedblleft\textbf{well
paid\textquotedblright} take the following values:%
\begin{align*}
\mathbf{intelligent}\text{ }\mathbf{workers}  &  =\left\{  0.8/e_{1}%
,0.9/e_{2},1/e_{3},0.2/e_{4}\right\} \\
\mathbf{well\ paid}  &  =\left\{  1/e_{1},0.8/e_{2},0.3/e_{3},0.1/e_{4}%
\right\} \\
Q\left(  X_{1},X_{2}\right)   &  =\left\{
\begin{array}
[c]{cc}%
\max\left\{  2\left(  \frac{\left\vert X_{1}\cap X_{2}\right\vert }{\left\vert
X_{1}\right\vert }\right)  -1,0\right\}  & X_{1}\neq\varnothing\\
1 & X_{1}=\varnothing
\end{array}
\right.
\end{align*}
We compute the probabilities of the representatives of the fuzzy sets
\textbf{\textquotedblleft intelligent workers\textquotedblright} and
\textquotedblleft\textbf{well paid\textquotedblright}:%
\begin{align*}
m_{\mathbf{intelligent}\text{ }\mathbf{workers}}\left(  \varnothing\right)
&  =\left(  1-0.8\right)  \left(  1-0.9\right)  \left(  1-1\right)  \left(
1-0.2\right)  =0\\
m_{\mathbf{intelligent}\text{ }\mathbf{workers}}\left(  \left\{
e_{1}\right\}  \right)   &  =0.8\left(  1-0.9\right)  \left(  1-1\right)
\left(  1-0.2\right)  =0\\
&  \ldots\\
m_{\mathbf{intelligent}\text{ }\mathbf{workers}}\left(  \left\{  e_{1}%
,e_{2},e_{3},e_{4}\right\}  \right)   &  =0.8\cdot0.9\cdot1\cdot0.2=0.144\\
m_{\mathbf{well\ paid}}\left(  \varnothing\right)   &  =\left(  1-1\right)
\left(  1-0.8\right)  \left(  1-0.3\right)  \left(  1-0.1\right)  =0\\
m_{\mathbf{well\ paid}}\left(  \left\{  e_{1}\right\}  \right)   &
=1\cdot\left(  1-0.8\right)  \left(  1-0.3\right)  \left(  1-0.1\right)
=0.126\\
&  \ldots\\
m_{\mathbf{well\ paid}}\left(  \left\{  e_{1},e_{2},e_{3},e_{4}\right\}
\right)   &  =0.8\cdot0.9\cdot1\cdot0.2=0.144
\end{align*}
And using expression \ref{ModeloVerosimilitudes}:
\begin{align*}
&  \mathcal{F}^{A}\left(  Q\right)  \left(  \mathbf{intelligent}\text{
}\mathbf{workers},\mathbf{well\ paid}\right) \\
&  =\sum_{Y_{1}\in\mathcal{P}\left(  E\right)  }\sum_{Y_{2}\in\mathcal{P}%
\left(  E\right)  }m_{X_{1}}\left(  Y_{1}\right)  m_{X_{2}}\left(
Y_{2}\right)  Q\left(  Y_{1},Y_{2}\right)  =0.346
\end{align*}

\end{example}

The QFM $\mathcal{F}^{A}$\ fulfills the DFS axiomatic framework; that is, the
QFM $\mathcal{F}^{A}$ is a finite DFS. Moreover, the QFM $\mathcal{F}^{A}$
fulfills the additional properties we have presented in this technical
report\footnote{With respect to the additional properties defined in
\cite[Chapter 6]{Glockner06Libro}. we mention that the $\mathcal{F}^{A}$ model
does not fulfill the \textquotedblleft propagation of fuzziness
property\textquotedblright\ (basically because is not fulfilled by the induced
operators).\ Other reasonable properties, as conservativity, are in
contradiction with the DFS framework, and cannot be fulfilled.}. The analysis
of properties of the model is made in the appendix.

\section{Some additional results about the $\mathcal{F}^{A}$ model}

\subsection{Limit case approximation of the $\mathcal{F}^{A}$ model}

In this section we will prove that the asintotic behavior of the
$\mathcal{F}^{A}$ model for unary proportional quantifiers is the Zadeh's
model. As a practical consequence, the $\mathcal{F}^{A}$ model can be
approximated in linear time when the base set is composed of a large number of elements.

To prove this property we will use a particular result derived of the central
limit theorem \cite[page 263]{Degroot88}.

\begin{theorem}
\textbf{Central limit theorem applied to Bernoulli variables}. Let
$X_{1},\ldots,X_{m}$ be independent random variables, each $X_{i}$ following a
Bernoulli distribution with parameter $p_{i}$. Moreover, let us suppose that
the infinite sum $\sum_{i=1}^{\infty}p_{i}\left(  1-p_{i}\right)  $ is
divergent and let $Y_{n}$ be%
\[
Y_{m}=\frac{\sum_{i=1}^{m}X_{i}-\sum_{i=1}^{n}p_{i}}{\left(  \sum_{i=1}%
^{m}p_{i}q_{i}\right)  ^{1/2}}%
\]
Then
\[
\lim_{n\rightarrow\infty}\Pr\left(  Y_{m}\leq x\right)  =\Phi\left(  x\right)
\]
where $\Phi\left(  x\right)  $ is the standard normal distribution function.
\end{theorem}

Using this result we will prove $\mathcal{F}^{A}$ approximation for
proportional unary quantifiers.

\begin{theorem}
Let $Q:\mathcal{P}\left(  E^{m}\right)  \rightarrow\mathbf{I}$ be a unary
proportional semi-fuzzy quantifier defined by means of a continuous
proportional fuzzy number $\mu_{Q}:\left[  0,1\right]  \rightarrow\mathbf{I}$%
\[
Q\left(  Y\right)  =\mu_{Q}\left(  \frac{\left\vert Y\right\vert }{\left\vert
E\right\vert }\right)
\]
for $Y\in\mathcal{P}\left(  E\right)  $. Let $e_{1,\ldots},e_{m}$ a succession
and $X_{m}\in\widetilde{\mathcal{P}}\left(  E\right)  $ be a fuzzy set
constructed on such sucession. Let $Y_{m}=\left\{  \mu_{X}\left(  e\right)
:\mu_{X}\left(  e\right)  =1\vee\mu_{X}\left(  e\right)  =0\right\}  $ the
crisp set constructed with the crisp elements of $X$. If the following limit
there exists:%
\[
\lim_{m\rightarrow\infty}\frac{\left\vert Y_{m}\right\vert }{\left\vert
E_{m}\right\vert }%
\]
\newline Then, when the size of the base set $E$ tends to infinite,
$\mathcal{F}^{A}\left(  Q\right)  \left(  X\right)  $ tends to:%
\[
\lim_{\left\vert E\right\vert \rightarrow\infty}\mathcal{F}^{A}\left(
Q\right)  \left(  X\right)  =fn\left(  \frac{\sum_{e\in E}\mu_{X}\left(
e\right)  }{\left\vert E\right\vert }\right)
\]
for $X\in\widetilde{\mathcal{P}}\left(  E\right)  $.
\end{theorem}

The assumption of existence of the limit $\lim_{m\rightarrow\infty}%
\frac{\left\vert Y_{m}\right\vert }{\left\vert E_{m}\right\vert }$ is very
weak and irrelevant from a practical point of view. We simply are asking the
proportional cardinality of the succession $Y_{m}$ does not oscillate as $m$
tends to infinite.

\begin{proof}
The interpretation underlying the $\mathcal{F}^{A}$ model assumes that each
$e_{i}\in E$ represents and independent Bernoulli process of probability
$\mu_{X}\left(  e_{i}\right)  $. Let us denote $X_{i}$ this Bernoulli process.
\end{proof}

The probability
\[
\Pr\left(  card_{X}=i\right)  =\sum_{Y\in\mathcal{P}\left(  E\right)  }%
m_{X}\left(  Y\right)
\]
represents the probability of the random variable $X_{m}=\sum_{i=1}%
^{\left\vert E\right\vert }X_{i}$.

First, let us suppose that $\lim_{\left\vert E\right\vert \rightarrow\infty
}\mu_{X}\left(  e_{i}\right)  \left(  1-\mu_{X}\left(  e_{i}\right)  \right)
$ is divergent. Using previous theorem, for an enough big $m=\left\vert
E\right\vert $, we can approximate this distribution for a normal distribution
of parameters $\overline{X_{m}}=\sum_{i=1}^{m}\mu_{X}\left(  e_{i}\right)  $
and $\sigma_{m}^{2}=\sum_{i=1}^{m}\mu_{X}\left(  e_{i}\right)  \left(
1-\mu_{X}\left(  e_{i}\right)  \right)  $\footnote{Let $X$ be a random
variable following a normal distribution of parameters $\left(  \mu
,\sigma\right)  $. If we define $Y=aX+b$ then $Y$ follows a normal
distribution of parameters $\left(  a\mu+b,a^{2}\sigma^{2}\right)  $.}.

As the fuzzy number associated to the fuzzy quantifier is defined on $\left[
0,1\right]  $, we can use the transformation $Y=$ $\frac{X_{m}}{m}$ to adapt
this distribution to the $\left[  0,1\right]  $ interval. The probability
distribution of $Y$ is a normal distribution of parameters $\left(  \frac{\mu
}{m},\frac{\sigma^{2}}{m^{2}}\right)  $.

Let us note that the normal distribution fulfills that the probability in $k$
standard deviations of the mean is identical for all the normal distributions.
As $k$ tends to infinite, the probability mass in $\left(  \mu-k\sigma
,\mu+k\sigma\right)  $ tends to $1$. That is, we always can find a $k$ such
that the probability mass in $\left(  \mu-k\sigma,\mu+k\sigma\right)  $ would
be as close to $1$ as we wanted.

Let us compute the limit of the variance when $m$ tends to infinite
\[
\lim_{m\rightarrow\infty}\frac{\sigma^{2}}{m^{2}}=\lim_{m\rightarrow\infty
}\frac{\sum_{i=1}^{m}\mu_{X}\left(  e_{i}\right)  \left(  1-\mu_{X}\left(
e_{i}\right)  \right)  }{m^{2}}\leq\lim_{m\rightarrow\infty}\frac{\sum
_{i=1}^{m}1}{m^{2}}=\lim_{m\rightarrow\infty}\frac{m}{m^{2}}=0
\]

That is, the probability distribution of $Y$ is more and more concentrated as
$m$ tends to infinite. And then, for every $\delta,\theta>0$ we can find a
sufficiently large $m$ such that $\Pr\left(  \frac{\mu}{m}-\delta,\frac{\mu
}{m}+\delta\right)  >1-\theta$.

As the fuzzy number $\mu_{Q}$ is continuous, then it is clear that
\[
\lim_{\left\vert E\right\vert \rightarrow\infty}\mathcal{F}^{A}\left(
Q\right)  \left(  X\right)  =fn\left(  \frac{\sum_{e\in E}\mu_{X}\left(
e\right)  }{\left\vert E\right\vert }\right)
\]
when $m=\left\vert E\right\vert $ tends to infinite.

Let us consider now that $\sum_{i=1}^{\infty}p_{i}\left(  1-p_{i}\right)  $ is
finite. In this case, we cannot apply the central limit theorem.

In this situation there are an infinite number of $is$ such that $p_{i}=1$ or
$p_{i}=0$. Moreover, when $m$ tends to infinite, the proportion of $is$ such
that $p_{i}q_{i}\neq0$ with respect to $m$ tends to $0$:%

\[
\lim_{\left\vert E\right\vert \rightarrow\infty}\frac{\left\vert e:\mu
_{X}\left(  e\right)  \neq1\wedge\mu_{X}\left(  e\right)  \neq0\right\vert
}{\left\vert E\right\vert }=0
\]

Then, $X$\ tends to a set in which only a finite number of elements are fuzzy.
Let $Y=\left\{  \mu_{X}\left(  e\right)  :\mu_{X}\left(  e\right)  =1\vee
\mu_{X}\left(  e\right)  =0\right\}  $ and let $k$ the finite numbers of $is$
such that $p_{i}q_{i}\neq0$.

It should be noted that by supposition the following limit there exists:%
\[
\lim_{\left\vert E\right\vert \rightarrow\infty}\frac{\left\vert Y\right\vert
}{\left\vert E\right\vert }=c
\]

For a sufficiently large $m$, all $is$ such that $p_{i}q_{i}\neq0$ are in
$X_{m}$. Let us consider the \textquotedblleft shape\textquotedblright\ of the
probability distribution of $X_{m}$. As the $is$ such that $p_{i}q_{i}=0$ are
crisp, the probability distribution of $X_{m}$ will consist on $k+1$ points
with $\Pr\left(  i\right)  \neq0$ and $m-k-1$ points with $\Pr\left(
i\right)  =0$. Moreover, by construction of $\Pr\left(  i\right)  $ all these
points are consecutive.

When we normalize $X_{m}$ to apply the fuzzy number that define the quantifier
(by means ot the transformation $Y=$ $\frac{X_{m}}{m}$) we are in the same
case that when $\sum_{i=1}^{\infty}p_{i}\left(  1-p_{i}\right)  $ is
divergent. The probability distribution will be more and more concentrated
around $\frac{\mu}{m}$ and the approximation would be valid.\smallskip

Although we have not developed a similar approximation for other kinds of
quantifiers, it seems easy to extend previous proof to more complex cases. For
example, in the case of proportional quantifiers we will have to consider two
Bernoulli successions $X_{\left(  1\right)  m}=\mu_{X_{1}}\left(
e_{1}\right)  ,\ldots,\mu_{X_{1}}\left(  e_{m}\right)  $ and $X_{\left(
2\right)  m}=\mu_{X_{2}}\left(  e_{1}\right)  ,\ldots,$ $\mu_{X_{2}}\left(
e_{m}\right)  $. As we are assuming independence, we can built a third
Bernoulli succession $Z_{m}=\mu_{X_{1}}\left(  e_{1}\right)  \mu_{X_{2}%
}\left(  e_{1}\right)  ,\ldots,\mu_{X_{1}}\left(  e_{m}\right)  \mu_{X_{2}%
}\left(  e_{m}\right)  $ and the previous results can be applied for
approximating the probability distribution of the cardinality of
$X_{1}\widetilde{\cap}X$). When $m$ tends to infinite, the probability
distribution of $Z_{m}$ tends to $\left(  \frac{\sum\mu_{X_{1}\widetilde{\cap
}X_{2}}\left(  e_{i}\right)  }{m},0\right)  $. For the same reason, the
probability distribution of $X_{1}$ tends to $\left(  \frac{\sum\mu_{X_{1}%
}\left(  e_{i}\right)  }{m},0\right)  $. And then, the proportional
cardinality of $X_{2}$ in $X_{1}$ tends to $\frac{\sum\mu_{X_{1}%
\widetilde{\cap}X_{2}}\left(  e_{i}\right)  }{\sum\mu_{X_{1}}\left(
e_{i}\right)  }$.

\subsection{Applying the $\mathcal{F}^{A}$ model to continuous fuzzy signals:
Temporal Quantification.}

The limit case approximation of the $\mathcal{F}^{A}$ opens the possibility of
applying the model to continuous fuzzy signals, fundamental for the
application of the model for \textit{fuzzy quantified temporal reasoning}%
\ \cite{Carinena01tfcis,Carinena03tesis,Mucientes03,Mucientes01}.

Let us consider a continuous fuzzy signal\footnote{The same argument allow us
to apply the model to a non continuous signal with at most, a finite number of
discontinuities. From a practical point of view, this is enough for
applications.} $S\left(  t\right)  $ where $t$ represents time in an interval
$E=\left[  t_{0},t_{1}\right]  $. And let us suppose we want to evaluate a
proportional quantifier $Q:\mathcal{P}\left(  \left[  t_{0},t_{1}\right]
\right)  \rightarrow\mathbf{I}$ on $S$ where $Q$ defined by means of a
continuous fuzzy number. For example, $Q$ could be defined as:%
\[
Q\left(  Y\right)  =S_{0.6,0.8}\left(  \frac{\lambda\left(  Y\right)
}{\lambda\left(  E\right)  }\right)
\]
where $\lambda$ represents the Lebesgue measure.

As the $\mathcal{F}^{A}$ model is finite, it cannot be directly applied to
continuous quantifiers. A reasonable possibility to apply the $\mathcal{F}%
^{A}$ model on a continuous set is to discretize the interval $\left[
t_{0},t_{1}\right]  $ in $m$ subintervals, $h=\frac{t_{1}-t_{0}}{m}$ and to
compute the result of the model in $E=\left\{  e_{0}=t_{0},e_{1}%
=t_{0}+h,\ldots,e_{m}=t_{1}\right\}  $. It should be noted that in the crisp
case, as $m$ tends to infinite, $Q\left(  \left\{  \chi_{Y}\left(
x_{0}\right)  ,\ldots,\chi_{Y}\left(  x_{m}\right)  \right\}  \right)  $ tends
to $Q\left(  Y\right)  $.

Let us consider the behavior of this approach in the limit case. Let
$X\in\widetilde{\mathcal{P}}\left(  E\right)  $ be the fuzzy set defined as
$\mu_{X}\left(  x_{i}\right)  =S\left(  x_{i}\right)  $%
\[
\lim_{m\rightarrow\infty}\mathcal{F}^{A}\left(  Q\right)  \left(  X\right)
=\lim_{m\rightarrow\infty}\sum_{Y\in\mathcal{P}\left(  E\right)  }m_{X}\left(
Y\right)  Q\left(  Y\right)
\]
and by using the limit approximation of the $\mathcal{F}^{A}$ model:%
\begin{align*}
\lim_{m\rightarrow\infty}\mathcal{F}^{A}\left(  Q\right)  \left(  X\right)
&  \approx\lim_{m\rightarrow\infty}\mu_{Q}\left(  \frac{\sum_{e_{i}}\mu
_{X}\left(  e_{i}\right)  }{m}\right) \\
&  =\lim_{m\rightarrow\infty}\mu_{Q}\left(  \frac{h}{h}\frac{\sum_{e_{i}}%
\mu_{X}\left(  e_{i}\right)  }{m}\right) \\
&  =\lim_{m\rightarrow\infty}\mu_{Q}\left(  \frac{\sum_{e_{i}}h\mu_{X}\left(
e_{i}\right)  }{hm}\right) \\
&  =\lim_{m\rightarrow\infty}\mu_{Q}\left(  \frac{\sum_{e_{i}}h\mu_{X}\left(
e_{i}\right)  }{t_{1}-t_{0}}\right)
\end{align*}
when $m$ tends to infinite, $\sum_{e_{i}}h\mu_{X}\left(  e_{i}\right)  $ tends
to $\int S\left(  t\right)  dt$ (as $\mu_{X}\left(  e_{i}\right)  \in\lbrack$
$\inf\left\{  S\left(  x\right)  :e_{i}\leq x<e_{i}+h\right\}  ,$
$\sup\left\{  S\left(  x\right)  :e_{i}\leq x<e_{i}+h\right\}  ]$ then
$\sum_{e_{i}}h\mu_{X}\left(  e_{i}\right)  $ is between the inferior integral
and the superior integral of $S$). And then,%
\[
\lim_{m\rightarrow\infty}\mathcal{F}^{A}\left(  Q\right)  \left(  X\right)
=\mu_{Q}\left(  \frac{\int S\left(  t\right)  dt}{t_{1}-t_{0}}\right)
\]

\subsection{Applying the $\mathcal{F}^{A}$ model to a population described by
means of a probability distribution}

Let $f$ be a probability distribution and $label$ a fuzzy label defined on the
referential universe of $f$. For example, $f$ could be a normal distribution
of parameters $\left(  \mu,\sigma\right)  $ representing the probability of
\textit{\textquotedblleft heights for male adults\textquotedblright}, and
$label$ the fuzzy label \textit{\textquotedblleft being tall\textquotedblright%
}.

Let $Z=Z_{1},\ldots,Z_{m}$ a random sample of $f$, and let $X=\{$ $\mu
_{X}\left(  z_{i}\right)  =label\left(  z_{i}\right)  /z_{i}$ $:i=1,\ldots
,m\}$. Then%
\[
\lim_{m\rightarrow\infty}\mathcal{F}^{A}\left(  Q\right)  \left(  X\right)
=\lim_{m\rightarrow\infty}\sum_{Y\in\mathcal{P}\left(  E\right)  }m_{X}\left(
Y\right)  Q\left(  Y\right)
\]
and by using the limit approximation of the $\mathcal{F}^{A}$ model:%

\begin{align*}
\lim_{m\rightarrow\infty}\mathcal{F}^{A}\left(  Q\right)  \left(  X\right)
&  =...\\
&  =f_{Q}\left(  \lim_{m\rightarrow\infty}\frac{\sum_{i}\mu_{X}\left(
z_{i}\right)  }{m}\right) \\
&  =f_{Q}\left(  \lim_{m\rightarrow\infty}\frac{1}{m}\sum_{i}\mu_{X}\left(
z_{i}\right)  \right) \\
&  =f_{Q}\left(  \overline{Z_{i}}\right) \\
&  =f_{Q}\left(  \int p\left(  x\right)  \mu_{X}\left(  x\right)  dx\right)
\end{align*}

that is, the probability of the fuzzy event \cite{Zadeh68}, or the probability
of the label given the population distribution.

\section{Conclusions}

In this report we described and studied the theoretical behavior of the
$\mathcal{F}^{A}$ QFM\footnote{Most of the theoretical analysis have been
previously published in \cite{DiazHermida06Tesis}, in spanish.}. The analysis
have proved the model is a finite DFS \cite{Glockner06Libro} essentially
different of the standard DFSs proposed by this author. Moreover the
underlying probabilistic semantics makes the model particularly interesting
for applications.

Other interesting results are the limit case approximation of the model, that
allows its application to continuous domains, and the study of the application
of the model to populations described by means of a probability distributions.

\section{Apendix A. Analysis of properties of the $\mathcal{F}^{A}$
QFM\label{AnalisisPropiedadesModeloFA}}

In this section we analyze the most relevant properties of the QFM
$\mathcal{F}^{A}$. A sligthly more detailed discussion can be consulted in
\cite{DiazHermida06Tesis}.

First of all, we will proof some preliminary results.

\begin{lemma}
\label{LemaFAOperInduc}It holds that\newline1) $\widetilde{\mathcal{F}^{A}%
}\left(  id_{2}\right)  \left(  x\right)  =\widetilde{id}_{\mathbf{I}}\left(
x\right)  $ where $id_{2}:\mathbf{2}\rightarrow\mathbf{2}$ is the bivalued
identity and $\widetilde{id}_{\mathbf{I}}:\mathbf{I}\rightarrow\mathbf{I}$ is
the fuzzy identity.\newline2) $\widetilde{\mathcal{F}^{A}}\left(
\lnot\right)  =\widetilde{\lnot}\left(  x\right)  =$ where $\widetilde{\lnot}$
is the standard negation.\newline3) $\widetilde{\mathcal{F}^{A}}\left(
\wedge\right)  \left(  x_{1},x_{2}\right)  =x_{1}\times x_{2}$; that is, the
product tnorm.$\newline$4) $\widetilde{\mathcal{F}^{A}}\left(  \vee\right)
\left(  x_{1},x_{2}\right)  =\widetilde{\lnot}\widetilde{\mathcal{F}^{A}%
}\left(  \wedge\right)  \left(  \widetilde{\lnot}x_{1},\widetilde{\lnot}%
x_{2}\right)  =x_{1}+x_{2}-x_{1}\cdot x_{2}$; that is, the probabilistic
tconorm, the dual of the product.$\newline$5)$\widetilde{\mathcal{F}^{A}%
}\left(  \rightarrow\right)  \left(  x_{1},x_{2}\right)  =1-x_{1}+x_{1}\cdot
x_{2}$, in this case the Rechenbach fuzzy implication.$\newline$
\end{lemma}

\begin{proof}
We only are going to show the proof of $\widetilde{\mathcal{F}^{A}}\left(
\vee\right)  \left(  x_{1},x_{2}\right)  =x_{1}+x_{2}-x_{1}\times x_{2}$. The
rest of the proofs can be consulted in \cite[appendix A]{DiazHermida06Tesis}.

First, note that the definici\'{o}n of\ $Q_{\vee}:\mathcal{P}\left(  \left\{
1,2\right\}  \right)  \rightarrow\mathbf{I}$ is:%
\begin{align*}
Q_{\vee}\left(  \varnothing\right)   &  =\vee\left(  \eta^{-1}\left(
\varnothing\right)  \right)  =\vee\left(  0,0\right)  =0\\
Q_{\vee}\left(  \left\{  1\right\}  \right)   &  =\vee\left(  \eta^{-1}\left(
\left\{  1\right\}  \right)  \right)  =\vee\left(  1,0\right)  =1\\
Q_{\vee}\left(  \left\{  2\right\}  \right)   &  =\vee\left(  \eta^{-1}\left(
\left\{  2\right\}  \right)  \right)  =\vee\left(  0,1\right)  =1\\
Q_{\vee}\left(  \left\{  1,2\right\}  \right)   &  =\wedge\left(  \eta
^{-1}\left(  \left\{  1,2\right\}  \right)  \right)  =\wedge\left(
1,1\right)  =1
\end{align*}
Then%
\begin{align*}
\widetilde{\mathcal{F}^{A}}\left(  x_{1},x_{2}\right)   &  =\mathcal{F}%
^{A}\left(  Q_{\vee}\right)  \left(  \eta^{-1}\left(  x_{1},x_{2}\right)
\right) \\
&  =\mathcal{F}^{A}\left(  Q_{\vee}\right)  \left(  \left\{  x_{1}%
/1,x_{2}/2\right\}  \right)  =\sum_{Y\in\mathcal{P}\left(  \left\{
1,2\right\}  \right)  }m_{\left\{  x_{1}/1,x_{2}/2\right\}  }\left(  Y\right)
Q\left(  Y\right) \\
&  =m_{\left\{  x_{1}/1,x_{2}/2\right\}  }\left(  \varnothing\right)  Q\left(
\varnothing\right)  +m_{\left\{  x_{1}/1,x_{2}/2\right\}  }\left(  \left\{
1\right\}  \right)  Q\left(  \left\{  1\right\}  \right) \\
&  +m_{\left\{  x_{1}/1,x_{2}/2\right\}  }\left(  \left\{  2\right\}  \right)
Q\left(  \left\{  2\right\}  \right)  +m_{\left\{  x_{1}/1,x_{2}/2\right\}
}\left(  \left\{  1,2\right\}  \right)  Q\left(  \left\{  1,2\right\}  \right)
\\
&  =\mu_{\left\{  x_{1}/1,x_{2}/2\right\}  }\left(  1\right)  \left(
1-\mu_{\left\{  x_{1}/1,x_{2}/2\right\}  }\left(  2\right)  \right)  +\\
&  \left(  1-\mu_{\left\{  x_{1}/1,x_{2}/2\right\}  }\left(  1\right)
\right)  \mu_{\left\{  x_{1}/1,x_{2}/2\right\}  }\left(  2\right)
+\mu_{\left\{  x_{1}/1,x_{2}/2\right\}  }\left(  1\right)  \mu_{\left\{
x_{1}/1,x_{2}/2\right\}  }\left(  2\right) \\
&  =x_{1}\left(  1-x_{2}\right)  +\left(  1-x_{1}\right)  x_{2}+x_{1}x_{2}\\
&  =x_{1}+x_{2}-x_{1}x_{2}%
\end{align*}

\end{proof}

Moreover, in the proofs of the properties of the $\mathcal{F}^{A}$ model we
need the following lemmas too.

\begin{lemma}
\label{LemaCasoNitido}Let $X,Y\in\mathcal{P}\left(  E\right)  $ be crisp sets.
It holds that
\[
m_{X}\left(  Y\right)  =\left\{
\begin{tabular}
[c]{lll}%
$0$ & $:$ & $X\neq Y$\\
$1$ & $:$ & $X=Y$%
\end{tabular}
\ \right.
\]

\end{lemma}

\begin{proof}
The definition of $m_{X}\left(  Y\right)  $ is%
\[
m_{X}\left(  Y\right)  =%
{\displaystyle\prod\limits_{e\in Y}}
\mu_{X}\left(  e\right)
{\displaystyle\prod\limits_{e\in E\backslash Y}}
\left(  1-\mu_{X}\left(  e\right)  \right)
\]
and as $X$ is crisp $\mu_{X}\left(  e\right)  =1$ if $e\in E$ and $\mu
_{X}\left(  e\right)  =0$ if $e\notin E$.
\end{proof}

\begin{lemma}
\label{LemaProyeccion}Let $X\in\widetilde{\mathcal{P}}\left(  E\right)  $ be a
fuzzy set $E^{\prime}\subseteq E,E^{\prime\prime}=E\backslash E^{\prime}$
(that is $E^{\prime}\cup E^{\prime\prime}=E$). Let $Y\in\mathcal{P}\left(
E\right)  $ be a crisp set. Then\footnote{It should be remembered that with
the notation $X^{E^{\prime}}$ where $X\in\widetilde{\mathcal{P}}\left(
E\right)  $ is a fuzzy set we represent the restriction of $X$ to the
reference universe $E^{\prime}\subseteq E$.},
\[
m_{X}^{E}\left(  Y\right)  =m_{X}^{E^{\prime}}\left(  Y^{E^{\prime}}\right)
m_{X}^{E^{\prime\prime}}\left(  Y^{E^{\prime\prime}}\right)
\]

\end{lemma}

\begin{proof}
By definition of $m_{X}^{E}\left(  Y\right)  $, the probability of $m_{X}%
^{E}\left(  Y\right)  $ is the product of the probabilities on their
projections:%
\begin{align*}
m_{X}^{E}\left(  Y\right)   &  =%
{\displaystyle\prod\limits_{e\in Y}}
\mu_{X}\left(  e\right)
{\displaystyle\prod\limits_{e\in E\backslash Y}}
\left(  1-\mu_{X}\left(  e\right)  \right) \\
&  =%
{\displaystyle\prod\limits_{e\in Y\cap E^{\prime}}}
\mu_{X}\left(  e\right)
{\displaystyle\prod\limits_{e\in\left(  E\backslash Y\right)  \cap E^{\prime}%
}}
\left(  1-\mu_{X}\left(  e\right)  \right)  \cdot%
{\displaystyle\prod\limits_{e\in Y\cap E^{\prime\prime}}}
\mu_{X}\left(  e\right)
{\displaystyle\prod\limits_{e\in\left(  E\backslash Y\right)  \cap
E^{\prime\prime}}}
\left(  1-\mu_{X}\left(  e\right)  \right) \\
&  =m_{X}^{E^{\prime}}\left(  Y^{E^{\prime}}\right)  m_{X}^{E^{\prime\prime}%
}\left(  Y^{E^{\prime\prime}}\right)
\end{align*}

\end{proof}

\begin{lemma}
\label{LemaTConormaProbabilistica}Let
\[
\widetilde{\vee}\left(  x_{1},x_{2}\right)  =x_{1}+x_{2}-x_{1}x_{2}%
\]
be the probabilistic tconorm. By $\widetilde{\vee}\left(  x_{1},\ldots
,x_{m}\right)  $ we denote its $m$-ary version; that is,
\[
\widetilde{\vee}\left(  x_{1},\ldots,x_{m}\right)  =\widetilde{\vee}\left(
x_{1},\widetilde{\vee}\left(  x_{2},\widetilde{\vee}\left(  x_{3}%
\ldots,\widetilde{\vee}\left(  x_{m-1},x_{m}\right)  \right)  \right)
\right)
\]
It is fulfilled\footnote{It should be noted that for $\widetilde{\vee}\left(
x\right)  =x$ it also is fulfilled that $1-%
{\displaystyle\prod\limits_{i=1}^{1}}
\left(  1-x\right)  =1-\left(  1-x\right)  =x$. Moreover, if we define
$\widetilde{\vee}\left(  {}\right)  =0$ and $%
{\displaystyle\prod\limits_{i\in\varnothing}}
=1$ previous relationship is also fulfilled.}
\[
\widetilde{\vee}\left(  x_{1},\ldots,x_{m}\right)  =1-%
{\displaystyle\prod\limits_{i=1}^{m}}
\left(  1-x_{i}\right)
\]

\end{lemma}

\begin{proof}
Proof is by induction.

Case $i=2$:%
\[
\widetilde{\vee}\left(  x_{1},x_{2}\right)  =x_{1}+x_{2}-x_{1}x_{2}%
\]
and then%
\begin{align*}
1-%
{\displaystyle\prod\limits_{i=1}^{2}}
\left(  1-x_{i}\right)   &  =1-\left(  1-x_{1}\right)  \left(  1-x_{2}\right)
\\
&  =1-\left(  1-x_{1}-x_{2}+x_{1}x_{2}\right) \\
&  =x_{1}+x_{2}-x_{1}x_{2}%
\end{align*}

Induction supposition: Case $i=m$%
\[
\widetilde{\vee}\left(  x_{1},\ldots,x_{m}\right)  =1-%
{\displaystyle\prod\limits_{i=1}^{m}}
\left(  1-x_{i}\right)
\]

Case $i=m+1$%
\begin{align*}
\widetilde{\vee}\left(  x_{1},\ldots,x_{m+1}\right)   &  =\widetilde{\vee
}\left(  \widetilde{\vee}\left(  x_{1},\ldots,x_{m}\right)  ,x_{m+1}\right) \\
&  =\widetilde{\vee}\left(  1-%
{\displaystyle\prod\limits_{i=1}^{m}}
\left(  1-x_{i}\right)  ,x_{m+1}\right) \\
&  =1-%
{\displaystyle\prod\limits_{i=1}^{m}}
\left(  1-x_{i}\right)  +x_{m+1}-\left(  \left(  1-%
{\displaystyle\prod\limits_{i=1}^{m}}
\left(  1-x_{i}\right)  \right)  \times x_{m+1}\right) \\
&  =\left(  1-%
{\displaystyle\prod\limits_{i=1}^{m}}
\left(  1-x_{i}\right)  \right)  \left(  1-x_{m+1}\right)  +x_{m+1}\\
&  =\left(  1-x_{m+1}\right)  -\left(  1-x_{m+1}\right)
{\displaystyle\prod\limits_{i=1}^{m}}
\left(  1-x_{i}\right)  +x_{m+1}\\
&  =\left(  1-x_{m+1}\right)  -%
{\displaystyle\prod\limits_{i=1}^{m+1}}
\left(  1-x_{i}\right)  +x_{m+1}\\
&  =1-%
{\displaystyle\prod\limits_{i=1}^{m+1}}
\left(  1-x_{i}\right)
\end{align*}

\end{proof}

\subsection{The $\mathcal{F}^{A}$ QFM is a DFS}

\paragraph{Axiom Z-1}

It holds that
\[
\mathcal{U}\left(  \mathcal{F}^{A}\left(  Q\right)  \right)  =Q\quad\text{if
}n\leq1
\]

\begin{proof}
Let $Q:\mathcal{P}\left(  E\right)  \rightarrow\mathbf{I}$ be a semi-fuzzy
quantifier, $Y\in\mathcal{P}\left(  E\right)  $ a crisp set. Using the lemma
\ref{LemaCasoNitido} we have%
\begin{align*}
\mathcal{F}^{A}\left(  Q\right)  \left(  Y\right)   &  =\sum_{Z\in
\mathcal{P}\left(  E\right)  }m_{Y}\left(  Z\right)  Q\left(  Z\right)
=m_{Y}\left(  Y\right)  Q\left(  Y\right) \\
&  =Q\left(  Y\right)
\end{align*}
And then $\mathcal{U}\left(  \mathcal{F}\left(  Q\right)  \right)  \left(
Y\right)  =Q\left(  Y\right)  $ for all $Y\in\mathcal{P}\left(  E\right)  $.
\end{proof}

\paragraph{Axiom Z-2}

It holds that
\[
\mathcal{F}^{A}\left(  Q\right)  =\widetilde{\pi_{e}}\quad si\text{ }Q=\pi
_{e}\text{ for some }e\in E
\]

\begin{proof}
Using the lemma \ref{LemaProyeccion}%
\begin{align*}
\mathcal{F}^{A}\left(  \pi_{e}\right)  \left(  X\right)   &  =\sum
_{Y\in\mathcal{P}\left(  E\right)  }m_{X}\left(  Y\right)  \pi_{e}\left(
Y\right)  =\sum_{Y\in\mathcal{P}\left(  E\right)  }m_{X}\left(  Y\right)
\chi_{Y}\left(  e\right) \\
&  =\sum_{Y\in\mathcal{P}\left(  E\right)  |e\in Y}m_{X}\left(  Y\right)
=\sum_{\left\{  e\right\}  \subseteq Y\subseteq E}m_{X}\left(  Y\right) \\
&  =\sum_{\left\{  e\right\}  \subseteq Y\subseteq E}\mu_{X}\left(  e\right)
m_{X}^{E\backslash\left\{  e\right\}  }\left(  Y\backslash\left\{  e\right\}
\right)  =\mu_{X}\left(  e\right)  \sum_{\varnothing\subseteq Y\subseteq
E\backslash\left\{  e\right\}  }m_{X}^{E\backslash\left\{  e\right\}  }\left(
Y\right) \\
&  =\mu_{X}\left(  e\right)
\end{align*}

\end{proof}

\paragraph{Axiom Z-3}

It holds that%
\[
\mathcal{F}^{A}\left(  Q\widetilde{\square}\right)  =\mathcal{F}^{A}\left(
Q\right)  \widetilde{\square}\quad n>0
\]

For the proof of this axiom we will need the proofs of the properties of
internal and external negation.

\textbf{Proof of the property of external negation}

We have to prove that%
\[
\mathcal{F}^{A}\left(  \widetilde{\lnot}Q\right)  =\widetilde{\lnot
}\mathcal{F}^{A}\left(  Q\right)
\]

In the lemma \ref{LemaFAOperInduc} we have established that the induced
negation operation is the standard negation.

\begin{proof}%
\begin{align*}
\mathcal{F}^{A}\left(  \widetilde{\lnot}Q\right)  \left(  X_{1},\ldots
,X_{n}\right)   &  =\sum_{Y_{1}\in\mathcal{P}\left(  E\right)  }\ldots
\sum_{Y_{n}\in\mathcal{P}\left(  E\right)  }m_{X_{1}}\left(  Y_{1}\right)
\ldots m_{X_{n}}\left(  Y_{n}\right)  \left(  \widetilde{\lnot}Q\right)
\left(  Y_{1},\ldots,Y_{n}\right) \\
&  =\sum_{Y_{1}\in\mathcal{P}\left(  E\right)  }\ldots\sum_{Y_{n}%
\in\mathcal{P}\left(  E\right)  }m_{X_{1}}\left(  Y_{1}\right)  \ldots
m_{X_{n}}\left(  Y_{n}\right)  \left(  1-Q\left(  Y_{1},\ldots,Y_{n}\right)
\right) \\
&  =\sum_{Y_{1}\in\mathcal{P}\left(  E\right)  }\ldots\sum_{Y_{n}%
\in\mathcal{P}\left(  E\right)  }m_{X_{1}}\left(  Y_{1}\right)  \ldots
m_{X_{n}}\left(  Y_{n}\right) \\
&  -\sum_{Y_{1}\in\mathcal{P}\left(  E\right)  }\ldots\sum_{Y_{n}%
\in\mathcal{P}\left(  E\right)  }Q\left(  Y_{1},\ldots,Y_{n}\right)  m_{X_{1}%
}\left(  Y_{1}\right)  \ldots m_{X_{n}}\left(  Y_{n}\right) \\
&  =1-\mathcal{F}^{A}\left(  Q\right)  \left(  X_{1},\ldots,X_{n}\right)
=\widetilde{\lnot}\left(  \mathcal{F}^{A}\left(  Q\right)  \left(
X_{1},\ldots,X_{n}\right)  \right)
\end{align*}

\end{proof}

\textbf{Proof of the property of internal negation}

We have to prove that%
\[
\mathcal{F}^{A}\left(  Q\widetilde{\lnot}\right)  =\mathcal{F}^{A}\left(
Q\right)  \widetilde{\lnot}%
\]

\begin{lemma}
\label{LemaNegacionInterna}Let $X\in\widetilde{\mathcal{P}}\left(  E\right)  $
be a fuzzy set. Then%
\[
m_{X}\left(  Y\right)  =m_{\widetilde{\lnot}X}\left(  \lnot Y\right)
\]

\end{lemma}

\begin{proof}%
\begin{align*}
m_{X}\left(  Y\right)   &  =\Pr\left(  representative_{X}=Y\right)  =%
{\displaystyle\prod\limits_{e\in Y}}
\mu_{X}\left(  e\right)
{\displaystyle\prod\limits_{e\in E\backslash Y}}
\left(  1-\mu_{X}\left(  e\right)  \right) \\
&  =%
{\displaystyle\prod\limits_{e\in Y}}
\widetilde{\lnot}\widetilde{\lnot}\mu_{X}\left(  e\right)
{\displaystyle\prod\limits_{e\in E\backslash Y}}
\widetilde{\lnot}\mu_{X}\left(  e\right)  =%
{\displaystyle\prod\limits_{e\in E\backslash Y}}
\widetilde{\lnot}\mu_{X}\left(  e\right)
{\displaystyle\prod\limits_{e\in Y}}
\widetilde{\lnot}\widetilde{\lnot}\mu_{X}\left(  e\right) \\
&  =%
{\displaystyle\prod\limits_{e\in E\backslash Y}}
\mu_{\widetilde{\lnot}X}\left(  e\right)
{\displaystyle\prod\limits_{e\in E\backslash\left(  E\backslash Y\right)  }}
\widetilde{\lnot}\mu_{\widetilde{\lnot}X}\left(  e\right)  =%
{\displaystyle\prod\limits_{e\in E\backslash Y}}
\mu_{\widetilde{\lnot}X}\left(  e\right)
{\displaystyle\prod\limits_{e\in E\backslash\left(  E\backslash Y\right)  }}
\left(  1-\mu_{\widetilde{\lnot}X}\left(  e\right)  \right) \\
&  =\Pr\left(  representative_{\widetilde{\lnot}X}=E\backslash Y\right)
=m_{\widetilde{\lnot}X}\left(  \lnot Y\right)
\end{align*}

\end{proof}

Proof of the property of internal negation:

\begin{proof}%
\begin{align*}
\mathcal{F}^{A}\left(  Q\lnot\right)  \left(  X_{1},\ldots,X_{n}\right)   &
=\sum_{Y_{1}\in\mathcal{P}\left(  E\right)  }\ldots\sum_{Y_{n}\in
\mathcal{P}\left(  E\right)  }m_{X_{1}}\left(  Y_{1}\right)  \ldots m_{X_{n}%
}\left(  Y_{n}\right)  \left(  Q\lnot\right)  \left(  Y_{1},\ldots
,Y_{n}\right) \\
&  =\sum_{Y_{1}\in\mathcal{P}\left(  E\right)  }\ldots\sum_{Y_{n}%
\in\mathcal{P}\left(  E\right)  }m_{X_{1}}\left(  Y_{1}\right)  \ldots
m_{X_{n}}\left(  Y_{n}\right)  Q\left(  Y_{1},\ldots,\lnot Y_{n}\right)
\end{align*}
and using the lemma \ref{LemaNegacionInterna} the above expression is equal to%
\begin{align*}
\ldots &  =\sum_{Y_{1}\in\mathcal{P}\left(  E\right)  }\ldots\sum_{Y_{n}%
\in\mathcal{P}\left(  E\right)  }m_{X_{1}}\left(  Y_{1}\right)  \ldots
m_{\widetilde{\lnot}X_{n}}\left(  \lnot Y_{n}\right)  Q\left(  Y_{1}%
,\ldots,\lnot Y_{n}\right) \\
\ldots &  =\sum_{Y_{1}\in\mathcal{P}\left(  E\right)  }\ldots\sum_{Y_{n}%
\in\mathcal{P}\left(  E\right)  }m_{X_{1}}\left(  Y_{1}\right)  \ldots
m_{\widetilde{\lnot}X_{n}}\left(  Y_{n}\right)  Q\left(  Y_{1},\ldots
,Y_{n}\right) \\
\ldots &  =\mathcal{F}^{A}\left(  Q\right)  \left(  X_{1},\ldots
,\widetilde{\lnot}X_{n}\right)
\end{align*}
\smallskip
\end{proof}

Using the properties of internal and external negation duality is trivial.%
\begin{align*}
\mathcal{F}^{A}\left(  Q\square\right)   &  =\mathcal{F}^{A}\left(
\widetilde{\lnot}Q\lnot\right)  =\widetilde{\lnot}\mathcal{F}^{A}\left(
Q\right)  \widetilde{\lnot}\\
&  =\mathcal{F}^{A}\left(  Q\right)  \widetilde{\square}%
\end{align*}

\paragraph{Axiom Z-4}

It holds that%
\begin{equation}
\mathcal{F}\left(  Q\cup\right)  =\mathcal{F}\left(  Q\right)  \widetilde
{\cup}\quad n>0 \label{EqAxiomaZ4ModeloFA_1}%
\end{equation}

In the proof of \ref{EqAxiomaZ4ModeloFA_1} we will use the following results:

\begin{lemma}
\label{LemaEncuentrosInternos} Let $X_{1},X_{2}\in\widetilde{\mathcal{P}%
}\left(  E\right)  $ be given, $R\in\mathcal{P}\left(  E\right)  $ a crisp
set. Then
\begin{equation}
\sum_{Y_{1}\in\mathcal{P}\left(  E\right)  }\sum_{Y_{2}\in\mathcal{P}\left(
E\right)  /Y_{1}\cap Y_{2}=R}m_{X_{1}}\left(  Y_{1}\right)  m_{X_{2}}\left(
Y_{2}\right)  =m_{X_{1}\widetilde{\cap}X_{2}}\left(  R\right)
\label{EqLemaEncuentrosInternos_1}%
\end{equation}
where $X_{1}\widetilde{\cap}X_{2}$ is defined by means of the induced tnorm.
\end{lemma}

\begin{proof}
Let $E^{\prime}=E\backslash R$. We will use \ref{LemaProyeccion}.

It fulfills the following:
\begin{equation}
m_{X_{1}\widetilde{\cap}X_{2}}\left(  R\right)  =%
{\displaystyle\prod\limits_{e\in R}}
\mu_{X_{1}}\left(  e\right)  \mu_{X_{2}}\left(  e\right)
{\displaystyle\prod\limits_{e\in E\backslash R}}
\left(  1-\mu_{X_{1}}\left(  e\right)  \mu_{X_{2}}\left(  e\right)  \right)
\label{EqLemaEncuentrosInternos_2}%
\end{equation}

As the sum \ref{EqLemaEncuentrosInternos_1} is restricted to the $Y_{1}%
,Y_{2}\in\mathcal{P}\left(  E\right)  $ such that
\[
Y_{1}\cap Y_{2}=R
\]
then,%
\begin{align}
&  \sum_{Y_{1}\in\mathcal{P}\left(  E\right)  }\sum_{Y_{2}\in\mathcal{P}%
\left(  E\right)  /Y_{1}\cap Y_{2}=R}m_{X_{1}}\left(  Y_{1}\right)  m_{X_{2}%
}\left(  Y_{2}\right) \label{EqLemaEncuentrosInternos_3}\\
&  =\sum_{R\subseteq Y_{1}\subseteq E}\sum_{R\subseteq Y_{2}\subseteq
E/Y_{1}\cap Y_{2}=R}m_{X_{1}}\left(  Y_{1}\right)  m_{X_{2}}\left(
Y_{2}\right) \nonumber\\
&  =\sum_{\varnothing\subseteq C_{1}\subseteq E^{\prime}}\sum_{\varnothing
\subseteq C_{2}\subseteq E^{\prime}/C_{1}\cap C_{2}=\varnothing}m_{X_{1}%
}\left(  C_{1}\cup R\right)  m_{X_{2}}\left(  C_{2}\cup R\right) \nonumber
\end{align}
and using lemma \ref{LemaProyeccion}%
\begin{align}
&  =\sum_{\varnothing\subseteq C_{1}\subseteq E^{\prime}}\sum_{\varnothing
\subseteq C_{2}\subseteq E^{\prime}/C_{1}\cap C_{2}=\varnothing}m_{X_{1}%
}^{E^{\prime}}\left(  C_{1}\right)
{\displaystyle\prod\limits_{e\in R}}
\mu_{X_{1}}\left(  e\right)  m_{X_{2}}^{E^{\prime}}\left(  C_{2}\right)
{\displaystyle\prod\limits_{e\in R}}
\mu_{X_{2}}\left(  e\right) \nonumber\\
&  =%
{\displaystyle\prod\limits_{e\in R}}
\mu_{X_{1}}\left(  e\right)  \mu_{X_{2}}\left(  e\right)  \sum_{\varnothing
\subseteq C_{1}\subseteq E^{\prime}}\sum_{\varnothing\subseteq C_{2}\subseteq
E^{\prime}/C_{1}\cap C_{2}=\varnothing}m_{X_{1}}^{E^{\prime}}\left(
C_{1}\right)  m_{X_{2}}^{E^{\prime}}\left(  C_{2}\right) \nonumber\\
&  =%
{\displaystyle\prod\limits_{e\in R}}
\mu_{X_{1}}\left(  e\right)  \mu_{X_{2}}\left(  e\right)  \sum_{\varnothing
\subseteq C_{1}\subseteq E^{\prime}}\sum_{C_{1}\subseteq C_{2}\subseteq
E^{\prime}}m_{X_{1}}^{E^{\prime}}\left(  C_{1}\right)  m_{X_{2}}^{E^{\prime}%
}\left(  C_{2}\backslash C_{1}\right) \nonumber
\end{align}

And the equality \ref{EqLemaEncuentrosInternos_2} will be fulfilled if:%
\begin{equation}
\sum_{\varnothing\subseteq C_{1}\subseteq E^{\prime}}\sum_{C_{1}\subseteq
C_{2}\subseteq E^{\prime}}m_{X_{1}}^{E^{\prime}}\left(  C_{1}\right)
m_{X_{2}}^{E^{\prime}}\left(  C_{2}\backslash C_{1}\right)  =%
{\displaystyle\prod\limits_{e\in E^{\prime}}}
\left(  1-\mu_{X_{1}}\left(  e\right)  \mu_{X_{2}}\left(  e\right)  \right)
\label{EqLemaEncuentrosInternos_4}%
\end{equation}

The proof is by induction in the cardinality of $E^{\prime}$. We will denote
by $E^{i}=\left\{  e_{1},\ldots,e_{i}\right\}  $ a referential set with $i\,$elements.

\textbf{Case} $i=1$ ($\left\vert E^{\prime}\right\vert =1$ ($E^{\prime}%
=E^{1}=\left\{  e_{1}\right\}  $):%
\begin{align*}
&  \sum_{\varnothing\subseteq C_{1}\subseteq E^{\prime}}\sum_{C_{1}\subseteq
C_{2}\subseteq E^{\prime}}m_{X_{1}}\left(  C_{1}\right)  m_{X_{2}}\left(
C_{2}\backslash C_{1}\right) \\
&  =\sum_{\varnothing\subseteq C_{1}\subseteq\left\{  e_{1}\right\}  }%
\sum_{C_{1}\subseteq C_{2}\subseteq\left\{  e_{1}\right\}  }m_{X_{1}}\left(
C_{1}\right)  m_{X_{2}}\left(  C_{2}\backslash C_{1}\right) \\
&  =m_{X_{1}}\left(  \varnothing\right)  \left(  m_{X_{2}}\left(
\varnothing\right)  +m_{X_{2}}\left(  \left\{  e_{1}\right\}  \right)
\right)  +m_{X_{1}}\left(  \left\{  e_{1}\right\}  \right)  m_{X_{2}}\left(
\varnothing\right) \\
&  =\left(  1-\mu_{X_{1}}\left(  e_{1}\right)  \right)  \left(  1-\mu_{X_{2}%
}\left(  e_{1}\right)  +\mu_{X_{2}}\left(  e_{1}\right)  \right)  +\mu_{X_{1}%
}\left(  e_{1}\right)  \left(  1-\mu_{X_{2}}\left(  e_{1}\right)  \right) \\
&  =1-\mu_{X_{1}}\left(  e_{1}\right)  +\mu_{X_{1}}\left(  e_{1}\right)
\left(  1-\mu_{X_{2}}\left(  e_{1}\right)  \right) \\
&  =1-\mu_{X_{1}}\left(  e_{1}\right)  +\mu_{X_{1}}\left(  e_{1}\right)
-\mu_{X_{1}}\left(  e_{1}\right)  \mu_{X_{2}}\left(  e_{1}\right) \\
&  =1-\mu_{X_{1}}\left(  e_{1}\right)  \mu_{X_{2}}\left(  e_{1}\right)
\end{align*}

\textbf{Induction hypothesis: Case} $i=m$ ($\left\vert E^{\prime}\right\vert
=m$, ($E^{\prime}=E^{m}=\left\{  e_{1},\ldots,e_{m}\right\}  $). We suppose
that%
\[
\sum_{\varnothing\subseteq C_{1}\subseteq E^{\prime}}\sum_{C_{1}\subseteq
C_{2}\subseteq E^{\prime}}m_{X_{1}}\left(  C_{1}\right)  m_{X_{2}}\left(
C_{2}\backslash C_{1}\right)  =%
{\displaystyle\prod\limits_{i=1}^{m}}
\left(  1-\mu_{X_{1}}\left(  e_{i}\right)  \mu_{X_{2}}\left(  e_{i}\right)
\right)
\]

\textbf{Case} $i=m+1$: $\left\vert E^{\prime}\right\vert =m+1$, ($E^{\prime
}=E^{m+1}=\left\{  e_{1},\ldots,e_{m+1}\right\}  $)

It should be noted that if $e_{m+1}\in C$ then
\begin{align*}
m_{X_{1}}\left(  C\right)   &  =%
{\displaystyle\prod\limits_{e\in C}}
\mu_{X_{1}}\left(  e\right)
{\displaystyle\prod\limits_{e\notin C}}
\left(  1-\mu_{X_{1}}\left(  e\right)  \right) \\
&  =\mu_{X_{1}}\left(  e_{m+1}\right)
{\displaystyle\prod\limits_{e\in C\backslash\left\{  e_{m+1}\right\}  }}
\mu_{X_{1}}\left(  e\right)
{\displaystyle\prod\limits_{e\notin C}}
\left(  1-\mu_{X_{1}}\left(  e\right)  \right) \\
&  =\mu_{X_{1}}\left(  e_{m+1}\right)  m_{X_{1}}^{E^{m}}\left(  C\backslash
\left\{  e_{m+1}\right\}  \right)
\end{align*}

Whilst if $e_{m+1}\notin C$ then we have
\begin{align*}
m_{X_{1}}\left(  C\right)   &  =%
{\displaystyle\prod\limits_{e\in C}}
\mu_{X_{1}}\left(  e\right)
{\displaystyle\prod\limits_{e\notin C}}
\left(  1-\mu_{X_{1}}\left(  e\right)  \right) \\
&  =\left(  1-\mu_{X_{1}}\left(  e_{m+1}\right)  \right)
{\displaystyle\prod\limits_{e\in C}}
\mu_{X_{1}}\left(  e\right)
{\displaystyle\prod\limits_{e\notin C\cup\left\{  e_{m+1}\right\}  }}
\left(  1-\mu_{X_{1}}\left(  e\right)  \right) \\
&  =\left(  1-\mu_{X_{1}}\left(  e_{m+1}\right)  \right)  m_{X_{1}}^{E^{m}%
}\left(  C\right)
\end{align*}

By computation%
\begin{align}
&  \sum_{\varnothing\subseteq C_{1}\subseteq E^{m+1}}\sum_{C_{1}\subseteq
C_{2}\subseteq E^{m+1}}m_{X_{1}}\left(  C_{1}\right)  m_{X_{2}}\left(
C_{2}\backslash C_{1}\right) \label{EqLemaEncuentrosInternos_5}\\
&  =\sum_{\varnothing\subseteq M_{1}\subseteq E^{m}}\sum_{M_{1}\subseteq
M_{2}\subseteq E^{m}}m_{X_{1}}\left(  M_{1}\right)  m_{X_{2}}\left(
M_{2}\backslash M_{1}\right)  +\nonumber\\
&  \sum_{\varnothing\subseteq M_{1}\subseteq E^{m}}\sum_{M_{1}\subseteq
M_{2}\subseteq E^{m}}m_{X_{1}}\left(  M_{1}\right)  m_{X_{2}}\left(  M_{2}%
\cup\left\{  e_{m+1}\right\}  \backslash M_{1}\right)  +\nonumber\\
&  \sum_{\varnothing\subseteq M_{1}\subseteq E^{m}}\sum_{M_{1}\subseteq
M_{2}\subseteq E^{m}}m_{X_{1}}\left(  M_{1}\cup\left\{  e_{m+1}\right\}
\right)  m_{X_{2}}\left(  M_{2}\cup\left\{  e_{m+1}\right\}  \backslash
M_{1}\cup\left\{  e_{m+1}\right\}  \right) \nonumber
\end{align}

As $C_{1}\subseteq C_{2}\subseteq E^{m+1}$ the situation in which $e_{m+1}\in
C_{1}$ and $e_{m+1}\notin C_{2}$ is not possible.

Let we evaluate the three sums in \ref{EqLemaEncuentrosInternos_5}. In the
computation we use the induction hypothesis.

First sum:%
\begin{align}
&  \sum_{\varnothing\subseteq M_{1}\subseteq E^{m}}\sum_{M_{1}\subseteq
M_{2}\subseteq E^{m}}m_{X_{1}}\left(  M_{1}\right)  m_{X_{2}}\left(
M_{2}\backslash M_{1}\right) \label{EqLemaEncuentrosInternos_6}\\
&  =\sum_{\varnothing\subseteq M_{1}\subseteq E^{m}}\sum_{M_{1}\subseteq
M_{2}\subseteq E^{m}}\left(  1-\mu_{X_{1}}\left(  e_{m+1}\right)  \right)
m_{X_{1}}^{E^{m}}\left(  M_{1}\right)  \left(  1-\mu_{X_{2}}\left(
e_{m+1}\right)  \right)  m_{X_{2}}^{E^{m}}\left(  M_{2}\backslash M_{1}\right)
\nonumber\\
&  =\left(  1-\mu_{X_{1}}\left(  e_{m+1}\right)  \right)  \left(  1-\mu
_{X_{2}}\left(  e_{m+1}\right)  \right)  \sum_{\varnothing\subseteq
C_{1}\subseteq E^{m}}\sum_{C_{1}\subseteq C_{2}\subseteq E^{m}}m_{X_{1}%
}^{E^{m}}\left(  C_{1}\right)  m_{X_{2}}^{E^{m}}\left(  C_{2}\backslash
C_{1}\right) \nonumber\\
&  =\left(  1-\mu_{X_{1}}\left(  e_{m+1}\right)  \right)  \left(  1-\mu
_{X_{2}}\left(  e_{m+1}\right)  \right)
{\displaystyle\prod\limits_{i=1}^{m}}
\left(  1-\mu_{X_{1}}\left(  e_{i}\right)  \mu_{X_{2}}\left(  e_{i}\right)
\right) \nonumber
\end{align}

Second sum:%
\begin{align}
&  \sum_{\varnothing\subseteq M_{1}\subseteq E^{m}}\sum_{M_{1}\subseteq
M_{2}\subseteq E^{m}}m_{X_{1}}\left(  M_{1}\right)  m_{X_{2}}\left(  M_{2}%
\cup\left\{  e_{m+1}\right\}  \backslash M_{1}\right)
\label{EqLemaEncuentrosInternos_7}\\
&  =\sum_{\varnothing\subseteq M_{1}\subseteq E^{m}}\sum_{M_{1}\subseteq
M_{2}\subseteq E^{m}}\left(  1-\mu_{X_{1}}\left(  e_{m+1}\right)  \right)
m_{X_{1}}^{E^{m}}\left(  M_{1}\right)  \mu_{X_{2}}\left(  e_{m+1}\right)
m_{X_{2}}^{E^{m}}\left(  M_{2}\backslash M_{1}\right) \nonumber\\
&  =\left(  1-\mu_{X_{1}}\left(  e_{m+1}\right)  \right)  \mu_{X_{2}}\left(
e_{m+1}\right)  \sum_{\varnothing\subseteq C_{1}\subseteq E^{m}}\sum
_{C_{1}\subseteq C_{2}\subseteq E^{m}}m_{X_{1}}^{E^{m}}\left(  C_{1}\right)
m_{X_{2}}^{E^{m}}\left(  C_{2}\backslash C_{1}\right) \nonumber\\
&  =\left(  1-\mu_{X_{1}}\left(  e_{m+1}\right)  \right)  \mu_{X_{2}}\left(
e_{m+1}\right)
{\displaystyle\prod\limits_{i=1}^{m}}
\left(  1-\mu_{X_{1}}\left(  e_{i}\right)  \mu_{X_{2}}\left(  e_{i}\right)
\right) \nonumber
\end{align}

Third sum:
\begin{align}
&  \sum_{\varnothing\subseteq M_{1}\subseteq E^{m}}\sum_{M_{1}\subseteq
M_{2}\subseteq E^{m}}m_{X_{1}}\left(  M_{1}\cup\left\{  e_{m+1}\right\}
\right)  m_{X_{2}}\left(  M_{2}\cup\left\{  e_{m+1}\right\}  \backslash
M_{1}\cup\left\{  e_{m+1}\right\}  \right)  \label{EqLemaEncuentrosInternos_8}%
\\
&  =\sum_{\varnothing\subseteq M_{1}\subseteq E^{m}}\sum_{M_{1}\subseteq
M_{2}\subseteq E^{m}}\mu_{X_{1}}\left(  e_{m+1}\right)  m_{X_{1}}^{E^{m}%
}\left(  M_{1}\right)  \left(  1-\mu_{X_{2}}\left(  e_{m+1}\right)  \right)
m_{X_{2}}^{E^{m}}\left(  M_{2}\backslash M_{1}\right) \nonumber\\
&  =\mu_{X_{1}}\left(  e_{m+1}\right)  \left(  1-\mu_{X_{2}}\left(
e_{m+1}\right)  \right)
{\displaystyle\prod\limits_{i=1}^{m}}
\left(  1-\mu_{X_{1}}\left(  e_{i}\right)  \mu_{X_{2}}\left(  e_{i}\right)
\right) \nonumber
\end{align}

And using expressions \ref{EqLemaEncuentrosInternos_6},
\ref{EqLemaEncuentrosInternos_7} and \ref{EqLemaEncuentrosInternos_8}:
\begin{align*}
&  \sum_{\varnothing\subseteq C_{1}\subseteq E^{m+1}}\sum_{C_{1}\subseteq
C_{2}\subseteq E^{m+1}}m_{X_{1}}\left(  C_{1}\right)  m_{X_{2}}\left(
C_{2}\backslash C_{1}\right) \\
&  =\ldots\\
&  =(\left(  1-\mu_{X_{1}}\left(  e_{m+1}\right)  \right)  \left(
1-\mu_{X_{2}}\left(  e_{m+1}\right)  \right)  +(1-\mu_{X_{1}}\left(
e_{m+1}\right)  \mu_{X_{2}}\left(  e_{m+1}\right)  +\\
&  \mu_{X_{1}}\left(  e_{m+1}\right)  \left(  1-\mu_{X_{2}}\left(
e_{m+1}\right)  \right)  )\times%
{\displaystyle\prod\limits_{i=1}^{m}}
\left(  1-\mu_{X_{1}}\left(  e_{i}\right)  \mu_{X_{2}}\left(  e_{i}\right)
\right) \\
&  =((1-\mu_{X_{1}}\left(  e_{m+1}\right)  \left(  1-\mu_{X_{2}}\left(
e_{m+1}\right)  +\mu_{X_{2}}\left(  e_{m+1}\right)  \right) \\
&  \mu_{X_{1}}\left(  e_{m+1}\right)  \left(  1-\mu_{X_{2}}\left(
e_{m+1}\right)  \right)  )\times%
{\displaystyle\prod\limits_{i=1}^{m}}
\left(  1-\mu_{X_{1}}\left(  e_{i}\right)  \mu_{X_{2}}\left(  e_{i}\right)
\right) \\
&  =((1-\mu_{X_{1}}\left(  e_{m+1}\right)  +\mu_{X_{1}}\left(  e_{m+1}\right)
-\mu_{X_{1}}\left(  e_{m+1}\right)  \mu_{X_{2}}\left(  e_{m+1}\right)  )\\
&  \times%
{\displaystyle\prod\limits_{i=1}^{m}}
\left(  1-\mu_{X_{1}}\left(  e_{i}\right)  \mu_{X_{2}}\left(  e_{i}\right)
\right) \\
&  =\left(  1-\mu_{X_{1}}\left(  e_{m+1}\right)  \mu_{X_{2}}\left(
e_{m+1}\right)  \right)
{\displaystyle\prod\limits_{i=1}^{m}}
\left(  1-\mu_{X_{1}}\left(  e_{i}\right)  \mu_{X_{2}}\left(  e_{i}\right)
\right) \\
&  =%
{\displaystyle\prod\limits_{i=1}^{m+1}}
\left(  1-\mu_{X_{1}}\left(  e_{i}\right)  \mu_{X_{2}}\left(  e_{i}\right)
\right)
\end{align*}

In this way we have proved \ref{EqLemaEncuentrosInternos_4}, and the lemma is satisfied.
\end{proof}

\begin{lemma}
\label{LemaEncuentrosInternos2}Let $Q:\mathcal{P}\left(  E\right)
\rightarrow\mathbf{I}$ be a unary semi-fuzzy quantifier. Then it is fulfilled:%
\[
\sum_{Y_{1}\in\mathcal{P}\left(  E\right)  }\sum_{Y_{2}\in\mathcal{P}\left(
E\right)  }m_{X_{1}}\left(  Y_{1}\right)  m_{X_{2}}\left(  Y_{2}\right)
Q\left(  Y_{1}\cap Y_{2}\right)  =\sum_{Y\in\mathcal{P}\left(  E\right)
}m_{X_{1}\widetilde{\cap}X_{2}}\left(  Y\right)  Q\left(  Y\right)
\]

\end{lemma}

\begin{proof}
Using lemma \ref{LemaEncuentrosInternos}:%
\begin{align*}
&  \sum_{Y_{1}\in\mathcal{P}\left(  E\right)  }\sum_{Y_{2}\in\mathcal{P}%
\left(  E\right)  }m_{X_{1}}\left(  Y_{1}\right)  m_{X_{2}}\left(
Y_{2}\right)  Q\left(  Y_{1}\cap Y_{2}\right) \\
&  =\sum_{R\in\mathcal{P}\left(  E\right)  }\sum_{Y_{1}\in\mathcal{P}\left(
E\right)  }\sum_{Y_{2}\in\mathcal{P}\left(  E\right)  /Y_{1}\cap Y_{2}%
=R}m_{X_{1}}\left(  Y_{1}\right)  m_{X_{2}}\left(  Y_{2}\right)  Q\left(
R\right) \\
&  =\sum_{R\in\mathcal{P}\left(  E\right)  }Q\left(  R\right)  \sum_{Y_{1}%
\in\mathcal{P}\left(  E\right)  }\sum_{Y_{2}\in\mathcal{P}\left(  E\right)
/Y_{1}\cap Y_{2}=R}m_{X_{1}}\left(  Y_{1}\right)  m_{X_{2}}\left(
Y_{2}\right) \\
&  =\sum_{R\in\mathcal{P}\left(  E\right)  }m_{X_{1}\widetilde{\cap}X_{2}%
}\left(  R\right)  Q\left(  R\right) \\
&  =\sum_{Y\in\mathcal{P}\left(  E\right)  }m_{X_{1}\widetilde{\cap}X_{2}%
}\left(  Y\right)  Q\left(  Y\right)
\end{align*}

\end{proof}

And finally we prove the fulfillment of expression \ref{EqAxiomaZ4ModeloFA_1}:

\begin{proof}%
\begin{align*}
&  \mathcal{F}^{A}\left(  Q\cap\right)  \left(  X_{1},\ldots,X_{n}%
,X_{n+1}\right) \\
&  =\sum_{Y_{1}\in\mathcal{P}\left(  E\right)  }\ldots\sum_{Y_{n+1}%
\in\mathcal{P}\left(  E\right)  }m_{X_{1}}\left(  Y_{1}\right)  \ldots
m_{X_{n+1}}\left(  Y_{n+1}\right)  \left(  Q\cap\right)  \left(  Y_{1}%
,\ldots,Y_{n+1}\right) \\
&  =\sum_{Y_{1}\in\mathcal{P}\left(  E\right)  }\ldots\sum_{Y_{n+1}%
\in\mathcal{P}\left(  E\right)  }m_{X_{1}}\left(  Y_{1}\right)  \ldots
m_{X_{n+1}}\left(  Y_{n+1}\right)  Q\left(  Y_{1},\ldots,Y_{n}\cap
Y_{n+1}\right) \\
&  =\sum_{Y_{1}\in\mathcal{P}\left(  E\right)  }\ldots\sum_{Y_{n-1}%
\in\mathcal{P}\left(  E\right)  }m_{X_{1}}\left(  Y_{1}\right)  \ldots
m_{X_{n-1}}\left(  Y_{n-1}\right) \\
&  \sum_{Y_{n}\in\mathcal{P}\left(  E\right)  }\sum_{Y_{n+1}\in\mathcal{P}%
\left(  E\right)  }m_{X_{n}}\left(  Y_{n}\right)  m_{X_{n+1}}\left(
Y_{n+1}\right)  Q\left(  Y_{1},\ldots,Y_{n-1},Y_{n}\cap Y_{n+1}\right) \\
&  =\sum_{Y_{1}\in\mathcal{P}\left(  E\right)  }\ldots\sum_{Y_{n-1}%
\in\mathcal{P}\left(  E\right)  }m_{X_{1}}\left(  Y_{1}\right)  \ldots
m_{X_{n-1}}\left(  Y_{n-1}\right) \\
&  \sum_{R\in\mathcal{P}\left(  E\right)  }\sum_{Y_{n}\in\mathcal{P}\left(
E\right)  }\sum_{Y_{n+1}\in\mathcal{P}\left(  E\right)  /Y_{n}\cap Y_{n+1}%
=R}m_{X_{n}}\left(  Y_{n}\right)  m_{X_{n+1}}\left(  Y_{n+1}\right)  Q\left(
Y_{1},\ldots,Y_{n-1,},R\right) \\
&  =\sum_{Y_{1}\in\mathcal{P}\left(  E\right)  }\ldots\sum_{Y_{n-1}%
\in\mathcal{P}\left(  E\right)  }m_{X_{1}}\left(  Y_{1}\right)  \ldots
m_{X_{n-1}}\left(  Y_{n-1}\right) \\
&  \sum_{R\in\mathcal{P}\left(  E\right)  }m_{X_{n}\widetilde{\cap}X_{n+1}%
}\left(  R\right)  Q\left(  Y_{1},\ldots,Y_{n-1},R\right)
\end{align*}
Where we have using \ref{LemaEncuentrosInternos2}. And then,%
\begin{align*}
&  \mathcal{F}^{A}\left(  Q\cap\right)  \left(  X_{1},\ldots,X_{n}%
,X_{n+1}\right) \\
&  =\ldots\\
&  =\sum_{Y_{1}\in\mathcal{P}\left(  E\right)  }\ldots\sum_{Y_{n-1}%
\in\mathcal{P}\left(  E\right)  }m_{X_{1}}\left(  Y_{1}\right)  \ldots
m_{X_{n-1}}\left(  Y_{n-1}\right) \\
&  \sum_{R\in\mathcal{P}\left(  E\right)  }m_{X_{n}\widetilde{\cap}X_{n+1}%
}\left(  R\right)  Q\left(  Y_{1},\ldots,Y_{n-1},R\right) \\
&  =\mathcal{F}^{A}\left(  Q\right)  \left(  X_{1},\ldots,X_{n}\widetilde
{\cap}X_{n+1}\right)
\end{align*}

\end{proof}

Now we prove the fulfillment of the Z-4 axiom:%
\[
\mathcal{F}\left(  Q\cup\right)  =\mathcal{F}\left(  Q\right)  \widetilde
{\cup},\quad n>0
\]
\ 

\begin{proof}
For a semi-fuzzy quantifier $Q:\mathcal{P}\left(  E\right)  ^{n}%
\rightarrow\mathbf{I}$ it is fulfilled ($\tau_{n}$ represents the trasposition
of the $n$ and $n+1$ element \cite[section 4.5]{Glockner06Libro}):%
\begin{align*}
\left(  Q\lnot\cap\lnot\tau_{n}\lnot\right)   &  =\left(  f^{\prime}:\left(
Y_{1}^{\prime},\ldots,Y_{n}^{\prime}\right)  \rightarrow Q\left(
Y_{1}^{\prime},\ldots,\lnot Y_{n}^{\prime}\right)  \right)  \cap\lnot\tau
_{n}\lnot\\
&  =\left(  f^{\prime\prime}:\left(  Y_{1}^{\prime\prime},\ldots,Y_{n}%
^{\prime\prime},Y_{n+1}^{\prime\prime}\right)  \rightarrow f^{\prime}\left(
Y_{1}^{\prime\prime},\ldots,Y_{n}^{\prime\prime}\cap Y_{n+1}^{\prime\prime
}\right)  \right)  \lnot\tau_{n}\lnot\\
&  =\left(  f^{\prime\prime}:\left(  Y_{1}^{\prime\prime},\ldots,Y_{n}%
^{\prime\prime},Y_{n+1}^{\prime\prime}\right)  \rightarrow Q\left(
Y_{1}^{\prime\prime},\ldots,\lnot\left(  Y_{n}^{\prime\prime}\cap
Y_{n+1}^{\prime\prime}\right)  \right)  \right)  \lnot\tau_{n}\lnot\\
&  =\left(  f^{\prime\prime\prime}:\left(  Y_{1}^{\prime\prime\prime}%
,\ldots,Y_{n}^{\prime\prime\prime},Y_{n+1}^{\prime\prime\prime}\right)
\rightarrow f^{\prime\prime}:\left(  Y_{1}^{\prime\prime\prime},\ldots
,Y_{n}^{\prime\prime\prime},\lnot Y_{n+1}^{\prime\prime\prime}\right)
\right)  \tau_{n}\lnot\\
&  =\left(  f^{\prime\prime\prime}:\left(  Y_{1}^{\prime\prime\prime}%
,\ldots,Y_{n}^{\prime\prime\prime},Y_{n+1}^{\prime\prime\prime}\right)
\rightarrow Q:\left(  Y_{1}^{\prime\prime\prime},\ldots,\lnot\left(
Y_{n}^{\prime\prime}\cap\lnot Y_{n+1}^{\prime\prime}\right)  \right)  \right)
\tau_{n}\lnot\\
&  =\left(  f^{\prime\prime\prime\prime}:\left(  Y_{1}^{\prime\prime
\prime\prime},\ldots,Y_{n}^{\prime\prime\prime\prime},Y_{n+1}^{\prime
\prime\prime\prime}\right)  \rightarrow f^{\prime\prime\prime}:\left(
Y_{1}^{\prime\prime\prime\prime},\ldots,Y_{n+1}^{\prime\prime\prime\prime
},Y_{n}^{\prime\prime\prime\prime}\right)  \right)  \lnot\\
&  =\left(  f^{\prime\prime\prime\prime}:\left(  Y_{1}^{\prime\prime
\prime\prime},\ldots,Y_{n}^{\prime\prime\prime\prime},Y_{n+1}^{\prime
\prime\prime\prime}\right)  \rightarrow Q:\left(  Y_{1}^{\prime\prime
\prime\prime},\ldots,\lnot\left(  Y_{n+1}^{\prime\prime\prime\prime}\cap\lnot
Y_{n}^{\prime\prime\prime\prime}\right)  \right)  \right)  \lnot\\
&  =\left(  f^{\prime\prime\prime\prime\prime}:\left(  Y_{1}^{\prime
\prime\prime\prime\prime},\ldots,Y_{n}^{\prime\prime\prime\prime\prime
},Y_{n+1}^{\prime\prime\prime\prime\prime}\right)  \rightarrow f^{\prime
\prime\prime\prime}:\left(  Y_{1}^{\prime\prime\prime\prime},\ldots
,Y_{n}^{\prime\prime\prime\prime},\lnot Y_{n+1}^{\prime\prime\prime\prime
}\right)  \right) \\
&  =\left(  f^{\prime\prime\prime\prime\prime}:\left(  Y_{1}^{\prime
\prime\prime\prime\prime},\ldots,Y_{n}^{\prime\prime\prime\prime\prime
},Y_{n+1}^{\prime\prime\prime\prime\prime}\right)  \rightarrow Q\left(
Y_{1}^{\prime\prime\prime\prime\prime},\ldots,\lnot\left(  \lnot Y_{n}%
^{\prime\prime\prime\prime\prime}\cap\lnot Y_{n+1}^{\prime\prime\prime
\prime\prime}\right)  \right)  \right)
\end{align*}
and then%
\begin{align*}
\left(  Q\lnot\cap\lnot\tau_{n}\lnot\right)  \left(  Y_{1},\ldots
,Y_{n,}Y_{n+1}\right)   &  =Q\left(  Y_{1},\ldots,Y_{n-1},\lnot\left(  \lnot
Y_{n}\cap\lnot Y_{n+1}\right)  \right) \\
&  =Q\left(  Y_{1},\ldots,Y_{n-1},Y_{n}\cup Y_{n+1}\right)
\end{align*}

In this way, we can use the properties of external negation, internal negation
and trasposition of arguments (trivially fulfilled) and the expression
\ref{EqAxiomaZ4ModeloFA_1} to obtain:%
\begin{align*}
\mathcal{F}^{A}\left(  Q\cup\right)  \left(  X_{1},\ldots,X_{n},X_{n+1}%
\right)   &  =\mathcal{F}^{A}\left(  Q\lnot\cap\lnot\tau_{n}\lnot\right)
\left(  X_{1},\ldots,X_{n},X_{n+1}\right) \\
&  =\mathcal{F}^{A}\left(  Q\right)  \widetilde{\lnot}\widetilde{\cap
}\widetilde{\lnot}\tau_{n}\widetilde{\lnot}\left(  X_{1},\ldots,X_{n}%
,X_{n+1}\right) \\
&  =\mathcal{F}^{A}\left(  Q\right)  \left(  X_{1},\ldots,\widetilde{\lnot
}\left(  \widetilde{\lnot}X_{n}\widetilde{\cap}\widetilde{\lnot}%
X_{n+1}\right)  \right) \\
&  =\mathcal{F}^{A}\left(  Q\right)  \left(  X_{1},\ldots,X_{n}\widetilde
{\cup}X_{n+1}\right)
\end{align*}
where in the last step we use that $\widetilde{\cap}$ and $\widetilde{\cup}$
are constructed by means of dual operators.
\end{proof}

\paragraph{Axiom Z-5}

It holds that

\begin{quote}
If $Q$ is nonincreasing in the $n$-th arg, then $\mathcal{F}\left(  Q\right)
$ is nonincreasing in the $n$-th arg, $n>0$.
\end{quote}

\begin{proof}
We will consider first the unary case.

Let $Q:\mathcal{P}\left(  E\right)  \rightarrow\mathbf{I}$ be a nonincreasing
semi-fuzzy quantifier. We will proof that $\mathcal{F}^{A}\left(  Q\right)
\left(  X\right)  $ is nonincreasing using induction on the cardinality of the referential.

\textbf{Case} $i=1$; that is, the referential contains only one element
($E=E^{1}=\left\{  e_{1}\right\}  $).

Let $X,X^{\prime}\in\widetilde{\mathcal{P}}\left(  E\right)  $ be fuzzy sets
fulfilling $X\subseteq X^{\prime}$. Note
\[
\mu_{X}\left(  e_{1}\right)  \leq\mu_{X^{\prime}}\left(  e_{1}\right)
,1-\mu_{X}\left(  e_{1}\right)  \geq1-\mu_{X^{\prime}}\left(  e_{1}\right)
\]

As $Q:\mathcal{P}\left(  E\right)  \rightarrow\mathbf{I}$ is monotonic
nonincreasing for $h\geq0$ it holds that%
\begin{equation}
aQ\left(  \varnothing\right)  +bQ\left(  \left\{  e_{1}\right\}  \right)
\geq\left(  a-h\right)  Q\left(  \varnothing\right)  +\left(  b+h\right)
Q\left(  \left\{  e_{1}\right\}  \right)  \label{EqAxiomaZ5ModeloFA_1}%
\end{equation}
as a consequence of%
\[
\left(  a-h\right)  Q\left(  \varnothing\right)  +\left(  b+h\right)  Q\left(
\left\{  e_{1}\right\}  \right)  =aQ\left(  \varnothing\right)  +bQ\left(
\left\{  e_{1}\right\}  \right)  +h\left(  Q\left(  \left\{  e_{1}\right\}
\right)  -Q\left(  \varnothing\right)  \right)
\]
and $Q\left(  \left\{  e_{1}\right\}  \right)  -Q\left(  \varnothing\right)
\leq0$ because $Q$ is nonincreasing.

Let be $h=\mu_{X^{\prime}}\left(  e_{1}\right)  -\mu_{X}\left(  e_{1}\right)
\geq0$. By \ref{EqAxiomaZ5ModeloFA_1}%
\begin{align*}
&  \mathcal{F}^{A}\left(  Q\right)  \left(  X\right) \\
&  =\left(  1-\mu_{X}\left(  e_{1}\right)  \right)  Q\left(  \varnothing
\right)  +\mu_{X}\left(  e_{1}\right)  Q\left(  \left\{  e_{1}\right\}
\right) \\
&  \geq\left(  1-\mu_{X}\left(  e_{1}\right)  -\left(  \mu_{X^{\prime}}\left(
e_{1}\right)  -\mu_{X}\left(  e_{1}\right)  \right)  \right)  Q\left(
\varnothing\right)  +\left(  \mu_{X}\left(  e_{1}\right)  +\left(
\mu_{X^{\prime}}\left(  e_{1}\right)  -\mu_{X}\left(  e_{1}\right)  \right)
\right)  Q\left(  \left\{  e_{1}\right\}  \right) \\
&  =\left(  1-\mu_{X}\left(  e_{1}\right)  -\mu_{X^{\prime}}\left(
e_{1}\right)  +\mu_{X}\left(  e_{1}\right)  \right)  Q\left(  \varnothing
\right)  +\left(  \mu_{X}\left(  e_{1}\right)  +\mu_{X^{\prime}}\left(
e_{1}\right)  -\mu_{X}\left(  e_{1}\right)  \right)  Q\left(  \left\{
e_{1}\right\}  \right) \\
&  =\left(  1-\mu_{X^{\prime}}\left(  e_{1}\right)  \right)  Q\left(
\varnothing\right)  +\mu_{X^{\prime}}\left(  e_{1}\right)  Q\left(  \left\{
e_{1}\right\}  \right)  =\mathcal{F}^{A}\left(  Q\right)  \left(  X^{\prime
}\right)
\end{align*}
and the property is fulfilled for a one referential set

\textbf{Hypothesis of induction: Case} $i=m$ ($E=E^{m}=\left\{  e_{1}%
,\ldots,e_{m}\right\}  $).\ For $X,X^{\prime}\in\widetilde{\mathcal{P}}\left(
E^{m}\right)  $ such that $X\subseteq X^{\prime}$ it holds that $\mathcal{F}%
^{A}\left(  Q\right)  \left(  X\right)  \geq\mathcal{F}^{A}\left(  Q\right)
\left(  X^{\prime}\right)  $.

\textbf{Case }$i=m+1$ ($E=E^{m+1}=\left\{  e_{1},\ldots,e_{m+1}\right\}  $).

Based on $Q:\mathcal{P}\left(  E^{m+1}\right)  \rightarrow\mathbf{I}$ we
define the semi-fuzzy quantifiers $Q^{a}:\mathcal{P}\left(  E^{m}\right)
\rightarrow\mathbf{I}$ and $Q^{b}:\mathcal{P}\left(  E^{m}\right)
\rightarrow\mathbf{I}$ as
\begin{align*}
Q^{a}\left(  X\right)   &  =Q\left(  X\right)  ,X\in\mathcal{P}\left(
E^{m}\right) \\
Q^{b}\left(  X\right)   &  =Q\left(  X\cup\left\{  e_{m+1}\right\}  \right)
,X\in\mathcal{P}\left(  E^{m}\right)
\end{align*}

$Q^{a}$ and $Q^{b}$ are monotonic nonincreasing on $E^{m}$.

Let $h=\left(  \mu_{X^{\prime}}\left(  e_{m+1}\right)  -\mu_{X}\left(
e_{m+1}\right)  \right)  $. Then%
\begin{align*}
&  \mathcal{F}^{A}\left(  Q\right)  \left(  X\right) \\
&  =\sum_{Y\in\mathcal{P}\left(  E^{m+1}\right)  }m_{X}\left(  Y\right)
Q\left(  Y\right) \\
&  =\sum_{Y\in\mathcal{P}\left(  E^{m}\right)  }\left(  1-\mu_{X}\left(
e_{m+1}\right)  \right)  m_{X}\left(  Y\right)  Q\left(  Y\right)  +\sum
_{Y\in\mathcal{P}\left(  E^{m}\right)  }\mu_{X}\left(  e_{m+1}\right)
m_{X}\left(  Y\right)  Q\left(  Y\cup\left\{  e_{m+1}\right\}  \right) \\
&  \geq\left(  1-\mu_{X}\left(  e_{m+1}\right)  -h\right)  \sum_{Y\in
\mathcal{P}\left(  E^{m}\right)  }m_{X}\left(  Y\right)  Q\left(  Y\right)
+\left(  \mu_{X}\left(  e_{m+1}\right)  +h\right)  \sum_{Y\in\mathcal{P}%
\left(  E^{m}\right)  }m_{X}\left(  Y\right)  Q\left(  Y\cup\left\{
e_{m+1}\right\}  \right) \\
&  =\left(  1-\mu_{X}\left(  e_{m+1}\right)  -\left(  \mu_{X^{\prime}}\left(
e_{m+1}\right)  -\mu_{X}\left(  e_{m+1}\right)  \right)  \right)  \sum
_{Y\in\mathcal{P}\left(  E^{m}\right)  }m_{X}\left(  Y\right)  Q\left(
Y\right) \\
&  +\left(  \mu_{X}\left(  e_{m+1}\right)  +\left(  \mu_{X^{\prime}}\left(
e_{m+1}\right)  -\mu_{X}\left(  e_{m+1}\right)  \right)  \right)  \sum
_{Y\in\mathcal{P}\left(  E^{m}\right)  }m_{X}\left(  Y\right)  Q\left(
Y\cup\left\{  e_{m+1}\right\}  \right) \\
&  =\left(  1-\mu_{X^{\prime}}\left(  e_{m+1}\right)  \right)  \sum
_{Y\in\mathcal{P}\left(  E^{m}\right)  }m_{X}\left(  Y\right)  Q\left(
Y\right)  +\mu_{X^{\prime}}\left(  e_{m+1}\right)  \sum_{Y\in\mathcal{P}%
\left(  E^{m}\right)  }m_{X}\left(  Y\right)  Q\left(  Y\cup\left\{
e_{m+1}\right\}  \right)
\end{align*}
\newline

And using the induction hypothesis%
\begin{align*}
\sum_{Y\in\mathcal{P}\left(  E^{m}\right)  }m_{X}\left(  Y\right)
Q^{a}\left(  Y\right)   &  \geq\sum_{Y\in\mathcal{P}\left(  E^{m}\right)
}m_{X^{\prime}}\left(  Y\right)  Q^{a}\left(  Y\right) \\
\sum_{Y\in\mathcal{P}\left(  E^{m}\right)  }m_{X}\left(  Y\right)
Q^{b}\left(  Y\right)   &  \geq\sum_{Y\in\mathcal{P}\left(  E^{m}\right)
}m_{X^{\prime}}\left(  Y\right)  Q^{b}\left(  Y\right)
\end{align*}
because $Q^{a}:\mathcal{P}\left(  E^{m}\right)  \rightarrow\mathbf{I}$ and
$Q^{b}:\mathcal{P}\left(  E^{m}\right)  \rightarrow\mathbf{I}$ are monotonic
nonincreasing on a referential of $m$ elements. We continue the computation:%
\begin{align*}
&  \mathcal{F}^{A}\left(  Q\right)  \left(  X\right) \\
&  \geq\left(  1-\mu_{X^{\prime}}\left(  e_{m+1}\right)  \right)  \sum
_{Y\in\mathcal{P}\left(  E^{m}\right)  }m_{X}\left(  Y\right)  Q\left(
Y\right)  +\mu_{X^{\prime}}\left(  e_{m+1}\right)  \sum_{Y\in\mathcal{P}%
\left(  E^{m}\right)  }m_{X}\left(  Y\right)  Q\left(  Y\cup\left\{
e_{m+1}\right\}  \right) \\
&  \geq\left(  1-\mu_{X^{\prime}}\left(  e_{m+1}\right)  \right)  \sum
_{Y\in\mathcal{P}\left(  E^{m}\right)  }m_{X^{\prime}}\left(  Y\right)
Q^{a}\left(  Y\right)  +\mu_{X^{\prime}}\left(  e_{m+1}\right)  \sum
_{Y\in\mathcal{P}\left(  E^{m}\right)  }m_{X^{\prime}}\left(  Y\right)
Q^{b}\left(  Y\right) \\
&  =\sum_{Y\in\mathcal{P}\left(  E^{m}\right)  }m_{X^{\prime}}\left(
Y\right)  Q\left(  Y\right) \\
&  =\mathcal{F}^{A}\left(  Q\right)  \left(  X^{\prime}\right)
\end{align*}

Let us consider now the general case.Let $Q:\mathcal{P}\left(  E\right)
^{n}\rightarrow\mathbf{I}$ be an $n$-ary semi-fuzzy quantifier non increasing
in its $n\,$argument. Then,
\begin{align*}
&  \mathcal{F}^{A}\left(  Q\right)  \left(  X_{1},\ldots,X_{n}\right) \\
&  =\sum_{Y_{1}\in\mathcal{P}\left(  E\right)  }\ldots\sum_{Y_{n}%
\in\mathcal{P}\left(  E\right)  }m_{X_{1}}\left(  Y_{1}\right)  \ldots
m_{X_{n}}\left(  Y_{n}\right)  Q\left(  Y_{1},\ldots,Y_{n}\right) \\
&  =\sum_{Y_{1}\in\mathcal{P}\left(  E\right)  }\ldots\sum_{Y_{n-1}%
\in\mathcal{P}\left(  E\right)  }m_{X_{1}}\left(  Y_{1}\right)  \ldots
m_{X_{n-1}}\left(  Y_{n-1}\right)  \sum_{Y_{n}\in\mathcal{P}\left(  E\right)
}m_{X_{n}}\left(  Y_{n}\right)  Q\left(  Y_{1},\ldots,Y_{n}\right)
\end{align*}
In
\[
\sum_{Y_{n}\in\mathcal{P}\left(  E\right)  }m_{X_{n}}\left(  Y_{n}\right)
Q\left(  Y_{1},\ldots,Y_{n}\right)
\]
the $Y_{1},\ldots,Y_{n-1}$ are constant. The unary semi-fuzzy quantifier
$Q^{\prime}:\mathcal{P}\left(  E\right)  \rightarrow\mathbf{I}$%
\[
Q^{\prime}\left(  Y\right)  =Q\left(  Y_{1},\ldots,Y_{n-1},Y\right)
\]
is monotonic non increasing, and then $\mathcal{F}^{A}\left(  Q^{\prime
}\right)  $ is also monotonic non increasing.\ As this fact is fulfilled for
all%
\[
Y_{1},\ldots,Y_{n-1}\in\mathcal{P}\left(  E\right)
\]
then the proposition is fulfilled.
\end{proof}

\paragraph{Axiom Z-6}

It holds that%
\begin{equation}
\mathcal{F}\left(  Q\circ\underset{i=1}{\overset{n}{\times}}\widehat{f_{i}%
}\right)  =\mathcal{F}\left(  Q\right)  \circ\underset{i=1}{\overset{n}%
{\times}}\widehat{\mathcal{F}}\left(  f_{i}\right)  \text{ where }f_{1}%
,\ldots,f_{n}:E^{\prime}\rightarrow E,E^{\prime}\neq\varnothing
\label{EqAxiomaZ6ModeloFA}%
\end{equation}

To prove this property we need some previous results.

\textbf{Existential quantifier}

\begin{proposition}
[$\mathcal{F}^{A}\left(  \exists\right)  \left(  X\right)  $]%
\label{PropAxiomaZ6ModeloFA_1}Let $X\in\widetilde{\mathcal{P}}\left(
E\right)  $ a fuzzy set. Then%
\[
\mathcal{F}^{A}\left(  \exists\right)  \left(  X\right)  =\sup\left\{
\overset{m}{\underset{i=1}{\widetilde{\vee}}}\mu_{X}\left(  a_{i}\right)
:A=\left\{  a_{1},\ldots,a_{m}\right\}  \in\mathcal{P}\left(  E\right)
,a_{i}\neq a_{j}\text{ if }i\neq j\right\}
\]

\end{proposition}

\begin{proof}
Let $E=\left\{  e_{1},\ldots,e_{m}\right\}  $ be given. Using the lemma
\ref{LemaTConormaProbabilistica}%
\begin{align*}
\mathcal{F}^{A}\left(  \exists\right)  \left(  X\right)   &  =\sum
_{Y\in\mathcal{P}\left(  E\right)  }m_{X}\left(  Y\right)  Q\left(  Y\right)
\\
&  =\sum_{Y\in\mathcal{P}\left(  E\right)  |Y\neq\varnothing}m_{X}\left(
Y\right) \\
&  =1-m_{X}\left(  \varnothing\right) \\
&  =1-%
{\displaystyle\prod\limits_{i=1}^{m}}
\left(  1-\mu_{X}\left(  e_{i}\right)  \right) \\
&  =\overset{m}{\underset{i=1}{\widetilde{\vee}}}\mu_{X}\left(  e_{i}\right)
\end{align*}

\end{proof}

\textbf{Universal quantifier}

\begin{proposition}
[$\mathcal{F}^{A}\left(  \forall\right)  \left(  X\right)  $]%
\label{PropAxiomaZ6ModeloFA_2}Let $X\in\widetilde{\mathcal{P}}\left(
E\right)  $ a fuzzy set. Then%
\[
\mathcal{F}^{A}\left(  \forall\right)  \left(  X\right)  =\inf\left\{
\overset{m}{\underset{i=1}{\widetilde{\wedge}}}\mu_{X}\left(  a_{i}\right)
:A=\left\{  a_{1},\ldots,a_{m}\right\}  \in\mathcal{P}\left(  E\right)
,a_{i}\neq a_{j}\text{ si }i\neq j\right\}
\]

\end{proposition}

\begin{proof}
Let $E=\left\{  e_{1},\ldots,e_{m}\right\}  $ be given. Then
\begin{align*}
\mathcal{F}^{A}\left(  \forall\right)  \left(  X\right)   &  =\sum
_{Y\in\mathcal{P}\left(  E\right)  }m_{X}\left(  Y\right)  \forall\left(
Y\right) \\
&  =m_{X}\left(  E\right) \\
&  =%
{\displaystyle\prod\limits_{i=1}^{m}}
\mu_{X}\left(  e_{i}\right)
\end{align*}

\end{proof}

\textbf{Induced extension principle.}

To compute the induce extension principle of the $\mathcal{F}^{A}$ we will use
the definition \ref{DefPropFuncAplPrinExtension}.

\begin{notation}
[$\widetilde{\vee}:\widetilde{\mathcal{P}}\left(  E\right)  \rightarrow
\mathbf{I}$]\label{NotacionAxiomaZ6GeneralizacionOR}Let $E=\left\{
e_{1},\ldots,e_{m}\right\}  $ and $X\in$ $\widetilde{\mathcal{P}}\left(
E\right)  $ a fuzzy set. By $\widetilde{\vee}:\widetilde{\mathcal{P}}\left(
E\right)  \rightarrow\mathbf{I}$ we will denote the generalization of the
induced tconorm to fuzzy sets; that is,%
\[
\widetilde{\vee}\left(  X\right)  =\left\{
\begin{array}
[c]{ccc}%
\widetilde{\vee}\left(  \mu_{X}\left(  e_{1}\right)  ,\widetilde{\vee}\left(
\ldots,\widetilde{\vee}\left(  \mu_{X}\left(  e_{m-1}\right)  ,\mu_{X}\left(
e_{m}\right)  \right)  \right)  \right)  & : & m>0\\
0 & : & m=0
\end{array}
\right.
\]

\end{notation}

\begin{proposition}
\label{PropAxiomaZ6ModeloFA_3}Let $f:E\rightarrow E^{\prime}$ (where
$E,E^{\prime}\neq\varnothing$). Let $e_{i}^{\prime}\in E^{\prime}$ be given
and let $\widehat{f^{-1}}\left(  e_{i}^{\prime}\right)  =\left\{  e_{i_{1}%
},\ldots e_{i_{k_{i}}}\right\}  $ be the inverse image of $e_{i}^{\prime}$.
Then the induced extension principle of $\mathcal{F}^{A}$ for $f$ is
\begin{align}
\mu_{\widehat{\mathcal{F}}\left(  f\right)  \left(  X\right)  }\left(
e_{i}^{\prime}\right)   &  =\left\{
\begin{array}
[c]{ccc}%
\widetilde{\vee}\left(  \left\{  \mu_{X}\left(  e_{i_{1}}\right)  /e_{i_{1}%
},\ldots,\mu_{X}\left(  e_{i_{k_{i}}}\right)  /e_{i_{k_{i}}}\right\}  \right)
& : & k_{i}\geq1\\
0 & : & \text{otherwise}%
\end{array}
\right. \label{EqAxiomaZ6ModeloFA_1}\\
&  =\mathcal{F}^{A}\left(  \exists_{\widehat{f^{-1}}\left(  e_{i}^{\prime
}\right)  }\right)  \left(  X^{\widehat{f^{-1}}\left(  e_{i}^{\prime}\right)
}\right) \nonumber
\end{align}
where $\exists_{\widehat{f^{-1}}\left(  e_{i}^{\prime}\right)  }$ represents
the existential quantifier on the base universe $\widehat{f^{-1}}\left(
e_{i}^{\prime}\right)  $ and $X^{\widehat{f^{-1}}\left(  e_{i}^{\prime
}\right)  }$ is the proyection of the fuzzy set $X$ over $\widehat{f^{-1}%
}\left(  e_{i}^{\prime}\right)  $.
\end{proposition}

\begin{proof}%
\begin{align}
\mu_{\widehat{\mathcal{F}^{A}}\left(  f\right)  \left(  X\right)  }\left(
e^{\prime}\right)   &  =\mathcal{F}^{A}\left(  \chi_{\widehat{f}\left(
\cdot\right)  }\left(  e^{\prime}\right)  \right)  \left(  X\right)
\label{EqAxiomaZ6ModeloFA_2}\\
&  =\sum_{Y\in\mathcal{P}\left(  E\right)  }m_{X}\left(  Y\right)
\chi_{\widehat{f}\left(  \cdot\right)  }\left(  e^{\prime}\right)  \left(
Y\right) \nonumber\\
&  =\sum_{Y\in\mathcal{P}\left(  E\right)  }m_{X}\left(  Y\right)
\chi_{\widehat{f}\left(  Y\right)  }\left(  e^{\prime}\right) \nonumber\\
&  =\sum_{Y\in\mathcal{P}\left(  E\right)  }m_{X}\left(  Y\right)  \left\{
\begin{array}
[c]{ccc}%
0 & : & e^{\prime}\notin\widehat{f}\left(  Y\right) \\
1 & : & e^{\prime}\in\widehat{f}\left(  Y\right)
\end{array}
\right. \nonumber\\
&  =\sum_{Y\in\mathcal{P}\left(  E\right)  |\exists e\in Y,f\left(  e\right)
=e^{\prime}}m_{X}\left(  Y\right) \nonumber
\end{align}
If $\widehat{f^{.-1}}\left(  e^{\prime}\right)  =\varnothing$ the previous sum
is $0$ and we are in the second situation of expression
\ref{EqAxiomaZ6ModeloFA_1}.

Let us suppose $\widehat{f^{.-1}}\left(  e^{\prime}\right)  \neq\varnothing$.
\end{proof}

It should be noted that for each $Y\in\mathcal{P}\left(  E\right)
\,$fulfilling $\exists e\in Y,$ $f\left(  e\right)  =e^{\prime}$ then the
intersection of $Y$ with the inverse image of $e^{\prime}$ ($\widehat{f^{.-1}%
}\left(  e^{\prime}\right)  $) is not empty. As for all $Y$ fulfilling this
condition can be decomposed in the part intersecting with $\widehat{f^{.-1}%
}\left(  e^{\prime}\right)  $, and the part that does not intersect with
$\widehat{f^{.-1}}\left(  e^{\prime}\right)  $ ($E\backslash\widehat{f}%
^{-1}\left(  e^{\prime}\right)  $) expression \ref{EqAxiomaZ6ModeloFA_2} is
equal to%
\begin{align*}
\mu_{\widehat{\mathcal{F}^{A}}\left(  f\right)  \left(  X\right)  }\left(
e^{\prime}\right)   &  =\ldots\\
&  =\sum_{\varnothing\subset M\subseteq\widehat{f}^{-1}\left(  e^{\prime
}\right)  }\sum_{\varnothing\subseteq R\subseteq E\backslash\widehat{f}%
^{-1}\left(  e^{\prime}\right)  }m_{X}\left(  M\cup R\right) \\
&  =\sum_{\varnothing\subset M\subseteq\widehat{f}^{-1}\left(  e^{\prime
}\right)  }\sum_{\varnothing\subseteq R\subseteq E\backslash\widehat{f}%
^{-1}\left(  e^{\prime}\right)  }m_{X}^{\widehat{f}^{-1}\left(  e^{\prime
}\right)  }\left(  M\right)  m_{X}^{E\backslash\widehat{f}^{-1}\left(
e^{\prime}\right)  }\left(  R\right) \\
&  =\sum_{\varnothing\subset M\subseteq\widehat{f}^{-1}\left(  e^{\prime
}\right)  }m_{X}^{\widehat{f}^{-1}\left(  e^{\prime}\right)  }\sum
_{\varnothing\subseteq R\subseteq E\backslash\widehat{f}^{-1}\left(
e^{\prime}\right)  }\left(  M\right)  m_{X}^{E\backslash\widehat{f}%
^{-1}\left(  e^{\prime}\right)  }\left(  R\right) \\
&  =\sum_{\varnothing\subset M\subseteq\widehat{f}^{-1}\left(  e^{\prime
}\right)  }m_{X_{1}}^{\widehat{f}^{-1}\left(  e^{\prime}\right)  }\left(
M\right)  \cdot1
\end{align*}
As $\varnothing\subseteq R\subseteq E\backslash\widehat{f}^{-1}\left(
e^{\prime}\right)  $ contains all the sets of $E\backslash\widehat{f}%
^{-1}\left(  e^{\prime}\right)  $. And then,
\begin{align*}
\mu_{\widehat{\mathcal{F}^{A}}\left(  f\right)  \left(  X\right)  }\left(
e^{\prime}\right)   &  =\ldots\\
&  =\sum_{\varnothing\subset M\subseteq\widehat{f}^{-1}\left(  e^{\prime
}\right)  }m_{X}^{\widehat{f}^{-1}\left(  e^{\prime}\right)  }\left(  M\right)
\\
&  =\mathcal{F}^{A}\left(  \exists_{\widehat{f}^{-1}\left(  e^{\prime}\right)
}\right)  \left(  X^{\widehat{f}^{-1}\left(  e^{\prime}\right)  }\right)
\end{align*}
If we denote $\widehat{f^{-1}}\left(  e^{\prime}\right)  =\left\{  e_{i_{1}%
},\ldots e_{i_{k}}\right\}  $ using expression \ref{PropAxiomaZ6ModeloFA_1} we
obtain%
\[
\mu_{\widehat{\mathcal{F}^{A}}\left(  f\right)  \left(  X\right)  }\left(
e^{\prime}\right)  =\widetilde{\vee}\left(  \left\{  \mu_{X}\left(  e_{i_{1}%
}\right)  /e_{i_{1}},\ldots,\mu_{X}\left(  e_{i_{k}}\right)  /e_{i_{k}%
}\right\}  \right)
\]

Now we will prove the fulfillment of
\[%
\begin{tabular}
[c]{l}%
$\mathcal{F}\left(  Q\circ\underset{i=1}{\overset{n}{\times}}\widehat{f_{i}%
}\right)  =\mathcal{F}\left(  Q\right)  \circ\underset{i=1}{\overset{n}%
{\times}}\widehat{\mathcal{F}}\left(  f_{i}\right)  $\\
where $f_{1},\ldots,f_{n}:E^{\prime}\rightarrow E,E^{\prime}\neq\varnothing$%
\end{tabular}
\ \
\]

\begin{proof}
Let $E=\left\{  e_{1},\ldots,e_{m}\right\}  $, $E^{\prime}=\left\{
e_{1},\ldots,e_{m^{\prime}}\right\}  $ finite sets, $Q:\mathcal{P}\left(
E\right)  ^{n}\rightarrow\mathbf{I}$ a semi-fuzzy quantifier, $X_{1}^{\prime
}\,\ldots,X_{n}^{\prime}\in\mathcal{P}\left(  E^{\prime}\right)  $ fuzzy sets
and $f_{1},\ldots,f_{n}:E^{\prime}\rightarrow E,E^{\prime}\neq\varnothing$. We
point that $\widehat{\mathcal{F}}\left(  f_{i}\right)  \left(  X_{i}\right)
\in\widetilde{\mathcal{P}}\left(  E\right)  $. By computation%
\begin{align}
&  \mathcal{F}^{A}\left(  Q\right)  \circ\times_{i=1}^{n}\widehat
{\mathcal{F}^{A}}\left(  f_{i}\right)  \left(  X_{1}^{\prime},\ldots
,X_{n}^{\prime}\right) \label{EqAxiomaZ6ModeloFA_3}\\
&  =\mathcal{F}^{A}\left(  Q\right)  \circ\left(  \widehat{\mathcal{F}^{A}%
}\left(  f_{1}\right)  \left(  X_{1}^{\prime}\right)  ,\ldots,\widehat
{\mathcal{F}^{A}}\left(  f_{n}\right)  \left(  X_{n}^{\prime}\right)  \right)
\nonumber\\
&  =\sum_{Y_{1}\in\mathcal{P}\left(  E\right)  }\ldots\sum_{Y_{n}%
\in\mathcal{P}\left(  E\right)  }m_{\widehat{\mathcal{F}^{A}}\left(
f_{1}\right)  \left(  X_{1}^{\prime}\right)  }\ldots m_{\widehat
{\mathcal{F}^{A}}\left(  f_{n}\right)  \left(  X_{n}^{\prime}\right)
}Q\left(  Y_{1},\ldots,Y_{n}\right) \nonumber
\end{align}

Using result \ref{PropAxiomaZ6ModeloFA_3} we can rewrite $\widehat
{\mathcal{F}^{A}}\left(  f_{i}\right)  \left(  X_{i}^{\prime}\right)
\in\widetilde{\mathcal{P}}\left(  E\right)  $ as%
\begin{align*}
&  \widehat{\mathcal{F}^{A}}\left(  f_{i}\right)  \left(  X_{i}^{\prime
}\right) \\
&  =\left\{  \mathcal{F}^{A}\left(  \exists_{\widehat{f_{i}}^{-1}\left(
e_{1}\right)  }\right)  \left(  \left(  X_{i}^{\prime}\right)  ^{\widehat
{f_{i}}^{-1}\left(  e_{1}\right)  }\right)  /e_{1},\ldots,\mathcal{F}%
^{A}\left(  \exists_{\widehat{f_{i}}^{-1}\left(  e_{m}\right)  }\right)
\left(  \left(  X_{i}^{\prime}\right)  ^{\widehat{f_{i}}^{-1}\left(
e_{m}\right)  }\right)  /e_{m}\right\}
\end{align*}

Rewriting expression \ref{EqAxiomaZ6ModeloFA_3}%
\begin{align}
\ldots &  =\sum_{Y_{1}\in\mathcal{P}\left(  E\right)  }\ldots\sum_{Y_{n}%
\in\mathcal{P}\left(  E\right)  }\label{EqAxiomaZ6ModeloFA_3p5}\\
&  m_{\left\{  \mathcal{F}^{A}\left(  \exists_{\widehat{f_{1}}^{-1}\left(
e_{1}\right)  }\right)  \left(  \left(  X_{1}^{\prime}\right)  ^{\widehat
{f_{1}}^{-1}\left(  e_{1}\right)  }\right)  /e_{1},\ldots,\mathcal{F}%
^{A}\left(  \exists_{\widehat{f_{1}}^{-1}\left(  e_{m}\right)  }\right)
\left(  \left(  X_{1}^{\prime}\right)  ^{\widehat{f_{1}}^{-1}\left(
e_{m}\right)  }\right)  /e_{m}\right\}  }\left(  Y_{1}\right)  \ldots
\nonumber\\
&  m_{\left\{  \mathcal{F}^{A}\left(  \exists_{\widehat{f_{n}}^{-1}\left(
e_{1}\right)  }\right)  \left(  \left(  X_{n}^{\prime}\right)  ^{\widehat
{f_{n}}^{-1}\left(  e_{1}\right)  }\right)  /e_{1},\ldots,\mathcal{F}%
^{A}\left(  \exists_{\widehat{f_{n}}^{-1}\left(  e_{m}\right)  }\right)
\left(  \left(  X_{n}^{\prime}\right)  ^{\widehat{f_{n}}^{-1}\left(
e_{m}\right)  }\right)  /e_{m}\right\}  }\left(  Y_{n}\right)  Q\left(
Y_{1},\ldots,Y_{n}\right) \nonumber
\end{align}

Let be$Y_{j}=\left\{  e_{j_{1}},\ldots,e_{_{j_{k}}}\right\}  \in
\mathcal{P}\left(  E\right)  ,E\backslash Y_{j}=\left\{  e_{_{j_{k+1}}}%
,\ldots,e_{_{j_{m}}}\right\}  $. We will compute the probability mass
$m_{\widehat{\mathcal{F}^{A}}\left(  f_{i}\right)  \left(  X_{i}^{\prime
}\right)  }\left(  Y_{j}\right)  $:%
\begin{align*}
&  m_{\widehat{\mathcal{F}^{A}}\left(  f_{i}\right)  \left(  X_{i}^{\prime
}\right)  }\left(  Y_{j}\right) \\
&  =m_{\left\{  \mathcal{F}^{A}\left(  \exists_{\widehat{f_{i}}^{-1}\left(
e_{1}\right)  }\right)  \left(  \left(  X_{i}^{\prime}\right)  ^{\widehat
{f_{i}}^{-1}\left(  e1\right)  }\right)  /e_{1},\ldots,\mathcal{F}^{A}\left(
\exists_{\widehat{f_{i}}^{-1}\left(  e_{m}\right)  }\right)  \left(  \left(
X_{i}^{\prime}\right)  ^{\widehat{f_{i}}^{-1}\left(  e_{m}\right)  }\right)
/e_{m}\right\}  }\left(  \left\{  e_{_{j_{1}}},\ldots,e_{_{j_{k}}}\right\}
\right) \\
&  =\mathcal{F}^{A}\left(  \exists_{\widehat{f_{i}}^{-1}\left(  e_{j_{1}%
}\right)  }\right)  \left(  \left(  X_{i}^{\prime}\right)  ^{\widehat{f_{i}%
}^{-1}\left(  e_{j_{1}}\right)  }\right)  \cdot\ldots\cdot\mathcal{F}%
^{A}\left(  \exists_{\widehat{f_{i}}^{-1}\left(  e_{j_{k}}\right)  }\right)
\left(  \left(  X_{i}^{\prime}\right)  ^{\widehat{f_{i}}^{-1}\left(  e_{j_{k}%
}\right)  }\right)  \cdot\\
&  \left(  1-\mathcal{F}^{A}\left(  \exists_{\widehat{f_{i}}^{-1}\left(
e_{j_{k+1}}\right)  }\right)  \left(  \left(  X_{i}^{\prime}\right)
^{\widehat{f_{i}}^{-1}\left(  e_{j_{k+1}}\right)  }\right)  \right)
\cdot\ldots\cdot\left(  1-\mathcal{F}^{A}\left(  \exists_{\widehat{f_{i}}%
^{-1}\left(  e_{j_{m}}\right)  }\right)  \left(  \left(  X_{i}^{\prime
}\right)  ^{\widehat{f_{i}}^{-1}\left(  e_{j_{m}}\right)  }\right)  \right)
\end{align*}

And by duality we now that $\widetilde{\lnot}\mathcal{F}^{A}\left(
\exists\right)  =\mathcal{F}^{A}\left(  \forall\right)  \widetilde{\lnot}$.
Then,,%
\begin{align}
&  m_{\widehat{\mathcal{F}^{A}}\left(  f_{i}\right)  \left(  X_{i}^{\prime
}\right)  }\left(  Y_{j}\right) \label{EqAxiomaZ6ModeloFA_4}\\
&  =\ldots\nonumber\\
&  =\mathcal{F}^{A}\left(  \exists_{\widehat{f_{i}}^{-1}\left(  e_{j_{1}%
}\right)  }\right)  \left(  \left(  X_{i}^{\prime}\right)  ^{\widehat{f_{i}%
}^{-1}\left(  e_{j_{1}}\right)  }\right)  \cdot\ldots\cdot\mathcal{F}%
^{A}\left(  \exists_{\widehat{f_{i}}^{-1}\left(  e_{j_{k}}\right)  }\right)
\left(  \left(  X_{i}^{\prime}\right)  ^{\widehat{f_{i}}^{-1}\left(  e_{j_{k}%
}\right)  }\right)  \cdot\nonumber\\
&  \mathcal{F}\left(  \forall_{\widehat{f_{i}}^{-1}\left(  e_{j_{k+1}}\right)
}\right)  \left(  \widetilde{\lnot}\left(  X_{i}^{\prime}\right)
^{\widehat{f_{i}}^{-1}\left(  e_{j_{k+1}}\right)  }\right)  \cdot\text{\ldots
}\cdot\mathcal{F}^{A}\left(  \forall_{\widehat{f_{i}}^{-1}\left(  e_{j_{m}%
}\right)  }\right)  \left(  \widetilde{\lnot}\left(  X_{i}^{\prime}\right)
^{\widehat{f_{i}}^{-1}\left(  e_{j_{m}}\right)  }\right) \nonumber
\end{align}

As%
\begin{align*}
\mathcal{F}^{A}\left(  \exists_{\widehat{f_{i}}^{-1}\left(  e_{j_{r}}\right)
}\right)  \left(  \left(  X_{i}^{\prime}\right)  ^{\widehat{f_{i}}^{-1}\left(
e_{j_{r}}\right)  }\right)   &  =\sum_{\varnothing\subset M\subseteq
\widehat{f}_{i}^{-1}\left(  e_{j_{r}}\right)  }m_{\left(  X_{i}^{\prime
}\right)  ^{\widehat{f_{i}}^{-1}\left(  e_{j_{r}}\right)  }}^{\widehat{f_{i}%
}^{-1}\left(  e_{j_{r}}\right)  }\left(  M\right) \\
&  =\sum_{\varnothing\subset M\subseteq\widehat{f}_{i}^{-1}\left(  e_{j_{r}%
}\right)  }m_{X_{i}^{\prime}}^{\widehat{f_{i}}^{-1}\left(  e_{j_{r}}\right)
}\left(  M\right) \\
\mathcal{F}^{A}\left(  \forall_{\widehat{f_{i}}^{-1}\left(  e_{j_{r}}\right)
}\right)  \left(  \widetilde{\lnot}\left(  X_{i}^{\prime}\right)
^{\widehat{f_{i}}^{-1}\left(  e_{j_{r}}\right)  }\right)   &  =%
{\displaystyle\prod_{e^{\prime}\in\widehat{f_{i}}^{-1}\left(  e_{j_{r}%
}\right)  }}
\widetilde{\lnot}\mu_{\left(  X_{i}^{\prime}\right)  ^{\widehat{f_{i}}%
^{-1}\left(  e_{j_{r}}\right)  }}\left(  e^{\prime}\right) \\
&  =%
{\displaystyle\prod_{e^{\prime}\in\widehat{f_{i}}^{-1}\left(  e_{j_{r}%
}\right)  }}
\widetilde{\lnot}\mu_{X_{i}^{\prime}}\left(  e^{\prime}\right)
\end{align*}
then expression \ref{EqAxiomaZ6ModeloFA_4} is equivalent to%
\begin{align}
&  m_{\widehat{\mathcal{F}^{A}}\left(  f_{i}\right)  \left(  X_{i}^{\prime
}\right)  }\left(  Y_{j}\right) \label{EqAxiomaZ6ModeloFA_5}\\
&  =\ldots\nonumber\\
&  =\sum_{\varnothing\subset M\subseteq\widehat{f}_{i}^{-1}\left(  e_{j_{1}%
}\right)  }m_{X_{i}^{\prime}}^{\widehat{f_{i}}^{-1}\left(  e_{j_{1}}\right)
}\left(  M\right)  \cdot\ldots\cdot\sum_{\varnothing\subset M\subseteq
\widehat{f}_{i}^{-1}\left(  e_{j_{k}}\right)  }m_{X_{i}^{\prime}}%
^{\widehat{f_{i}}^{-1}\left(  e_{j_{k}}\right)  }\left(  M\right) \nonumber\\
&  \cdot%
{\displaystyle\prod\limits_{e^{\prime}\in E^{\prime}\backslash\left(
\widehat{f}_{i}^{-1}\left(  e_{j_{1}}\right)  \cup\ldots\cup\widehat{f}%
_{r}^{-1}\left(  e_{j_{k}}\right)  \right)  }}
\widetilde{\lnot}\mu_{X_{i}}\left(  e^{\prime}\right) \nonumber\\
&  =\sum_{\substack{M\in\mathcal{P}\left(  E^{\prime}\right)  |\\\widehat
{f}_{i}^{-1}\left(  e_{j_{1}}\right)  \cap Y_{j}\neq\varnothing\wedge
\ldots\wedge\\\widehat{f}_{i}^{-1}\left(  e_{j_{k}}\right)  \cap Y_{j}%
\neq\varnothing\wedge\\\widehat{f}_{r}^{-1}\left(  e_{j_{k+1}}\right)  \cap
Y_{j}=\varnothing\wedge\ldots\wedge\\\widehat{f}_{r}^{-1}\left(  e_{j_{m}%
}\right)  \cap Y_{j}=\varnothing}}m_{X_{i}^{\prime}}\left(  M\right) \nonumber
\end{align}

In this way, the probability mass $m_{\widehat{\mathcal{F}^{A}}\left(
f_{i}\right)  \left(  X_{i}^{\prime}\right)  }\left(  Y_{j}\right)  $ is
computed by using the probability masses $m_{X_{i}^{\prime}}\left(  M\right)
$ that are associated to the $Ms\in\mathcal{P}\left(  E^{\prime}\right)  $
such that the intersection with the inverse image of the $e\in Y_{j}$ is not
empty, and such that the intersection with the $e\in E\backslash Y_{j}$ is
empty. Moreover, we should note that all $M\in\mathcal{P}\left(  E^{\prime
}\right)  $ is associated to one $Y\in\mathcal{P}\left(  E\right)  $; that is,
$\widehat{f_{i}}\left(  M\right)  =Y$ for some $Y$. In this way, if the $Ys$
visit $\mathcal{P}\left(  E\right)  $, then the $Ms$ visit $\mathcal{P}\left(
E^{\prime}\right)  $, and continuing with the computation of expression
\ref{EqAxiomaZ6ModeloFA_3p5} we obtain
\begin{align*}
&  \mathcal{F}^{A}\left(  Q\right)  \circ\times_{i=1}^{n}\widehat
{\mathcal{F}^{A}}\left(  f_{i}\right)  \left(  X_{1}^{\prime},\ldots
,X_{n}^{\prime}\right) \\
&  =\sum_{Y_{1}\in\mathcal{P}\left(  E\right)  }\ldots\sum_{Y_{n}%
\in\mathcal{P}\left(  E\right)  }m_{\widehat{\mathcal{F}^{A}}\left(
f_{1}\right)  \left(  X_{1}^{\prime}\right)  }\ldots m_{\widehat
{\mathcal{F}^{A}}\left(  f_{n}\right)  \left(  X_{n}^{\prime}\right)  }\left(
Y_{1},\ldots,Y_{n}\right) \\
&  =\ldots\\
&  =\sum_{Y_{1}\in\mathcal{P}\left(  E\right)  }\ldots\sum_{Y_{n}%
\in\mathcal{P}\left(  E\right)  }m_{\left\{  \mathcal{F}^{A}\left(
\exists_{\widehat{f_{1}}^{-1}\left(  e_{1}\right)  }\right)  \left(  \left(
X_{1}^{\prime}\right)  ^{\widehat{f_{1}}^{-1}\left(  e_{1}\right)  }\right)
/e_{1},\ldots,\mathcal{F}^{A}\left(  \exists_{\widehat{f_{1}}^{-1}\left(
e_{m}\right)  }\right)  \left(  \left(  X_{1}^{\prime}\right)  ^{\widehat
{f_{1}}^{-1}\left(  e_{m}\right)  }\right)  /e_{m}\right\}  }\left(
Y_{1}\right) \\
&  \ldots m_{\left\{  \mathcal{F}^{A}\left(  \exists_{\widehat{f_{n}}%
^{-1}\left(  e_{1}\right)  }\right)  \left(  \left(  X_{n}^{\prime}\right)
^{\widehat{f_{n}}^{-1}\left(  e_{1}\right)  }\right)  /e_{1},\ldots
,\mathcal{F}^{A}\left(  \exists_{\widehat{f_{n}}^{-1}\left(  e_{m}\right)
}\right)  \left(  \left(  X_{n}^{\prime}\right)  ^{\widehat{f_{n}}^{-1}\left(
e_{m}\right)  }\right)  /e_{m}\right\}  }\left(  Y_{n}\right) \\
&  Q\left(  Y_{1},\ldots,Y_{n}\right) \\
&  =\ldots\\
&  =\sum_{\substack{Y_{1}=\left\{  e_{\left(  Y_{1}\right)  _{1}}%
,\ldots,e_{\left(  Y_{1}\right)  _{k}}\right\}  \in\mathcal{P}\left(
E\right)  ,\\E\backslash Y_{1}=\left\{  e_{\left(  Y_{1}\right)  _{k+1}%
},\ldots,e_{\left(  Y_{1}\right)  _{m}}\right\}  }}\ldots\sum_{\substack{Y_{n}%
=\left\{  e_{\left(  Y_{n}\right)  _{1}},\ldots,e_{\left(  Y_{n}\right)  _{k}%
}\right\}  \in\mathcal{P}\left(  E\right)  ,\\E\backslash Y_{n}=\left\{
e_{\left(  Y_{n}\right)  _{k+1}},\ldots,e_{\left(  Y_{n}\right)  _{m}%
}\right\}  }}\\
&  \sum_{\substack{Y_{1}^{\prime}\in\mathcal{P}\left(  E^{\prime}\right)
|\\\widehat{f}_{1}^{-1}\left(  e_{\left(  Y_{1}\right)  _{1}}\right)  \cap
Y_{1}^{\prime}\neq\varnothing\wedge\ldots\wedge\\\widehat{f}_{1}^{-1}\left(
e_{\left(  Y_{1}\right)  _{k}}\right)  \cap Y_{1}^{\prime}\neq\varnothing
\wedge\\\widehat{f}_{1}^{-1}\left(  e_{\left(  Y_{1}\right)  _{k+1}}\right)
\cap Y_{1}^{\prime}=\varnothing\wedge\ldots\wedge\\\widehat{f}_{1}^{-1}\left(
e_{\left(  Y_{1}\right)  _{m}}\right)  \cap Y_{1}^{\prime}=\varnothing
}}m_{X_{1}^{\prime}}\left(  Y_{1}^{\prime}\right)  \ldots\sum_{\substack{Y_{n}%
^{\prime}\in\mathcal{P}\left(  E^{\prime}\right)  |\\\widehat{f}_{n}%
^{-1}\left(  e_{\left(  Y_{n}\right)  _{1}}\right)  \cap Y_{n}^{^{\prime}}%
\neq\varnothing\wedge\ldots\wedge\\\widehat{f}_{n}^{-1}\left(  e_{\left(
Y_{n}\right)  _{k}}\right)  \cap Y_{n}^{^{\prime}}\neq\varnothing
\wedge\\\widehat{f}_{n}^{-1}\left(  e_{\left(  Y_{n}\right)  _{k+1}}\right)
\cap Y_{n}^{^{\prime}}=\varnothing\wedge\ldots\wedge\\\widehat{f}_{n}%
^{-1}\left(  e_{\left(  Y_{n}\right)  _{m}}\right)  \cap Y_{n}^{\prime
}=\varnothing}}m_{X_{n}^{\prime}}\left(  Y_{n}^{\prime}\right)  Q\left(
Y_{1},\ldots,Y_{n}\right)
\end{align*}%
\begin{align*}
&  =\sum_{\substack{Y_{1}=\left\{  e_{\left(  Y_{1}\right)  _{1}}%
,\ldots,e_{\left(  Y_{1}\right)  _{k}}\right\}  \in\mathcal{P}\left(
E\right)  ,\\E\backslash Y_{1}=\left\{  e_{\left(  Y_{1}\right)  _{k+1}%
},\ldots,e_{\left(  Y_{1}\right)  _{m}}\right\}  }}\ldots\sum_{\substack{Y_{n}%
=\left\{  e_{\left(  Y_{n}\right)  _{1}},\ldots,e_{\left(  Y_{n}\right)  _{k}%
}\right\}  \in\mathcal{P}\left(  E\right)  ,\\E\backslash Y_{n}=\left\{
e_{\left(  Y_{n}\right)  _{k+1}},\ldots,e_{\left(  Y_{n}\right)  _{m}%
}\right\}  }}\\
&  \sum_{\substack{Y_{1}^{\prime}\in\mathcal{P}\left(  E^{\prime}\right)
|\\\widehat{f}_{1}^{-1}\left(  e_{\left(  Y_{1}\right)  _{1}}\right)  \cap
Y_{1}^{\prime}\neq\varnothing\wedge\ldots\wedge\\\widehat{f}_{1}^{-1}\left(
e_{\left(  Y_{1}\right)  _{k}}\right)  \cap Y_{1}^{\prime}\neq\varnothing
\wedge\\\widehat{f}_{1}^{-1}\left(  e_{\left(  Y_{1}\right)  _{k+1}}\right)
\cap Y_{1}^{\prime}=\varnothing\wedge\ldots\wedge\\\widehat{f}_{1}^{-1}\left(
e_{\left(  Y_{1}\right)  _{m}}\right)  \cap Y_{1}^{\prime}=\varnothing
}}m_{X_{1}^{\prime}}\left(  Y_{1}^{\prime}\right)  \ldots\sum_{\substack{Y_{n}%
^{\prime}\in\mathcal{P}\left(  E^{\prime}\right)  |\\\widehat{f}_{n}%
^{-1}\left(  e_{\left(  Y_{n}\right)  _{1}}\right)  \cap Y_{n}^{^{\prime}}%
\neq\varnothing\wedge\ldots\wedge\\\widehat{f}_{n}^{-1}\left(  e_{\left(
Y_{n}\right)  _{k}}\right)  \cap Y_{n}^{^{\prime}}\neq\varnothing
\wedge\\\widehat{f}_{n}^{-1}\left(  e_{\left(  Y_{n}\right)  _{k+1}}\right)
\cap Y_{n}^{^{\prime}}=\varnothing\wedge\ldots\wedge\\\widehat{f}_{n}%
^{-1}\left(  e_{\left(  Y_{n}\right)  _{m}}\right)  \cap Y_{n}^{\prime
}=\varnothing}}m_{X_{n}^{\prime}}\left(  Y_{n}^{\prime}\right)  Q\left(
\widehat{f}_{1}\left(  Y_{1}^{\prime}\right)  ,\ldots,\widehat{f}_{n}\left(
Y_{n}^{\prime}\right)  \right) \\
&  =\sum_{\substack{Y_{1}^{\prime}\in\mathcal{P}\left(  E^{\prime}\right)
}}m_{X_{1}^{\prime}}\left(  Y_{1}^{\prime}\right)  \ldots\sum_{\substack{Y_{n}%
^{\prime}\in\mathcal{P}\left(  E^{\prime}\right)  }}m_{X_{n}^{\prime}}\left(
Y_{n}^{\prime}\right)  Q\left(  \widehat{f}_{1}\left(  Y_{1}^{\prime}\right)
,\ldots,\widehat{f}_{n}\left(  Y_{n}^{\prime}\right)  \right) \\
&  =\sum_{\substack{Y_{1}^{\prime}\in\mathcal{P}\left(  E^{\prime}\right)
}}m_{X_{1}^{\prime}}\left(  Y_{1}^{\prime}\right)  \ldots\sum_{\substack{Y_{n}%
^{\prime}\in\mathcal{P}\left(  E^{\prime}\right)  }}m_{X_{n}^{\prime}}\left(
Y_{n}^{\prime}\right)  \left(  Q\circ\times_{i=1}^{n}\widehat{f_{i}}\right)
\left(  Y_{1}^{\prime},\ldots,Y_{n}^{\prime}\right) \\
&  =\mathcal{F}\left(  Q\circ\times_{i=1}^{n}\widehat{f_{i}}\right)  \left(
X_{1}^{\prime},\ldots,X_{n}^{\prime}\right)
\end{align*}
as we want to prove.
\end{proof}

\subsection{Properties of the $\mathcal{F}^{A}$ that are not consequences of
the DFS framework}

\subsubsection{Property of argument continuity}

The $\mathcal{F}^{A}$ model fulfills the property of argument continuity.

\begin{proof}
In \cite[appendix A]{DiazHermida06Tesis} a detailed proof can be found. But
the next arguments are enough to prove continuity. Let us consider the
definition of the $\mathcal{F}^{A}$ model:%
\[
\mathcal{F}^{A}\left(  Q\right)  \left(  X_{1},\ldots,X_{n}\right)
=\sum_{Y_{1}\in\mathcal{P}\left(  E\right)  }\ldots\sum_{Y_{n}\in
\mathcal{P}\left(  E\right)  }m_{X_{1}}\left(  Y_{1}\right)  \ldots m_{X_{n}%
}\left(  Y_{n}\right)  Q\left(  Y_{1},\ldots,Y_{n}\right)
\]
Note that for crisp sets $Y_{1},\ldots,Y_{n}\in\mathcal{P}\left(  E\right)  $
$m_{X_{1}}\left(  Y_{1}\right)  \ldots m_{X_{n}}\left(  Y_{n}\right)  Q\left(
Y_{1},\ldots,Y_{n}\right)  $ are continuous funtions, because $Q\left(
Y_{1},\ldots,Y_{n}\right)  $ is constant and $m_{X_{1}}\left(  Y_{1}\right)
\ldots m_{X_{n}}\left(  Y_{n}\right)  $ only involves the use of the product operation.

But the sum of continuous functions (that is, the sum over $\left(
Y_{1},\ldots,Y_{n}\right)  \in\mathcal{P}\left(  E\right)  ^{n}$ is
continuous. And then the model is continuous in arguments.
\end{proof}

\subsubsection{Propery of quantifier continuity}

The model $\mathcal{F}^{A}$ fulfills the property of $Q$-continuity:

\begin{proof}
Let $Q,Q^{\prime}:\mathcal{P}\left(  E\right)  ^{n}\rightarrow\mathbf{I}$
semi-fuzzy quantifiers. Then,%
\begin{align*}
&  d\left(  \mathcal{F}^{A}\left(  Q\right)  ,\mathcal{F}^{A}\left(
Q^{^{\prime}}\right)  \right) \\
&  =\sup\left\{  \left\vert
\begin{array}
[c]{c}%
\sum_{Y_{1}\in\mathcal{P}\left(  E\right)  }\ldots\sum_{Y_{n}\in
\mathcal{P}\left(  E\right)  }m_{X_{1}}\left(  Y_{1}\right)  \ldots m_{X_{n}%
}\left(  Y_{n}\right)  Q\left(  Y_{1},\ldots,Y_{n}\right) \\
-\sum_{Y_{1}\in\mathcal{P}\left(  E\right)  }\ldots\sum_{Y_{n}\in
\mathcal{P}\left(  E\right)  }m_{X_{1}}\left(  Y_{1}\right)  \ldots m_{X_{n}%
}\left(  Y_{n}\right)  Q^{\prime}\left(  Y_{1},\ldots,Y_{n}\right)
\end{array}
\right\vert \right. \\
&  \left.  :X_{1},\ldots,X_{n}\in\widetilde{\mathcal{P}}\left(  E\right)
\right\} \\
&  =\sup\left\{  \left\vert
\begin{array}
[c]{c}%
\sum_{Y_{1}\in\mathcal{P}\left(  E\right)  }\ldots\sum_{Y_{n}\in
\mathcal{P}\left(  E\right)  }m_{X_{1}}\left(  Y_{1}\right)  \ldots m_{X_{n}%
}\left(  Y_{n}\right) \\
\left(  Q\left(  Y_{1},\ldots,Y_{n}\right)  -Q^{\prime}\left(  Y_{1}%
,\ldots,Y_{n}\right)  \right)
\end{array}
\right\vert :X_{1},\ldots,X_{n}\in\widetilde{\mathcal{P}}\left(  E\right)
\right\} \\
&  \leq\sup\left\{  \left\vert \sum_{Y_{1}\in\mathcal{P}\left(  E\right)
}\ldots\sum_{Y_{n}\in\mathcal{P}\left(  E\right)  }m_{X_{1}}\left(
Y_{1}\right)  \ldots m_{X_{n}}\left(  Y_{n}\right)  d\left(  Q,Q^{^{\prime}%
}\right)  \right\vert :X_{1},\ldots,X_{n}\in\widetilde{\mathcal{P}}\left(
E\right)  \right\} \\
&  =\sup\left\{  \left\vert d\left(  Q,Q^{^{\prime}}\right)  \sum_{Y_{1}%
\in\mathcal{P}\left(  E\right)  }\ldots\sum_{Y_{n}\in\mathcal{P}\left(
E\right)  }m_{X_{1}}\left(  Y_{1}\right)  \ldots m_{X_{n}}\left(
Y_{n}\right)  \right\vert :X_{1},\ldots,X_{n}\in\widetilde{\mathcal{P}}\left(
E\right)  \right\} \\
&  =\sup\left\{  \left\vert d\left(  Q,Q^{^{\prime}}\right)  \right\vert
:X_{1},\ldots,X_{n}\in\widetilde{\mathcal{P}}\left(  E\right)  \right\} \\
&  =d\left(  Q,Q^{^{\prime}}\right)
\end{align*}

And the property is fulfilled for $\delta<\varepsilon$.
\end{proof}

\subsubsection{Property of the fuzzy argument insertion}

The $\mathcal{F}^{A}$ verifies the property of fuzzy argument insertion.

\begin{proof}
Let $Q:\mathcal{P}^{n+1}\left(  E\right)  \rightarrow\mathbf{I}$ a semi-fuzzy
quantifier be given and $A\in\widetilde{\mathcal{P}}\left(  E\right)  $ a
fuzzy set. For all $Y_{1},\ldots,Y_{n}\in\mathcal{P}\left(  E\right)  $ crisp
we have%
\begin{align*}
&  Q\widetilde{\vartriangleleft}A\left(  Y_{1},\ldots,Y_{n}\right) \\
&  =\mathcal{U}\left(  \mathcal{F}^{A}\left(  Q\right)  \vartriangleleft
A\right)  \left(  Y_{1},\ldots,Y_{n}\right) \\
&  =\mathcal{U}\left(  \left(
\begin{array}
[c]{c}%
f:\left(  K_{1},\ldots,K_{n+1}\right)  \in\widetilde{\mathcal{P}}\left(
E\right)  \rightarrow\sum_{Z_{1}\in\mathcal{P}\left(  E\right)  }\ldots
\sum_{Z_{n}\in\mathcal{P}\left(  E\right)  }\sum_{Z_{n+1}\in\mathcal{P}\left(
E\right)  }\\
m_{K_{1}}\left(  Z_{1}\right)  \ldots m_{K_{n}}\left(  Z_{n}\right)
m_{K_{n+1}}\left(  Z_{n+1}\right)  Q\left(  Z_{1},\ldots,Z_{n},Z_{n+1}\right)
\end{array}
\right)  \vartriangleleft A\right) \\
&  \left(  Y_{1},\ldots,Y_{n}\right) \\
&  =\mathcal{U}\left(  \left(  f^{\prime}:\left(  K_{1}^{\prime},\ldots
,K_{n}^{\prime}\right)  \in\widetilde{\mathcal{P}}\left(  E\right)
\rightarrow f\left(  K_{1}^{\prime},\ldots,K_{n}^{\prime},A\right)  \right)
\right)  \left(  Y_{1},\ldots,Y_{n}\right) \\
&  =\mathcal{U}\left(  \left(
\begin{array}
[c]{c}%
f^{\prime}:\left(  K_{1}^{\prime},\ldots,K_{n}^{\prime}\right)  \in
\widetilde{\mathcal{P}}\left(  E\right)  \rightarrow\sum_{Z_{1}\in
\mathcal{P}\left(  E\right)  }\ldots\sum_{Z_{n}\in\mathcal{P}\left(  E\right)
}\sum_{Z_{n+1}\in\mathcal{P}\left(  E\right)  }\\
m_{K_{1}^{\prime}}\left(  Z_{1}\right)  \ldots m_{K_{n}^{\prime}}\left(
Z_{n}\right)  m_{A}\left(  Z_{n+1}\right)  Q\left(  Z_{1},\ldots,Z_{n}%
,Z_{n+1}\right)
\end{array}
\right)  \right) \\
&  \left(  Y_{1},\ldots,Y_{n}\right) \\
&  =\left(
\begin{array}
[c]{c}%
f^{\prime\prime}:\left(  K_{1}^{\prime\prime},\ldots,K_{n}^{\prime\prime
}\right)  \in\mathcal{P}\left(  E\right)  \rightarrow\sum_{Z_{1}\in
\mathcal{P}\left(  E\right)  }\ldots\sum_{Z_{n}\in\mathcal{P}\left(  E\right)
}\sum_{Z_{n+1}\in\mathcal{P}\left(  E\right)  }\\
m_{K_{1}^{\prime\prime}}\left(  Z_{1}\right)  \ldots m_{K_{n}^{\prime\prime}%
}\left(  Z_{n}\right)  m_{A}\left(  Z_{n+1}\right)  Q\left(  Z_{1}%
,\ldots,Z_{n},Z_{n+1}\right)
\end{array}
\right) \\
&  \left(  Y_{1},\ldots,Y_{n}\right) \\
&  =\left(  f^{\prime\prime}:\left(  K_{1}^{\prime\prime},\ldots,K_{n}%
^{\prime\prime}\right)  \in\mathcal{P}\left(  E\right)  \rightarrow
\sum_{Z_{n+1}\in\mathcal{P}\left(  E\right)  }m_{A}\left(  Z_{n+1}\right)
Q\left(  K_{1}^{\prime\prime},\ldots,K_{n}^{\prime\prime},Z_{n+1}\right)
\right) \\
&  \left(  Y_{1},\ldots,Y_{n}\right)
\end{align*}
because $\left(  K_{1}^{\prime\prime},\ldots,K_{n}^{\prime\prime}\right)
\in\mathcal{P}\left(  E\right)  $ are crisp sets, and then%
\begin{align}
Q\widetilde{\vartriangleleft}A\left(  Y_{1},\ldots,Y_{n}\right)   &
=\ldots\label{EqModeloFAInserArgBorrosa_0p5}\\
&  =\sum_{Z_{n+1}\in\mathcal{P}\left(  E\right)  }m_{A}\left(  Z_{n+1}\right)
Q\left(  Y_{1},\ldots,Y_{n},Z_{n+1}\right) \nonumber
\end{align}

Using the previously obtained result (expression
\ref{EqModeloFAInserArgBorrosa_0p5}), then:
\begin{align}
&  \mathcal{F}^{A}\left(  Q\widetilde{\vartriangleleft}A\right)  \left(
X_{1},\ldots,X_{n}\right) \label{EqModeloFAInserArgBorrosa_1}\\
&  =\sum_{Y_{1}\in\mathcal{P}\left(  E\right)  }\ldots\sum_{Y_{n}%
\in\mathcal{P}\left(  E\right)  }m_{X_{1}}\left(  Y_{1}\right)  \ldots
m_{X_{n}}\left(  Y_{n}\right)  \left(  Q\widetilde{\vartriangleleft}A\right)
\left(  Y_{1},\ldots,Y_{n}\right) \nonumber\\
&  =\sum_{Y_{1}\in\mathcal{P}\left(  E\right)  }\ldots\sum_{Y_{n}%
\in\mathcal{P}\left(  E\right)  }m_{X_{1}}\left(  Y_{1}\right)  \ldots
m_{X_{n}}\left(  Y_{n}\right)  \sum_{Z_{n+1}\in\mathcal{P}\left(  E\right)
}m_{A}\left(  Z_{n+1}\right)  Q\left(  Y_{1},\ldots,Y_{n},Z_{n+1}\right)
\nonumber\\
&  =\sum_{Y_{1}\in\mathcal{P}\left(  E\right)  }\ldots\sum_{Y_{n}%
\in\mathcal{P}\left(  E\right)  }\sum_{Y_{n+1}\in\mathcal{P}\left(  E\right)
}m_{X_{1}}\left(  Y_{1}\right)  \ldots m_{X_{n}}\left(  Y_{n}\right)
m_{A}\left(  Y_{n+1}\right)  Q\left(  Y_{1},\ldots,Y_{n},A\right) \nonumber\\
&  =\mathcal{F}^{A}\left(  Q\right)  \left(  X_{1},\ldots,X_{n},A\right)
\nonumber\\
&  =\mathcal{F}^{A}\left(  Q\right)  \vartriangleleft A\nonumber
\end{align}

\end{proof}

\subsubsection{Property of the identity quantifier.}

This property is easily proved by induction. Let us denote%
\[
\Pr\left(  card_{X}=j\right)  =\sum_{Y\in\mathcal{P}\left(  E\right)
|\left\vert Y\right\vert =j}m_{X}\left(  Y\right)
\]
the probability that the cardinality of a crisp representative of $X$ let be
$j$.

\begin{proof}
For $\left\vert E\right\vert =m$ we have%
\begin{align*}
\mathcal{F}^{A}\left(  \mathbf{identity}\right)  \left(  X\right)   &
=\sum_{Y\in\mathcal{P}\left(  E\right)  }m_{X}\left(  Y\right)
\mathbf{identity}\left(  Y\right) \\
&  =\sum_{Y\in\mathcal{P}\left(  E\right)  }m_{X}\left(  Y\right)
\frac{\left\vert Y\right\vert }{\left\vert E\right\vert }\\
&  =\sum_{j=0}^{m}\sum_{Y\in\mathcal{P}\left(  E\right)  |\left\vert
Y\right\vert =j}m_{X}\left(  Y\right)  \frac{j}{m}\\
&  =\frac{1}{m}\sum_{j=0}^{m}j\Pr\left(  card_{X}=j\right)
\end{align*}

Let us begin the induction proof:

\textbf{Case} $i=1$, $X\in\widetilde{\mathcal{P}}\left(  E^{1}\right)  $. Evident.

\textbf{Induction hypothesis.} Case $i=m$ (that is, $E=E^{m}=\left\{
e_{1},\ldots,e_{m}\right\}  $). For $X\in\widetilde{\mathcal{P}}\left(
E\right)  $ it is fulfilled%
\[
\mathcal{F}^{A}\left(  \mathbf{identity}\right)  \left(  X\right)  =\frac
{1}{m}\sum_{j=1}^{m}\mu_{X}\left(  e_{j}\right)  ,X\in\widetilde{\mathcal{P}%
}\left(  E\right)
\]

\textbf{Casw} $i=m+1$ ( $E=E^{m+1}=\left\{  e_{1},\ldots,e_{m+1}\right\}  $).

For an $m$ elements referential is fulfilled (using the induction
hypothesis).
\begin{align}
&  \sum_{j=0}^{m}\Pr\left(  card_{X}=j\right)  \frac{j+1}{m+1}%
\label{EqModeloFACuantIdent_1}\\
&  =\Pr\left(  card_{X}=0\right)  \frac{1}{m+1}+\Pr\left(  card_{X}=1\right)
\frac{2}{m+1}+\ldots+\Pr\left(  card_{X}=m\right)  \frac{m+1}{m+1}\nonumber\\
&  =\Pr\left(  card_{X}=0\right)  \frac{0}{m+1}+\Pr\left(  card_{X}=1\right)
\frac{1}{m+1}+\ldots+\Pr\left(  card_{X}=m\right)  \frac{m}{m+1}+\nonumber\\
&  \Pr\left(  card_{X}=0\right)  \frac{1}{m+1}+\Pr\left(  card_{X}=1\right)
\frac{1}{m+1}+\ldots+\Pr\left(  card_{X}=m\right)  \frac{1}{m+1}\nonumber\\
&  =\sum_{j=0}^{m}\Pr\left(  card_{X}=j\right)  \frac{j}{m+1}+\frac{1}%
{m+1}\sum_{j=0}^{m}\Pr\left(  card_{X}=j\right) \nonumber
\end{align}

\begin{align}
&  =\sum_{j=0}^{m}\Pr\left(  card_{X}=j\right)  \frac{j}{m+1}+\frac{1}%
{m+1}\nonumber\\
&  =\frac{m}{m+1}\sum_{j=0}^{m}\Pr\left(  card_{X}=m\right)  \frac{j}{m}%
+\frac{1}{m+1}\nonumber\\
&  =\frac{m}{m+1}\mathcal{F}^{A}\left(  \mathbf{identity}\right)  \left(
X\right)  +\frac{1}{m+1}\nonumber\\
&  =\frac{m}{m+1}\frac{1}{m}\sum_{j=0}^{m}\mu_{X}\left(  e_{j}\right)
+\frac{1}{m+1}\nonumber\\
&  =\frac{1}{m+1}\sum_{j=0}^{m}\mu_{X}\left(  e_{j}\right)  +\frac{1}%
{m+1}\nonumber
\end{align}

Let us suppose now that $X\in\widetilde{\mathcal{P}}\left(  E^{m+1}\right)  $.
And let be $X^{\prime}\in\widetilde{\mathcal{P}}\left(  E^{m}\right)  $ the
fuzzy set%
\[
X^{\prime}=X^{E^{m}}%
\]

Then,%
\begin{align*}
&  \mathcal{F}^{A}\left(  \mathbf{identity}\right)  \left(  X\right) \\
&  =\sum_{j=0}^{m+1}\Pr\left(  card_{X}=j\right)  \frac{j}{m+1}\\
&  =\Pr\left(  card_{X^{\prime}}=0\right)  \left(  1-\mu_{X}\left(
e_{m+1}\right)  \right)  \frac{0}{m+1}+\\
&  +\sum_{j=1}^{m}\left(  \Pr\left(  card_{X^{\prime}}=j\right)  \left(
1-\mu_{X}\left(  e_{m+1}\right)  \right)  +\Pr\left(  card_{X^{\prime}%
}=j-1\right)  \mu_{X}\left(  e_{m+1}\right)  \right)  \frac{j}{m+1}\\
&  +\Pr\left(  card_{X^{\prime}}=m\right)  \mu_{X}\left(  e_{m+1}\right)
\frac{m+1}{m+1}%
\end{align*}%
\begin{align*}
&  =\sum_{j=1}^{m}\left(  \Pr\left(  card_{X^{\prime}}=j-1\right)  \mu
_{X}\left(  e_{m+1}\right)  \right)  \frac{j}{m+1}+\Pr\left(  card_{X^{\prime
}}=m\right)  \mu_{X}\left(  e_{m+1}\right)  \frac{m+1}{m+1}\\
&  +\sum_{j=1}^{m}\Pr\left(  card_{X^{\prime}}=j\right)  \left(  1-\mu
_{X}\left(  e_{m+1}\right)  \right)  \frac{j}{m+1}+\Pr\left(  card_{X^{\prime
}}=0\right)  \left(  1-\mu_{X}\left(  e_{m+1}\right)  \right)  \frac{0}{m+1}\\
&  =\sum_{j=1}^{m+1}\left(  \Pr\left(  card_{X^{\prime}}=j-1\right)  \mu
_{X}\left(  e_{m+1}\right)  \right)  \frac{j}{m+1}+\sum_{j=0}^{m}\Pr\left(
card_{X^{\prime}}=j\right)  \left(  1-\mu_{X}\left(  e_{m+1}\right)  \right)
\frac{j}{m+1}\\
&  =\mu_{X}\left(  e_{m+1}\right)  \sum_{j=0}^{m}\Pr\left(  card_{X^{\prime}%
}=j\right)  \frac{j+1}{m+1}+\left(  1-\mu_{X}\left(  e_{m+1}\right)  \right)
\sum_{j=0}^{m}\Pr\left(  card_{X^{\prime}}=j\right)  \frac{j}{m+1}\\
&  =\mu_{X}\left(  e_{m+1}\right)  \sum_{j=0}^{m}\Pr\left(  card_{X^{\prime}%
}=j\right)  \frac{j+1}{m+1}+\left(  1-\mu_{X}\left(  e_{m+1}\right)  \right)
\frac{m}{m+1}\sum_{j=0}^{m}\Pr\left(  card_{X^{\prime}}=j\right)  \frac{j}{m}%
\end{align*}

And using expression \ref{EqModeloFACuantIdent_1} and the induction
hypothesis:%
\begin{align*}
&  =\mu_{X}\left(  e_{m+1}\right)  \left(  \frac{1}{m+1}\sum_{j=0}^{m}\mu
_{X}\left(  e_{j}\right)  +\frac{1}{m+1}\right)  +\left(  1-\mu_{X}\left(
e_{m+1}\right)  \right)  \frac{m}{m+1}\frac{1}{m}\sum_{j=0}^{m}\mu_{X}\left(
e_{j}\right) \\
&  =\mu_{X}\left(  e_{m+1}\right)  \left(  \frac{1}{m+1}\sum_{j=0}^{m}\mu
_{X}\left(  e_{j}\right)  +\frac{1}{m+1}\right)  +\left(  1-\mu_{X}\left(
e_{m+1}\right)  \right)  \frac{1}{m+1}\sum_{j=0}^{m}\mu_{X}\left(
e_{j}\right) \\
&  =\mu_{X}\left(  e_{m+1}\right)  \frac{1}{m+1}\sum_{j=0}^{m}\mu_{X}\left(
e_{j}\right)  +\mu_{X}\left(  e_{m+1}\right)  \frac{1}{m+1}\\
&  +\frac{1}{m+1}\sum_{j=0}^{m}\mu_{X}\left(  e_{j}\right)  -\mu_{X}\left(
e_{m+1}\right)  \frac{1}{m+1}\sum_{j=0}^{m}\mu_{X}\left(  e_{j}\right)
\end{align*}%
\begin{align*}
&  =\mu_{X}\left(  e_{m+1}\right)  \frac{1}{m+1}+\frac{1}{m+1}\sum_{j=0}%
^{m}\mu_{X}\left(  e_{j}\right) \\
&  =\frac{1}{m+1}\left(  \sum_{j=0}^{m}\mu_{X}\left(  e_{j}\right)  +\mu
_{X}\left(  e_{m+1}\right)  \right) \\
&  =\frac{1}{m+1}\sum_{j=0}^{m+1}\mu_{X}\left(  e_{j}\right)
\end{align*}

\end{proof}

\subsubsection{Property of the probabilistic interpretation of quantifiers}

The model $\mathcal{F}^{A}$ fulfills the property of the probabilistic
interpretation of quantifiers:

\begin{proof}
Let $Q_{1},\ldots,Q_{r}:\mathcal{P}^{n}\left(  E\right)  \rightarrow
\mathbf{I}$ a probabilistic covering of the quantification universe. Then for
all $X_{1},\ldots,X_{n}\in\widetilde{\mathcal{P}}\left(  E\right)  $%
\begin{align*}
&  \mathcal{F}\left(  Q_{1}\right)  \left(  X_{1},\ldots,X_{n}\right)
+\ldots+\mathcal{F}\left(  Q_{r}\right)  \left(  X_{1},\ldots,X_{n}\right) \\
&  =\sum_{Y_{1}\in\mathcal{P}\left(  E\right)  }\ldots\sum_{Y_{n}%
\in\mathcal{P}\left(  E\right)  }m_{X_{1}}\left(  Y_{1}\right)  \ldots
m_{X_{n}}\left(  Y_{n}\right)  Q_{1}\left(  Y_{1},\ldots,Y_{n}\right)
+\ldots+\\
&  +\sum_{Y_{1}\in\mathcal{P}\left(  E\right)  }\ldots\sum_{Y_{n}%
\in\mathcal{P}\left(  E\right)  }m_{X_{1}}\left(  Y_{1}\right)  \ldots
m_{X_{n}}\left(  Y_{n}\right)  Q_{r}\left(  Y_{1},\ldots,Y_{n}\right) \\
&  =\sum_{Y_{1}\in\mathcal{P}\left(  E\right)  }\ldots\sum_{Y_{n}%
\in\mathcal{P}\left(  E\right)  }m_{X_{1}}\left(  Y_{1}\right)  \ldots
m_{X_{n}}\left(  Y_{n}\right)  \left(  Q_{1}\left(  Y_{1},\ldots,Y_{n}\right)
+\ldots+Q_{r}\left(  Y_{1},\ldots,Y_{n}\right)  \right) \\
&  =\sum_{Y_{1}\in\mathcal{P}\left(  E\right)  }\ldots\sum_{Y_{n}%
\in\mathcal{P}\left(  E\right)  }m_{X_{1}}\left(  Y_{1}\right)  \ldots
m_{X_{n}}\left(  Y_{n}\right) \\
&  =1
\end{align*}

\end{proof}

\section{Apendix B. Efficient computation of the $\mathcal{F}^{A}$ model}

Although the time to compute the result of evaluatign a quantified expression
could seem extremely high, it is possible to develop polynomial algorithms for
quantitative quantifiers\footnote{Quantitative quantifiers are invariant under
automorphims \cite[section 4.13]{Glockner06Libro}. Quantitative quantifiers
can be expressed as a function of the cardinalities of their arguments and
their boolean combinations.}. In table \ref{AlgoritmoUnarioFA} the algorithm
to evaluate unary quantitative quantifiers is shown. A quantitative unary
semi-fuzzy quantifier depends on a function $q:\left\{  0,\ldots,\left\vert
E\right\vert \right\}  \rightarrow\mathbf{I}$; that is, a function of the
possible cardinality values in $\mathbf{I}$. The idea of the algorithm is
that, if we know the probabilities of the cardinalities in a base set of size
$k$; that is, we know the probabilities $\Pr\left(  card_{X}=0\right)
,\ldots,\Pr\left(  card_{X}=k\right)  $ then when we add one element to the
base set ($e_{k+1}$) we have to consider two possibilities to compute the
change in the probabilities of the cardinalities. One possibility is that the
element $e_{k+1}$ fulfills the property represented by $X$ (with probability
$\mu_{X}\left(  e_{k+1}\right)  $) and the other is that the element does not
fulfill the property represented by $X$ (with probability $\left(  1-\mu
_{X}\left(  e_{k+1}\right)  \right)  $). The next formula expresses the change
in the probabilities:%

\[
\Pr\left(  card_{X}=j\right)  =\left\{
\begin{array}
[c]{ccc}%
\Pr\left(  card_{X}=0\right)  \left(  1-\mu_{X}\left(  e_{k+1}\right)  \right)
& : & j=0\\%
\begin{array}
[c]{c}%
\Pr\left(  card_{X}=j\right)  \left(  1-\mu_{X}\left(  e_{k+1}\right)  \right)
\\
+\Pr\left(  card_{X}=j-1\right)  \mu_{X}\left(  e_{k+1}\right)
\end{array}
& : & 1\leq j\leq k\\
\Pr\left(  card_{X}=m\right)  \mu_{X}\left(  e_{i+1}\right)  & : & j=k+1
\end{array}
\right.
\]

Similar ideas can be used to develop algorithms for other quantitative
quantifiers. The case of binary proportional quantifiers can be consulted in
\cite[pag. 348]{DiazHermida06Tesis}.%

\begin{table}[tbp] \centering
\begin{tabular}
[c]{|l|}\hline
\textbf{Algorithm for computing }$\mathcal{F}^{A}\left(  Q\right)  \left(
X\right)  $\\\hline%
\begin{tabular}
[c]{l}%
{\small INPUT:\ X[0,...,m-1], m }${\small \geq1}$,{\small q:\{0,\ldots
,m\}}${\small \rightarrow}$\textbf{I}\\
{\small double pr\_aux\_i,pr\_aux\_i\_minus\_1;}\\
{\small double pr[0,....,m];}\\
{\small result = 0;}\\
{\small pr[0] = 1;}\\
{\small for (j = 0;j
$<$
m;j++) \{}\\
\qquad{\small pr\_aux\_i = pr[0];}\\
\qquad{\small pr[0] = (1 - X[j]) }${\small \times}$ {\small pr\_aux\_i;}\\
\qquad{\small pr\_aux\_i\_minus\_1 = pr\_aux\_i;}\\
\qquad{\small for (i = 1; i
$<$%
= j;i++) \{}\\
\qquad\qquad{\small pr\_aux\_i = pr[i];}\\
\qquad\qquad{\small pr[i] = (1 - X[j]) }${\small \times}$ {\small pr\_aux\_i +
X[j] }${\small \times}$ {\small pr\_aux\_i\_minus\_1;}\\
$\qquad$\qquad{\small pr\_aux\_i\_minus\_1 = pr\_aux\_i;}\\
$\qquad${\small \}}\\
$\qquad${\small pr[j+1] = X[j] }${\small \times}$
{\small pr\_aux\_i\_minus\_1;}\\
{\small \}}\\
{\small for (j = 0;j
$<$%
= m;j++)}\\
$\qquad${\small result = result + pr[j] }${\small \times}$ {\small q(j);}\\
{\small return result;}%
\end{tabular}
\\\hline
\end{tabular}
\caption{{Algorithm for computing unary quantitative
quantifiers $\mathcal{F}^{A}\left( Q\right) \left( X\right) $\label{AlgoritmoUnarioFA}}}%
\end{table}%

\bibliographystyle{plain}

\end{document}